\RequirePackage{fix-cm}
\documentclass[smallextended]{svjour3}       
\smartqed  
\usepackage[dvips]{epsfig}
\usepackage{graphicx}
\graphicspath{ {images/} }
\usepackage{pstricks, pst-node}
\usepackage{caption}
\usepackage{subcaption}
\usepackage{framed}
\usepackage{bm}

\usepackage{cite}
\usepackage{enumerate}
\usepackage{lscape}

\usepackage{amsmath, amssymb, amsthm}
\usepackage{mathrsfs}

\usepackage{rotating} 

\usepackage{multirow,bigdelim}
\usepackage[all,arc]{xy}
\usepackage{enumerate}
\usepackage{mathrsfs}
\usepackage{pseudocode}
\usepackage[ruled,vlined]{algorithm2e}
\usepackage{ctable}
\usepackage{wrapfig}
\usepackage{ragged2e}
\usepackage{ctable}

\usepackage{adjustbox}
\usepackage{xcolor}
\usepackage[colorlinks]{hyperref}

\DeclareMathOperator*{\argmin}{arg\,min}
\DeclareMathOperator*{\argmax}{arg\,max}

\def\N{\mathbb{N}}

\def\I{\mathbb{I}}
\def\P{\mathrm{P}}

\def\-{\mbox{--}}

\theoremstyle{plain}

\newtheorem{assm}[]{Assumption}

\LinesNumbered

\let\oldnl\nl
\newcommand{\nonl}{\renewcommand{\nl}{\let\nl\oldnl}}
\makeatletter
\let\c@equation\c@thm
\makeatother
\def\bbbr{\mathbb{R}}

\begin{document}

\title{A Cross Entropy based Optimization Algorithm with Global Convergence Guarantees}



\author{Ajin George Joseph         \and
        Shalabh Bhatnagar 
}


\institute{Indian Institute of Science, Bangalore \at
              INDIA, 560012 \\
              Tel.: +123-45-678910\\
              Fax: +123-45-678910\\
              \email{ajin@iisc.ac.in}           
           \and
           Indian Institute of Science, Bangalore\at
              INDIA, 560012
}


\maketitle

\begin{abstract}
The cross entropy (CE) method is a model based search method to solve optimization problems where the objective function has minimal structure. The Monte-Carlo version of the CE method employs the naive sample averaging technique which is inefficient, both computationally and space wise. We provide a novel stochastic approximation version of the CE method, where the sample averaging is replaced with incremental geometric averaging. This approach can save considerable computational and storage costs.  Our algorithm is incremental in nature and possesses additional attractive features such as accuracy, stability, robustness and convergence to the global optimum for a particular class of objective functions. We evaluate the algorithm on a variety of global optimization benchmark problems and the results obtained corroborate our theoretical findings.
\end{abstract}

\section{Introduction}\label{intro}

In several optimization problems found in economics, biological sciences, social sciences, computational linguistics, computational physics, engineering and health sciences, where the objective function captures the marginal gain, health, error, potential, energy, loss, coherence or stability, there are situations where one demands to seek the ``absolute optimum'' of the objective function. The situation is made more challenging when the objective function considered in these complex settings is either non-linear and non-convex or the derivative is hard to compute or in some cases the analytic form of the function itself is unavailable which in turn prevents the verification of certain structural properties. In these settings, the objective function might contain multiple local extrema with their cardinality unknown and the margin of quality of the local extrema with respect to the global extreme is significant. There are various problems of this kind which are available in the literature \cite{wolfram1996mathematica,jacob2001illustrating,mandelbrot1983fractal,murraymathematical,wolfram2002new}.\\

\section{Problem Statement and Background}
The problem of global optimization\cite{spall2005introduction,pinter2002global,hu2007model} aims to seek the input parameter vector which attains the global optimum of the objective function. The global optimization problem can be formally defined as follows:
\begin{equation}
	\textrm{Find } x^{*} \in \argmax_{x \in \mathcal{X} \subseteq \bbbr^{m}}\mathcal{H}(x),
\end{equation} 
where $\mathcal{H}:\bbbr^{m} \rightarrow \bbbr$ is a multi-modal, bounded real-valued, Borel-measurable function. Multi-modality property implies that the objective function has multiple local optima. Further, we assume that the objective function need not be continuous  nor a closed form expression of the objective function is available. We also tacitly assume that the solution set $\argmax_{x \in \mathcal{X}}\mathcal{H}(x)$ is at most finite.\\
\begin{figure}[h]
	\centering
	\includegraphics[scale=0.5]{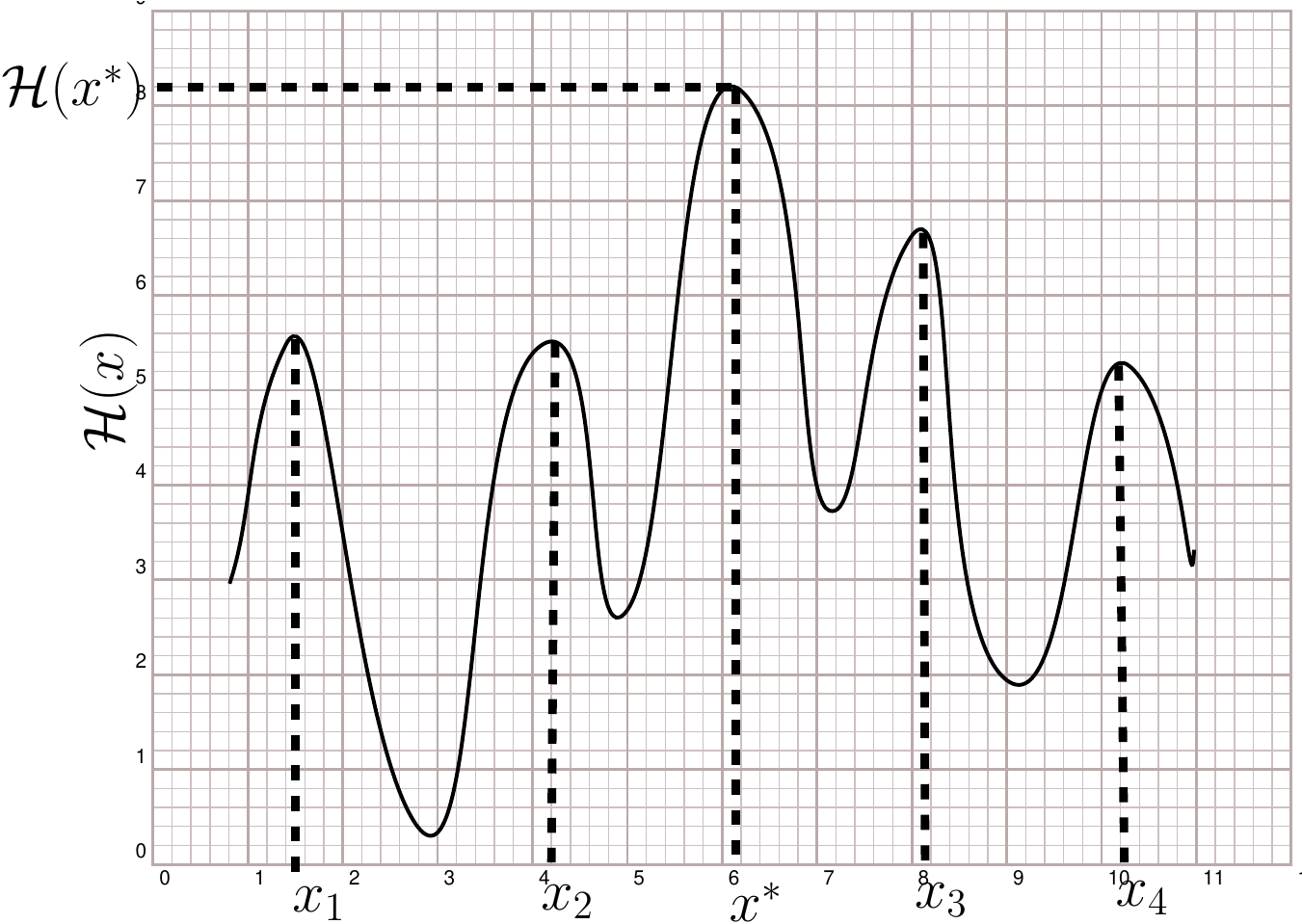}
	\caption{Example of a multi-modal function}
\end{figure}

The problem is inherently hard due to two primary reasons:
\begin{itemize}
	\item Presence of multiple local optima which considerably hinders the search for global optimum.
	\item A direct characterization (\emph{i.e.} analytic, closed form expression) of the global optimum is unavailable. Note that this is indeed a considerable deterrent. Any search technique requires a proper characterization of its goal which is critical in effectively guiding the search. For example, in the local gradient search methods, local optima can be characterized as the points where the gradient vanishes and hence the iterates can be guided in the direction where one can achieve this property. Apparently, global optima do not possess any efficient characterization.\\
\end{itemize}
\begin{figure}[t!]
	\begin{subfigure}[t]{0.43\textwidth}
		\centering
		\includegraphics[scale=0.38]{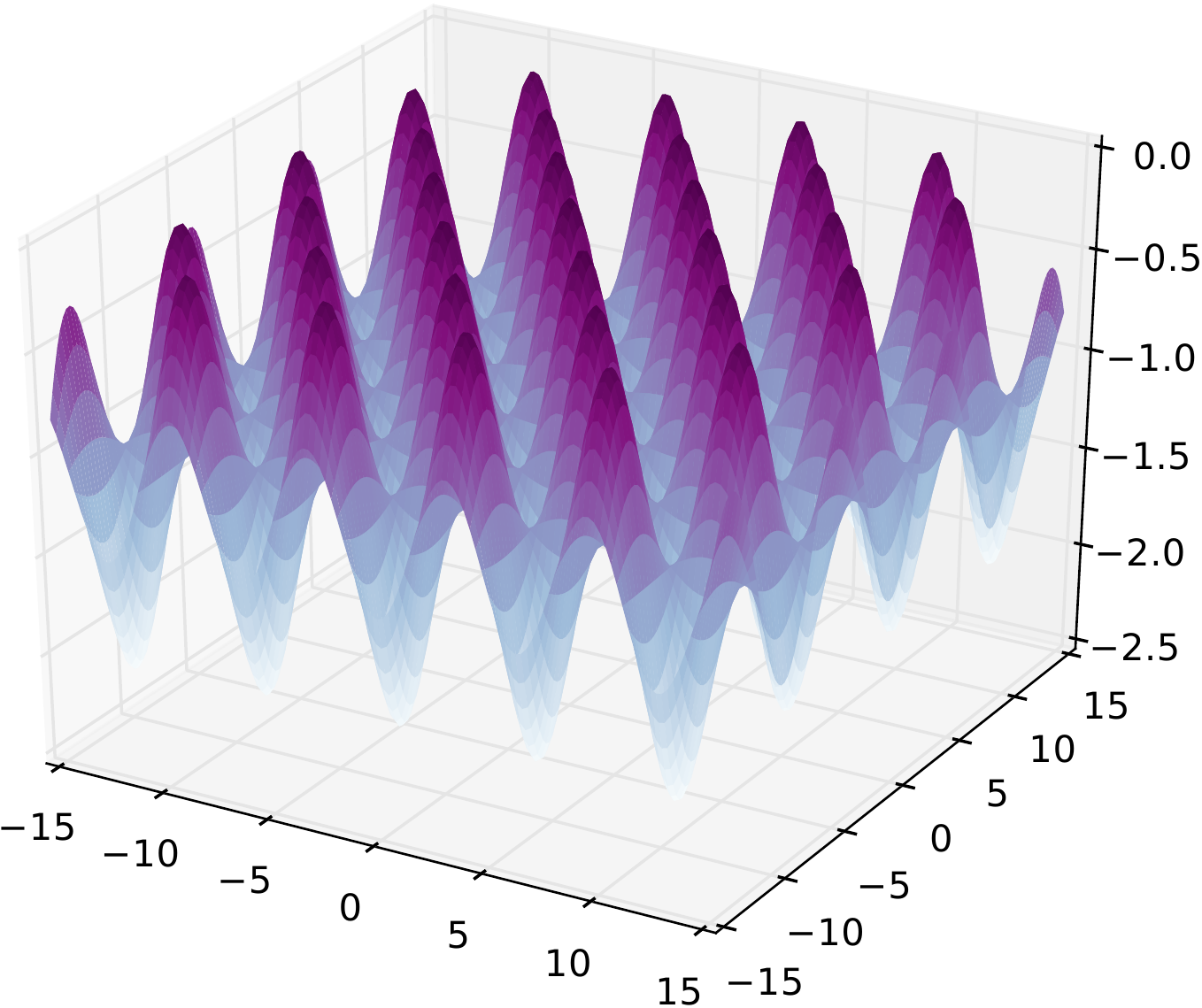}
		\subcaption{Griewank function on $\bbbr^{2}$ given by $\mathcal{H}_{1}(x) = -1-\frac{1}{4000}\sum_{i=1}^{2}x_{i}^{2}+\prod_{i=1}^{2}\cos{(x_i/\sqrt{i})}$}
	\end{subfigure}\hspace*{6mm}
	\begin{subfigure}[t]{0.45\textwidth}
		\centering
		\includegraphics[scale=0.22]{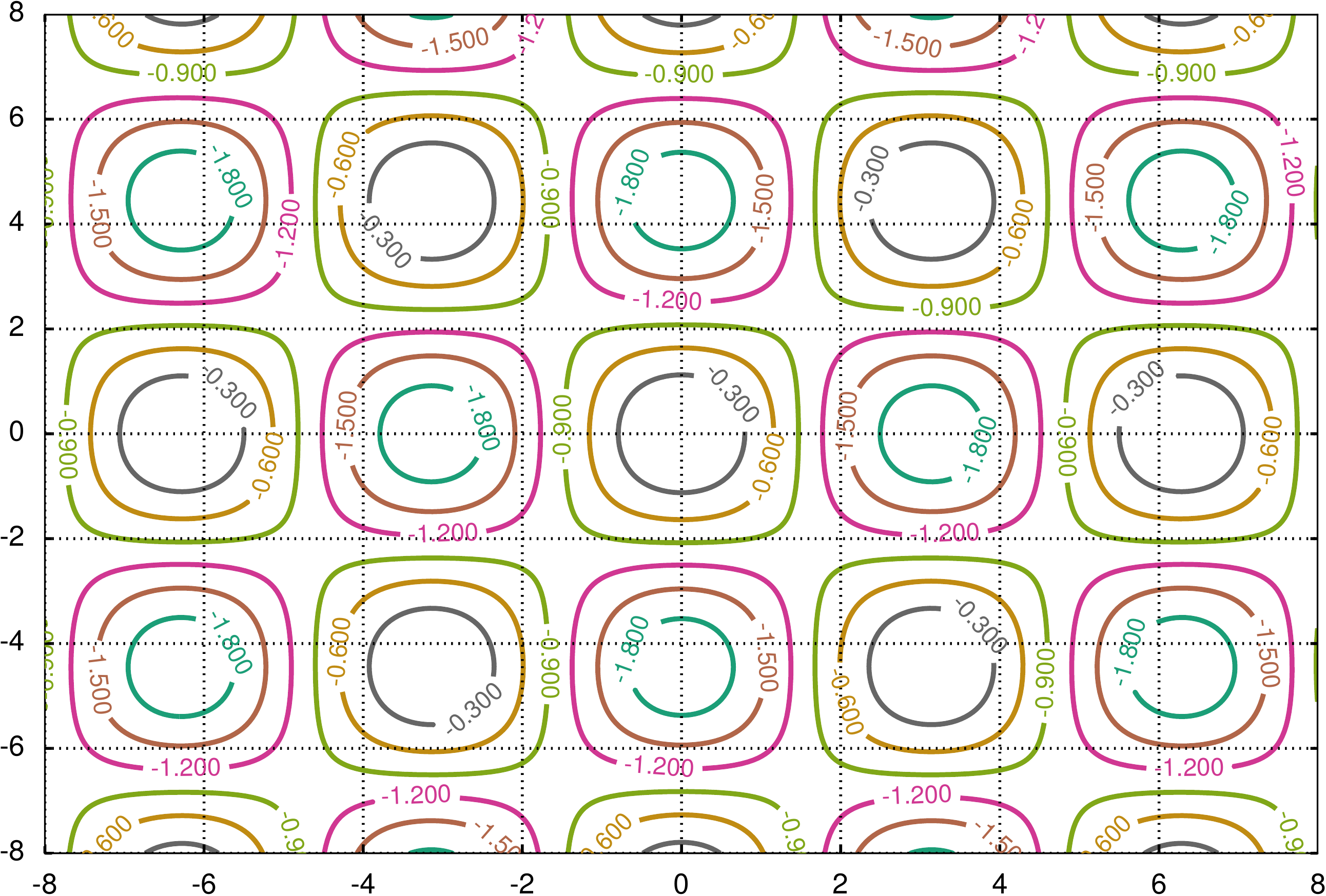}
		\subcaption{Contour map of the Griewank function which shows the high density of local optima}
	\end{subfigure}
\end{figure}
Local search methods like gradient search \cite{polyak1964gradient}, Newton method \cite{weerakoon2000variant}, variable metric method \cite{davidon1991variable} or conjugate gradient \cite{polyak1969conjugate} are local improvement algorithms which converge to the local optima of the objective function by utilizing local information like the gradient, Hessian, or other higher order derivatives of the objective function to guide their search. Hence they demand strict structural properties like smoothness (or finite times differentiability) of the objective function. Local optimization methods can find the global optimum only if the search is initiated close to the global optimum or the objective function is convex. Unfortunately, the above conditions are generally not satisfied in most practical applications. So to solve the global optimization problem, it is intuitive to presume that one has to either find an initial value close to the global optimum  or apodize the optimization  problem into a transformed setting, where the objective function is well-behaved and convex and then local optimization techniques can be applied to seek the global optimum. A practical solution to the former approach is to employ a chain of local search procedures, each solving a time dependent objective function with the solution obtained from each local search forming the initial value for the subsequent search. Smoothed functional schemes \cite{styblinski1990experiments,bhatnagar2007adaptive}, simulated annealing\cite{szu1986non,szu1987nonconvex}, simultaneous perturbation stochastic approximation (SPSA)\cite{maryak2001global, spall1992multivariate, spall2005introduction}, genetic algorithms \cite{fonseca1993genetic,gen2000genetic,golberg1989genetic} and tabu search \cite{glover2013tabu} are a few algorithms which belong to this category. Model based search methods which belong to the latter category are the primary topic of the paper, in particular the cross entropy method.

Model based search methods \cite{zlochin2004model} refer to a broad category of optimization methods which aim to generate a sequence of parametrized probability models $\{\theta_t\}$, where $\theta_t \in \Theta \subseteq \bbbr^{m}$ and each element $\theta_t$ of the sequence appropriately developed from the previous element $\theta_{t-1}$ and the model sequence satisfies the necessary property that it converges to the degenerate or singular model $\theta^{*} \in \Theta$ concentrated on the optimal value. By the degenerate distribution, we mean the Dirac measure with its entire mass at a single point in case when the global optimum is unique, otherwise (in case where the global optima form a finite set), the limiting model is the uniform distribution over the set of global optima. Model based search methods do not demand strong structural requirements on the  objective function and search is conducted by utilizing the objective function values. This is indeed an appealing feature and hence such methods can be applied to more general situations like the ``black box'' settings where the closed form expression of the objective function is absent, however function values of the input vectors are available. In the jargon of model based search, the search space $\mathcal{X}$ is referred to as the \textit{solution space}. The model based search methods are able to overcome the difficulty regarding the characterization of the global optima by defining the target as the degenerate probability distribution $\theta^{*} \in \Theta$ concentrated on the global optima. This enables the search to  be conducted on the model parameter space $\Theta$ with the ultimate goal to find the model parameter $\theta^{*}$. 

Model based search methods usually follow the ensuing framework:
\begin{enumerate}
	\item
	At time instant $t$, generate $N \in \mathbb{N}$ candidate solutions by drawing from the solution space $\mathcal{X}$ using the probability model defined by the parameter $\theta_t \in \Theta$.
	\item
	Using an update rule, deduce the new model parameter $\theta_{t+1} \in \Theta$ by utilizing the candidate solutions and their objective function values. The update rule is designed in such a way that  the new model parameter $\theta_{t+1}$ is effectively closer to the optimal probability model than its predecessor. Consequently, the quality of the candidate solutions drawn using the new probability model in the subsequent recursions will be better than those from its predecessor.
	\item
	Now set $t=t+1$ and go to step 1. Repeat until convergence. 
\end{enumerate}	
\begin{figure}[h]
	\centering
	\includegraphics[scale=0.5]{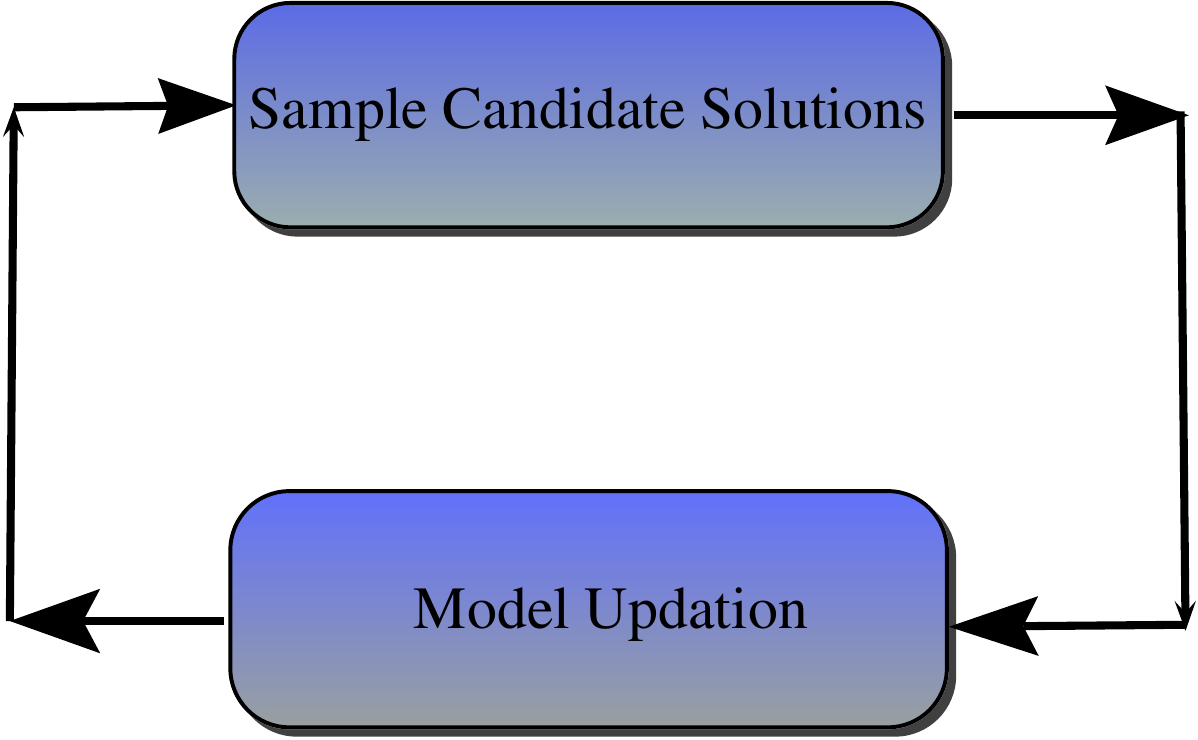}
	\caption{Flowchart of model based search method}
\end{figure}
An important and critical component of the model based search is the probability model subspace defined by $\Theta$. In order for the algorithm to provide good quality solutions, the probability model subspace is presumed to be rich enough where richness is defined in the sense that the subspace contains the degenerate probability distribution for each $x \in \mathcal{X}$, in particular $x^{*}$. 
\begin{figure}[h]
	\centering
	\includegraphics[scale=0.48]{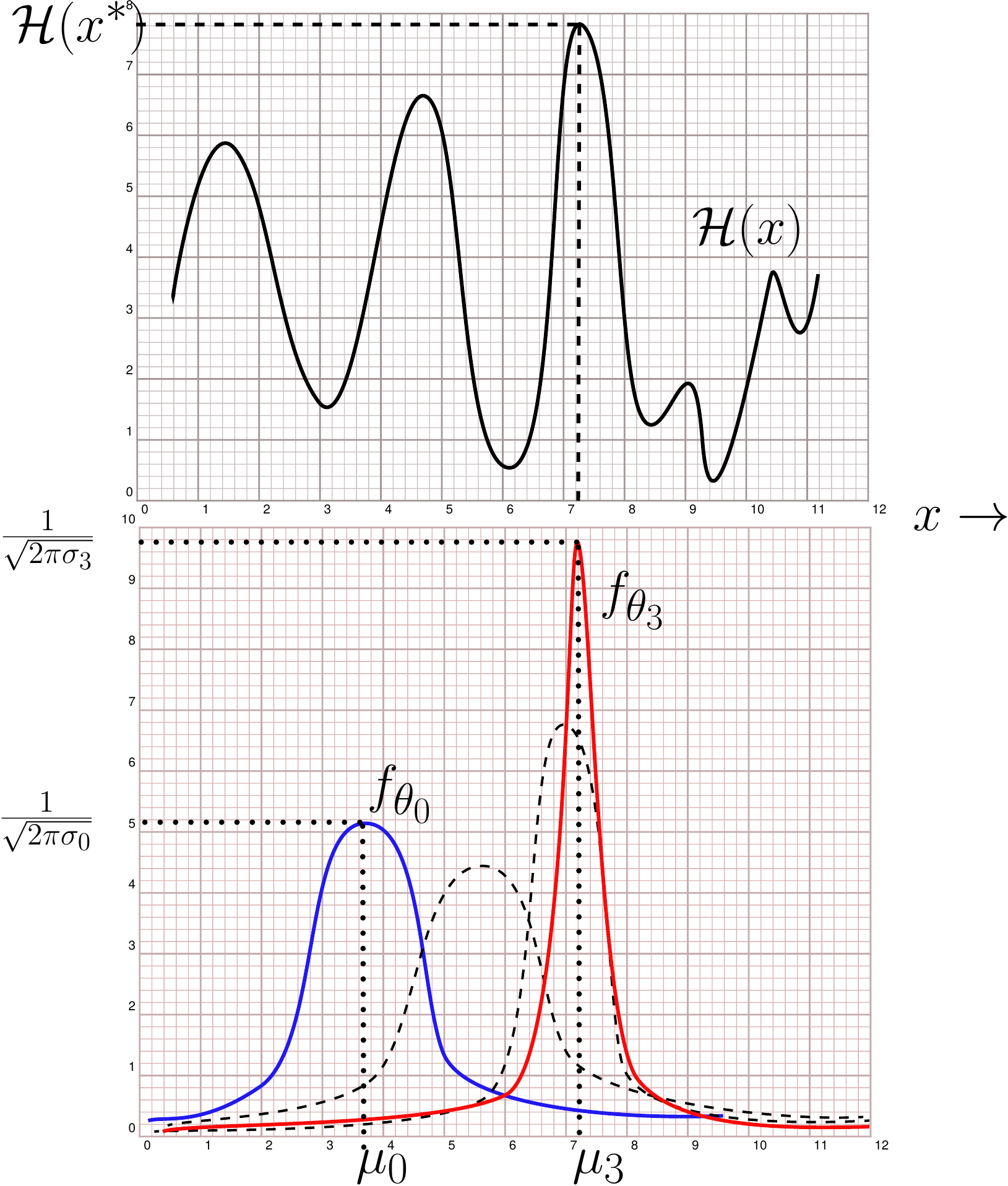}
	\caption{A candid demonstration of a particular instantiation of the model based search method. Here the Gaussian distribution parametrized by the tuple $\theta = (\mu, \sigma)^{\top}$, where $\mu \in \bbbr$  is the mean value and $\sigma \in \bbbr$  is the standard deviation. Here $\Theta \subseteq \bbbr^{2}$. And $f_{\theta}$ is the probability density function (PDF) of the Gaussian distribution defined by the parameter $\theta$. As illustrated here, the model sequence converges to the degenerate distribution concentrated at $x^{*}$, \emph{i.e.}, $\theta_t \rightarrow (x^{*}, 0)^{\top}$ as $t \rightarrow \infty$}
\end{figure}

Various algorithms exist in the optimization literature which belong to this class of methods. A few include model reference adaptive search (MRAS)\cite{hu2007model,fu2006model,chang2013simulation}, stochastic model reference adaptive search (SMRAS) \cite{hu2008model,chang2013simulation}, estimation of distribution algorithms (EDA) \cite{larranaga2002estimation,pelikan2002survey}, cross entropy (CE) method \cite{rubinstein1997optimization,rubinstein1999cross,kroese2006cross,rubinstein2001combinatorial,rubinstein2013cross,hu2012stochastic}, ant colony optimization (ACO) \cite{dorigo2006ant,socha2008ant,dorigo1999ant} and population based incremental learning \cite{baluja1994population}. The various algorithms differ in the way the update rule is defined. In this paper, the primary focus is on the well known cross entropy (CE) method. 

The CE method is a versatile Monte-Carlo technique used for estimation and optimization which was motivated from the adaptive variance reduction technique proposed for rare-event estimation\cite{rubinstein1997optimization}. Later this technique was adapted to design a combinatorial/discrete optimization algorithm by viewing the optimization procedure as a chain of inter-related rare-events. The proposed algorithm was initially used for solving various NP-hard problems like traveling sales man (TSP) problem \cite{de2005tutorial}, max-cut \cite{rubinstein2002cross,laguna2009hybridizing, de2005tutorial}, and graph bi-partitioning \cite{rubinstein2002cross}. The CE method has found applications far and wide. There is a rich literature on the applications of the CE method which include continuous multi-extremal optimization \cite{kroese2006cross,hu2012stochastic}, stochastic optimization \cite{he2010simulation,alon2005application,kroese2013cross}, constrained optimization \cite{kroese2013cross,kroese2006cross}, multi-objective optimization\cite{bekker2011cross}, network reliability optimization problem \cite{caserta2009cross,kroese2007network,ridder2005importance}, DNA sequence alignment \cite{keith2002rare}, power systems \cite{ernst2007cross}, buffer allocation \cite{alon2005application}, combinatorial auctions \cite{chan2008randomized}, network management \cite{wittner2003emergent}, machine learning \cite{kroese2007application,rubinstein2013cross}, queuing systems \cite{de2000analysis,de2004fast}, aeronautics \cite{wang2007double}, vehicle routing \cite{chepuri2005solving}, economic systems and policy analysis \cite{robinson2001updating,chan2010efficient,chan2012improved}, fuzzy control \cite{haber2010optimal}, social and biosciences \cite{sani2009stochastic,sani2008controlling}, telecommunication systems \cite{nariai2008cross,chen2009partial}, earth sciences \cite{you2005assessing}, hydraulics \cite{lind1989cross}, parameter estimation of ODEs \cite{wang2013parameter} and neural computation \cite{dubin2002application}, to name a few.

\section{Our Contributions}
A few enticing features of the cross entropy method are versatility, simplicity, robustness, flexibility, stability, non-dependency on the structural properties and zero-order operation. However, there are a few downsides too. One being the strong dependency on $N$, the cardinality of the set of candidate solutions at each iteration. The other drawback which is a serious hindrance in applying the CE method to a lot of interesting domains is its  offline nature. The offline property implies that the CE method can only be applied to settings, where, either the true objective function values or at least reasonable estimates are available with tolerable delay. But in many practical scenarios, the data arrives incrementally and the delay incurred in accumulating sufficient data in order to estimate the objective function values with reasonable error is often quite long. For example, in economic and financial systems, where various market indicators like the stock market indices, inflation rates, interest rates and production indices arrive sequentially with monthly or weekly delays and the prediction models have to evolve to accommodate the new information. Similar systems of this nature can be found in weather prediction systems, sequential decision making paradigms like the model free Markov decision processes as well as prediction models and computational inductions in health and social sciences, where data arrives sequentially in a delayed manner and hence the optimization involved has to operate in an incremental, online fashion by sequentially considering and accommodating the incoming data to guide the search. So our primary focus in the paper is to remodel the CE method to provide an online dimension without losing any of the attractive features like global optimum convergence and gradient-free operations which have made it so successful and appealing to the optimization community.\\

\noindent
\textbf{Notation: } We use $\mathsf{X}$ to denote a random variable and $x$ for deterministic variable. For $A \subset \mathbb{R}^{m}$,  $\mathbb{I}_{A}$ represents the indicator function of $A$, \emph{i.e.}, $\mathbb{I}_{A}(x) = 1$ if $x \in A$ and $0$ otherwise. Let $f_{\theta}(\cdot)$ denote the probability density function (PDF) parametrized by $\theta$. Also, $\P_{\theta}$ and $\mathbb{E}_{\theta}[\cdot]$ are the probability measure and expectation \emph{w.r.t.} the PDF $f_{\theta}$. For $\rho \in (0,1)$, let $\gamma_{\rho}(\mathcal{H}, \theta)$ denote the $(1-\rho)$-quantile of $\mathcal{H}(\mathsf{X})$ \emph{w.r.t.} the	 PDF $f_{\theta}$, \emph{i.e.},
\begin{equation}\label{eq:quantdef}
\gamma_{\rho}(\mathcal{H}, \theta) \triangleq \sup\{l: \P_{\theta}(\mathcal{H}(\mathsf{X}) \geq l) \geq \rho\}.
\end{equation}  
Let \hspace*{1mm}$supp(f) \triangleq \overline{\{x | f(x) \neq 0\}}$ denote the support of $f$ and interior($A$) be the \textit{interior} of the set $A$. Also $\lceil a \rceil$ denote the smallest integer greater than $a$. For $x \in \mathbb{R}^{m}$, let $\Vert x \Vert_{\infty}$ represent the sup-norm, \emph{i.e.}, $\Vert x \Vert_{\infty} = \max_i{\vert x_i \vert}$.

\section{Background and Motivation}
Recall that in this paper the continuous optimization problem that we consider is the following:
\begin{equation}\label{eqn:detoptprb}
	Find \hspace*{4mm} x^{*} \in \argmax_{x \in \mathcal{X} \subset \mathbb{R}^{m}}\mathcal{H}(x).  
\end{equation} Here $\mathcal{H}:\mathbb{R}^{m} \rightarrow \mathbb{R}$ is a deterministic, multi-modal, bounded real-valued function (\emph{i.e.}, $\mathcal{H}(x) \in [\mathcal{H}_{l}, \mathcal{H}_{u}], \forall x \in \mathcal{X}$, where $\mathcal{H}_{l}, \mathcal{H}_{u} \in \mathbb{R}$) and the solution space $\mathcal{X}$ is a compact subset of $\mathbb{R}^{m}$. We assume that $x^{*}$ is unique and $x^{*} \in \textrm{interior}(\mathcal{X})$. Note that the continuity of $\mathcal{H}$ implies that $\mathcal{H}(x^{*})$ is not an isolated point.

The cross entropy method seeks the best probability distribution which represents the global optimum from a given family of probability density functions $\{f_{\theta} \vert \theta \in \Theta\}$, where $f_{\theta}$ is a probability density function over the solution space $\mathcal{X}$ and $\Theta$ is a subset of a multi-dimensional Euclidean space.  A  family of distributions very commonly considered in this regard is the \textit{natural exponential family (NEF) of distributions}\cite{morris1982natural}. These probability distributions over $\bbbr^{m}$ are represented by  
\begin{equation}\label{eqn:nefdef}
\mathcal{F}_{exp} \triangleq \{f_{\theta}(x) = h(x)e^{\theta^{\top}\Gamma(x)-K(\theta)} \mid \theta\in \Theta \subset \bbbr^d\}, 
\end{equation}
where $h:\bbbr^{m} \longrightarrow \bbbr$, while $\Gamma:\bbbr^{m} \longrightarrow \bbbr^{d}$ is referred to as the \textit{sufficient statistic} and $K(\theta) = \log{\int{h(x)e^{\theta^{\top}\Gamma(x)}dx}}$ is called the \textit{cumulant function} of the family. The space $\Theta$ is chosen such that the cumulant function $K$ is finite, \emph{i.e.}, $\Theta = \{\theta \in \mathbb{R}^{d} \vert \hspace*{3mm}\vert K(\theta) \vert < \infty\}$. The space $\Theta$ is called the \textit{natural parameter space}. For a distribution belonging to NEF, there may exist multiple representations of the form (\ref{eqn:nefdef}). However, for the distribution, there definitely exists a representation where the components of the sufficient statistic are linearly independent and such a representation is referred to as \textit{minimal}. In this paper, we assume that the family is minimal. A few popular distributions which belong to the NEF family include Binomial, Poisson, Bernoulli, Gaussian, Geometric, Exponential distributions and their multivariate versions.

 A detailed description of the CE method is provided in \cite{kroese2013cross}. Here, we provide a succinct, yet comprehensive narrative of the CE method. If one observes the evolutionary nature of the CE method, then we  find that the algorithm generates a sequence of model parameters ${\{\theta_{t} \in \Theta\}}_{t \in \mathbb{N}}$ and an increasing sequence of thresholds ${\{\gamma_{t} \in \mathbb{R}\}}_{t \in \mathbb{N}}$ with efforts to direct the model sequence $\{\theta_t\}$ towards the degenerate distribution concentrated at $x^{*}$ and the threshold sequence $\{\gamma_t\}$ towards $\mathcal{H}(x^{*})$. A successful drift of the sequences towards the above mentioned limit points may or may not be achieved. It depends on the objective function and the quantile factor $\rho$. This particular characteristic of the algorithm will be discussed later in the paper. If one disregards other aesthetic aspects of the CE method, we find that the core component of the CE method is a recursion equation which is defined as follows:
\begin{equation}\label{eq:opt1}
	\theta_{t+1} = \argmax_{\theta \in \Theta}\mathbb{E}_{\theta_{t}}\left[S(\mathcal{H}(\mathsf{X}))\mathbb{I}_{\{\mathcal{H}(\mathsf{x}) \geq \gamma_{t+1}\}}\log{f_\theta(\mathsf{X})}\right],
\end{equation} 
where $S:\mathbb{R} \rightarrow \mathbb{R}_{+}$ is a positive and strictly monotonically increasing function. The most common choice for the threshold $\gamma_{t+1}$ is $\gamma_{\rho}(\mathcal{H}, \theta_{t})$: the $(1-\rho)$-quantile of $\mathcal{H}$ with respect to the PDF  $f_{\theta_{t}}$. Here, the quantile parameter $\rho \in (0,1)$ is set \textit{a priori} for the algorithm. Also, the parameter space $\Theta$ is assumed to be compact and it is chosen large enough so that the solution to the optimization problem (\ref{eq:opt1}) is contained in the $interior(\Theta)$ for all $t > 0$.

In this paper, we take the multivariate Gaussian distribution as the preferred choice for the distribution family $\{f_{\theta} \vert \theta \in \Theta\}$ of the CE method. In this case,
\begin{equation} \label{eq:gdist}
	f_{\theta}(x) = \frac{1}{\sqrt{(2\pi)^{m}|\Sigma|}}e^{-(x-\mu)^{\top}\Sigma^{-1}(x-\mu)/2},
\end{equation}
where $\mu \in \mathbb{R}^{m}$ is the  mean vector and $\Sigma \in \mathbb{R}^{m \times m}$ is the covariance matrix. Recall that multivariate Gaussian belongs to the natural exponential family (NEF) of distributions. So by letting
${\displaystyle h(x) = 1/(2\pi)^{m/2}}$ and $\Gamma(x) = (x, xx^{\top})^{\top}$, one obtains the natural NEF parametrization as  ${\displaystyle (\Sigma^{-1} \mu,\hspace*{1mm}-\frac{1}{2}\Sigma^{-1})^{\top}}$.

In the case of Gaussian based CE method, we let $\theta = (\mu, \Sigma)^{\top}$. Additionally, one can obtain a closed-form expression for $\theta_{t+1} = (\mu_{t+1}, \Sigma_{t+1})^{\top}$ by equating to $0$ the gradient of the objective function in Equation (\ref{eq:opt1}) with respect to the natural NEF parameter $(\Sigma^{-1} \mu, -\frac{1}{2}\Sigma^{-1})^{\top}$. Indeed, we obtain
\begin{flalign} 
	&\mu_{t+1} = \frac{\mathbb{E}_{\theta_{t}}\left[\mathsf{g_{1}}\big{(}\mathcal{H}(\mathsf{X}), \mathsf{X}, \gamma_{t+1}\big{)}\right]}{\mathbb{E}_{\theta_{t}}\left[\mathsf{g_{0}}(\mathcal{H}(\mathsf{X}), \gamma_{t+1})\right]} \triangleq \Upsilon_1(\theta_t, \gamma_{t+1}),\label{eq:sigmaideal1}\\
	&\Sigma_{t+1} = \frac{\mathbb{E}_{\theta_{t}}\left[\mathsf{g_{2}}\big{(}\mathcal{H}(\mathsf{X}), \mathsf{X}, \gamma_{t+1}, \Upsilon_{1}(\theta_t, \gamma_{t+1})\big{)}\right]}{\mathbb{E}_{\theta_{t}}\left[\mathsf{g_{0}}\big{(}\mathcal{H}(\mathsf{X}), \gamma_{t+1}\big{)}\right]} \triangleq \Upsilon_2(\theta_t, \gamma_{t+1}),\label{eq:sigmaideal2}
\end{flalign}
where
\begin{flalign}
	&\mathsf{g}_{0}(\mathcal{H}(x), \gamma) \triangleq S(\mathcal{H}(x))\mathbb{I}_{\{\mathcal{H}(x) \geq \gamma\}},\label{eqn:g0def}\\ &\mathsf{g_{1}}(\mathcal{H}(x), x, \gamma) \triangleq S(\mathcal{H}(x))\mathbb{I}_{\{\mathcal{H}(x) \geq \gamma\}}x  \textrm{ and } \label{eqn:g1def}\\ &\mathsf{g_{2}}(\mathcal{H}(x), x, \gamma, \mu) \triangleq S(\mathcal{H}(x))\mathbb{I}_{\{\mathcal{H}(x) \geq \gamma\}}(x-\mu)(x-\mu)^{\top}. \label{eqn:g2def}
\end{flalign}

\section{CE Method (Monte-Carlo Version)}
It is incredibly hard in general to compute the true value of the quantities $\mathbb{E}_{\theta_{t}}[\cdot]$ and $\gamma_{t+1}\hspace*{1mm} (=\gamma_{\rho}(\mathcal{H}, \theta_{t}))$ of Equations (\ref{eq:sigmaideal1}) and (\ref{eq:sigmaideal2}). To overcome this, a popular pragmatic approach is to employ their corresponding stochastic counterparts to track or estimate the corresponding true quantities. Here we maintain a user configured observation allocation rule $\{N_{t} \in \mathbb{N}\}_{t \in \mathbb{N}}$ to determine the sample size for each iteration of the CE method, where $N_{t}$ diverges to  $\infty$. In the Monte-Carlo version, the algorithm computes model sequences $\{\bar{\theta}_{t} = (\bar{\mu}_{t}, \bar{\Sigma}_{t})^{\top}\}_{t \in \mathbb{N}}$ and thresholds $\{\bar{\gamma}_{t} \in \bbbr\}_{t \in \mathbb{N}}$  using naive Monte-Carlo estimation. To elucidate further, at each iteration $t$, the Monte-Carlo version draws $N_{t}$ IID samples $\{\mathsf{X}_1, \mathsf{X}_2, \dots, \mathsf{X}_{N_{t}}\}$ from the solution space $\mathcal{X}$ using the PDF $f_{\bar{\theta}_{t}}$ and the threshold estimate $\bar{\gamma}_{t+1}$ is computed as follows: 
\begin{equation}\label{eqn:quantmcest}
	\bar{\gamma}_{t+1} = \mathcal{H}_{(\lceil(1-\rho)N_{t}\rceil)}, 
\end{equation}
where $\mathcal{H}_{(i)}$ is the $i$th-order statistic of $\{\mathcal{H}(\mathsf{X}_i)\}_{i=1}^{N_{t}}$.\\

The update of the model parameters also employs sample average estimates of $\Upsilon_1$ and $\Upsilon_2$ respectively. The model parameter $\bar{\theta}_{t+1} = (\bar{\mu}_{t+1}, \bar{\Sigma}_{t+1})^{\top}$  of the Monte-Carlo version is updated as follows:
\begin{flalign}
	&\bar{\mu}_{t+1} = \frac{\frac{1}{N_{t}}\sum_{i=1}^{N_{t}}\mathsf{g_{1}}(\mathcal{H}(\mathsf{X}_{i}), \mathsf{X}_{i}, \bar{\gamma}_{t+1})}{\frac{1}{N_{t}}\sum_{i=1}^{N_{t}}\mathsf{g_{0}}(\mathcal{H}(\mathsf{X}_{i}), \bar{\gamma}_{t+1})} \triangleq \bar{\Upsilon}_{1}(\bar{\theta}_{t}, \bar{\gamma}_{t+1}), \label{eqn:mcvers1}\\\nonumber\\
	&\bar{\Sigma}_{t+1} = \frac{\frac{1}{N_{t}}\sum_{i=1}^{N_{t}}\mathsf{g_{2}}(\mathcal{H}(\mathsf{X}_{i}), \mathsf{X}_{i}, \bar{\gamma}_{t+1}, \bar{\Upsilon}_{1}(\bar{\theta}_{t}, \bar{\gamma}_{t+1}))}{\frac{1}{N_{t}}\sum_{i=1}^{N_{t}}\mathsf{g_{0}}(\mathcal{H}(\mathsf{X}_{i}), \bar{\gamma}_{t+1})} \triangleq \bar{\Upsilon}_2(\bar{\theta}_{t}, \bar{\gamma}_{t+1}).\label{eqn:mcvers2}
\end{flalign}
Here we reuse the same IID samples $\{\mathsf{X}_{i}\}_{i=1}^{N_{t}}$  drawn for the quantile estimation.\\

The Monte-Carlo version of the CE method is illustrated in Algorithm \ref{algo:cemc}.\\\\

\hspace*{-6mm}\begin{minipage}{1.0\linewidth}
	\begin{algorithm}[H]
		\vspace*{5mm}
		\textbf{Initialization:} Choose an initial PDF $f_{\bar{\theta}_0}(\cdot)$ on $\mathcal{X}$, where $\bar{\theta}_{0} = (\bar{\mu}_{0}, \bar{\Sigma}_{0})^{\top}$; Fix an $\epsilon > 0$; Set $t = 0$; $\gamma_{0}^{*} = -\infty$;\vspace*{2mm}\\
		\textbf{Sampling Candidate Solutions:} Sample $N_{t}$  IID solutions $\{\mathsf{X}_1, \dots, \mathsf{X}_{N_{t}}\}$ from the solution space $\mathcal{X}$ using $f_{\bar{\theta}_{t}}$.\vspace*{3mm}\\
		\textbf{Threshold Evaluation:} Calculate the sample $(1-\rho)$-quantile $\bar{\gamma}_{t+1} = \mathcal{H}_{(\lceil (1-\rho)N_{t}\rceil)}$, where $\mathcal{H}_{(i)}$ is the $i$th-order statistic of the sequence $\{\mathcal{H}(\mathsf{X}_i)\}_{i=1}^{N_{t}}$.\vspace*{3mm}\\
		\nonl\textbf{Threshold Comparison:}\\
		\eIf{$\bar{\gamma}_{t+1} \geq \bar{\gamma}^{*}_{t}+\epsilon$}{
			$\bar{\gamma}^{*}_{t+1} = \bar{\gamma}_{t+1}$,
		}{
			$\bar{\gamma}^{*}_{t+1} = \bar{\gamma}^{*}_{t}$.
		}\vspace*{3mm}
		\textbf{Model Parameter Update:}  
		Update $\bar{\theta}_{t+1}  =  (\bar{\Upsilon}_{1}(\bar{\theta}_{t}, \bar{\gamma}^{*}_{t+1}), \bar{\Upsilon}_{2}(\bar{\theta}_{t}, \bar{\gamma}^{*}_{t+1}))^{\top}$.\vspace*{3mm}\\
		If the stopping rule is satisfied, then return $\bar{\theta}_{t+1}$ and terminate, else set $t=t+1$ and go to Step $2$.
		\vspace*{5mm}
		\caption{The Monte-Carlo CE Algorithm\label{algo:cemc}}
	\end{algorithm}
\end{minipage}
\subsection{Drawbacks of the Monte-Carlo CE Method}
\begin{enumerate}
	\item
	\textit{Inefficient use of prior information: }
	The naive approach of the Monte-Carlo CE does not utilize prior information efficiently while tracking the ideal CE method. Note that Monte-Carlo CE possesses a stateless behaviour. Indeed, at each iteration $t$, a completely new collection of samples is drawn from the solution space using the distribution $f_{\bar{\theta}_{t}}$. The collection of samples is then used to compute the estimates $\bar{\gamma}_{t+1}$ and $\bar{\theta}_{t+1}$. It is thus trivial to observe that the Monte-Carlo algorithm does not indeed utilize the estimates or samples generated prior to $t$.
	\item
	\textit{Computational limitations}:
	These arise due to the dependence of the algorithm on the sample size $N_{t}$. One does not know a priori  \textit{the best value for the $N_{t}$}. Higher values of $N_{t}$ while resulting in higher accuracy also require more computational resources. One often needs to apply brute force in order to obtain a good choice of $N_{t}$. Also as $m$, the dimension of the solution space $\mathcal{X}$, takes large values, more samples are required for better accuracy, making \textit{$N_{t}$ large} as well. This makes \textit{finding the $i$th-order statistic $\mathcal{H}_{(i)}$ in Step 3 harder}. Note that the order statistic $\mathcal{H}_{(i)}$ is obtained by sorting the list $\{\mathcal{H}(\mathsf{X}_1), \mathcal{H}(\mathsf{X}_2), \dots \mathcal{H}(\mathsf{X}_{N_{t}})\}$. The computational effort required in that case is at least $O(N_{t}\log{N_{t}})$ (the lower bound for sorting) which is indeed computationally expensive. Also note that $N_{t}$ diverges to infinity and hence this super linear relationship is computationally very expensive.
	\item
	\textit{Storage limitations}: The storage requirement at each iteration $t$ for storing the sample collection is $N_{t}m$. In situations when $m$ and $N_{t}$ are large, the storage requirement is a major concern.
\end{enumerate}

An illustration in Fig. \ref{fig:samplediffdet} demonstrates the dependency of the performance of Monte-Carlo CE on the sample size schedule $\{N_t, t \geq 0\}$. Here, we consider the Griewank function on $\mathbb{R}^{80}$, \emph{i.e.}, $\mathcal{H}(x) = -1-\frac{1}{4000}\sum_{i=1}^{80}x_{i}^{2}+\prod_{i=1}^{80}\cos{(x_i/\sqrt{i})}$. We take $N_{t+1} = \lceil\eta N_t\rceil$, where  $\eta > 1$. So a particular schedule can be identified by the pair $(N_0, \eta)$. Here we take $\eta = 1.005$ for all the schedules, however they differ in their initial value $N_0$. From the results plotted in Fig. \ref{fig:samplediffdet}, one can observe that the performance of the CE method improves as the initial sample size $N_0$ takes larger values, \emph{i.e.}, more samples are considered by the algorithm.
Different variants of the naive Monte-Carlo CE have been proposed in the literature, such as the gradient based Monte-Carlo cross entropy method (GMCCE) \cite{hu2012stochastic} and the modified Monte-Carlo cross entropy method \cite{wang2013parameter}. All the variants differ only in the model updating step, the other steps remain the same. Hence they also suffer from the above drawbacks.
\begin{figure}[h]
	\centering
	\includegraphics[width=85mm, height=60mm]{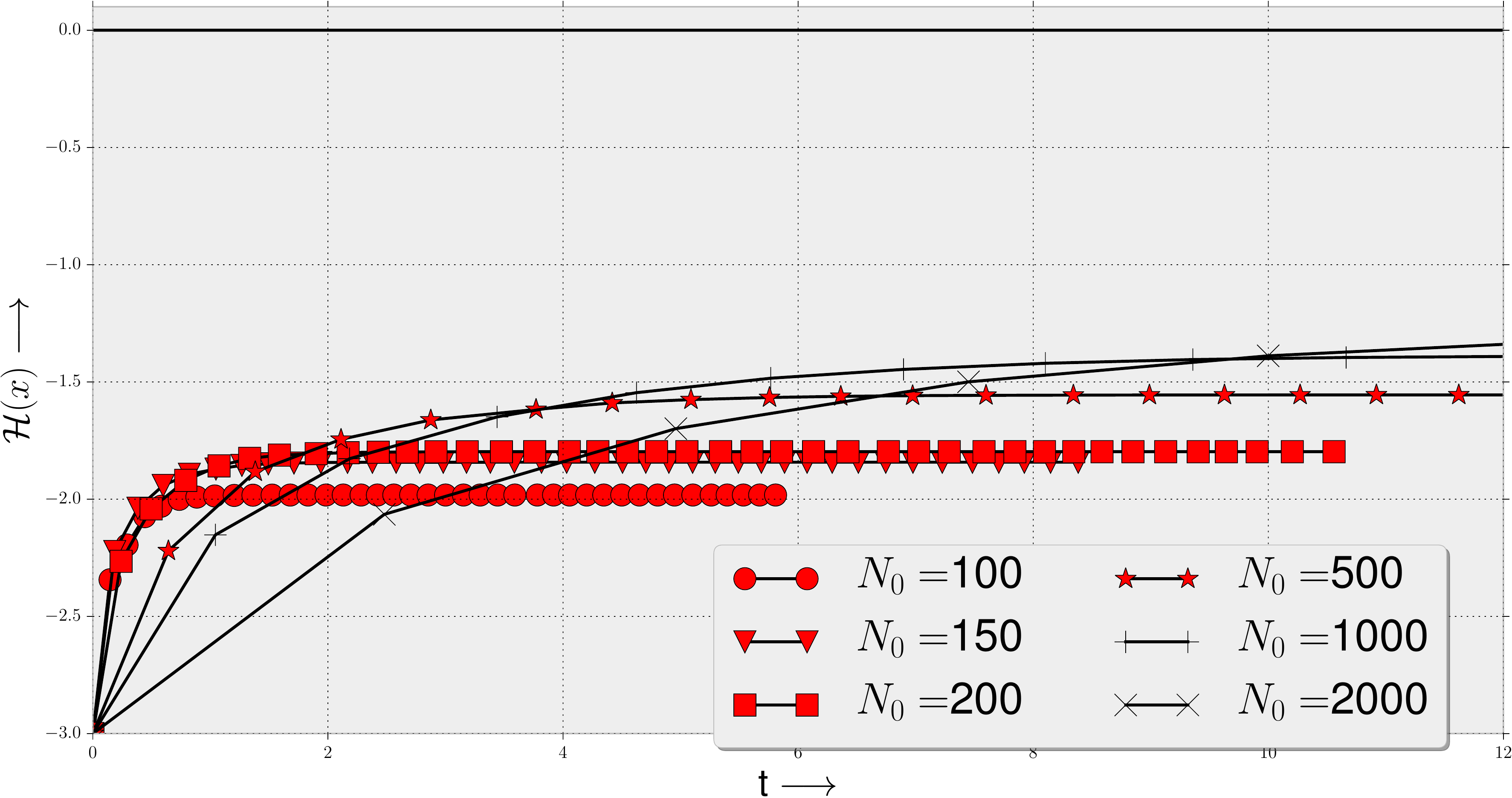}
	\caption{Plot of $\mathcal{H}(\bar{\mu}_{t})$, where $\mathcal{H}$ is the Griewank function. The plot shows the dependency of Monte-Carlo CE on the sample size schedules $\{N_t\}$.}\label{fig:samplediffdet}
\end{figure}
\section{Proposed Algorithm: CE2-ND}
The above mentioned drawbacks on the inefficient information utilization and the heavy cost on the space and computational requirements are primarily attributed to the non-incremental, batch based and stateless nature of the algorithm. In this paper, we resolve these shortcomings of the CE algorithm by remodeling the same under the stochastic approximation framework. We \textit{replace the sample averaging with a bootstrapping approach, \emph{i.e.}, derive new estimates using the current estimates} and thus in effect we achieve geometric averaging of the entire history of past estimates. The algorithm which we call CE2-ND (acronym for cross entropy $2$ with normal distribution) possesses  various features which we find desirable: 
\begin{enumerate}
	\item
	Stable, robust and easy to implement.
	\item
	Computational-wise and storage efficient.
	\item
	Limited restriction on the objective function, \emph{i.e.}, imposition of  very minimal structural restrictions on the objective function.
	\item
	Incremental in nature, \emph{i.e.},  evolves at each time instant according to the sample data (the function value $\mathcal{H}(\cdot)$) available at that particular instant. In other words the solution is built incrementally. 
	\item
	Efficient use of prior information, \emph{i.e.}, the algorithm adopts an adaptive nature where the function values $\mathcal{H}(\cdot)$ are requested only when required. The bootstrapping nature of the algorithm guarantees a continuous evolution (in contrast to the stateless nature of the Monte-Carlo version) and hence no data or prior information is under-utilized. 
	\item
	Convergence to the global optimum.
	A recent study \cite{hu2012stochastic,hu2007model} shows that the CE method is only a local improvement (local optimization) algorithm. In \cite{hu2012stochastic}, a few counter examples are also provided where the CE method fails to converge to the global optimum. But in many practical cases, the CE method is shown to produce high quality solutions. In this paper, we explore this dichotomy and propose a mixture model approach. We  provide a proof of convergence to the global optimum, for the novel mixture based CE method for a particular class of objective functions.
\end{enumerate}

\subsection{Anatomy of CE2-ND}\label{sec:anatomyce2nd}
The   suboptimal behaviour of the Monte-Carlo CE method both in terms of resource utilization and performance, especially in higher dimensional cases is primarily attributed to its batch based approach, \emph{i.e.}, processing batches of sample solutions at each iteration to compute estimates. We propose a novel approach, where we efficiently and effectively interleave the averaging of the samples to obtain an algorithm which not only asymptotically tracks the ideal CE method but also streamlines the whole procedure. In our approach, we  employ the well known and efficient stochastic approximation framework (discussed in Section \ref{sec:saframework}) to interleave the various averaging tasks. Our goal is to effectively track the ideal CE method using recursions of the kind (\ref{eq:strec}).

Note that the primary quantities of interest in the ideal CE method are $\gamma_t$, $\Upsilon_1$, $\Upsilon_2$ and $\theta_t$. In our algorithm, we track these quantities independently using stochastic recursions of the kind (\ref{eq:strec}). Thus we model our algorithm as a stochastic approximation algorithm containing multiple stochastic recursions operating in tandem to produce an equilibrium behaviour which is equivalent to the ideal CE method. Recall that a stochastic recursion is uniquely identified by its increment term, its initial value and the learning rate. We examine here the various recursions in great detail.\\\\
\textbf{$1.$ Tracking $\gamma_{\rho}(\mathcal{H}, \theta)$:} In our algorithm, we do not apply the naive order statistic method to estimate the $(1-\rho)$-quantile. Rather we employ an efficient stochastic recursion to estimate the $(1-\rho)$-quantile which is based on the following lemma:

The quantile problem is reformulated as an optimization problem in Lemma 1 of \cite{homem2007study}. The lemma provides a characterization of the $(1-\rho)$-quantile of a given real-valued function $\mathcal{H}$ with respect to a given probability measure $\P$. For better comprehension, we restate the lemma here:
\begin{lemma}\label{lma:qnopt} (Lemma 1 of \cite{homem2007study})
	The $(1-\rho)$-quantile of a bounded real valued function $\mathcal{H}(\cdot)$ $\Big(\textrm{with } \mathcal{H}(x) \in [\mathcal{H}_l, \mathcal{H}_u]\Big)$ with respect to the probability density function $f_{\theta}$ is reformulated as an optimization problem 
	\begin{eqnarray}\label{eqn:gmoptprob}
	\gamma_{\rho}(\mathcal{H}, \theta) = \argmin_{\gamma \in [\mathcal{H}_l, \mathcal{H}_u]}\mathbb{E}_{\theta}\left[\psi(\mathcal{H}(\mathsf{X}), \gamma)\right], 
	\end{eqnarray}
	where   $\mathsf{X} \sim f_{\theta}$, $\psi(\mathcal{H}(x), \gamma) = (1-\rho)(\mathcal{H}(x)-\gamma)\mathbb{I}_{\{\mathcal{H}(x) \geq \gamma\}} + \rho(\gamma-\mathcal{H}(x))\mathbb{I}_{\{\mathcal{H}(x) \leq \gamma \}}$ and $\mathbb{E}_{\theta}[\cdot]$ is the expectation with respect to the PDF $f_{\theta}$.
\end{lemma}

 In this paper, we maintain the time-dependent variable $\gamma_{t}$ to track  the true quantile $\gamma_{\rho}(\mathcal{H}, \cdot)$. 
 The increment term in the recursion is the sub-differential of $\psi$ with respect to $\gamma$ which is defined as follows:
\begin{equation}\label{eq:gmgradinc}
	\begin{aligned}
		\Delta \gamma_{t}(x) = -(1-\rho)\mathbb{I}_{\{\mathcal{H}(x) \geq \gamma_t\}}+\rho \mathbb{I}_{\{\mathcal{H}(x) \leq \gamma_t\}}.\\
	\end{aligned}
\end{equation}
The stochastic recursion which tracks $\gamma_{\rho}(\mathcal{H}, \cdot)$ is defined in Equation (\ref{eqn:ce2ndgamma}). The continuity relationship that holds between the $(1-\rho)$-quantile $\gamma_{\rho}(\mathcal{H}, \theta)$ and the model parameter $\theta$ is indeed beneficial since the evolutionary approach inherent in the stochastic approximation techniques effectively utilizes the relationship.\vspace*{2mm}\\
$2$. \textbf{Tracking $\Upsilon_1$ and $\Upsilon_2$}: In our algorithm, we completely avoid the sample averaging technique employed in the Monte-Carlo version to estimate $\Upsilon_1$ and $\Upsilon_2$. Rather, we employ the stochastic approximation recursion to track the above quantities. We maintain two time-dependent variables $\xi^{(0)}_t$ and $\xi^{(1)}_t$  to track $\Upsilon_1$ and $\Upsilon_2$ respectively. The increment functions used by their respective recursions are defined as follows:
\begin{flalign}
	&\Delta\xi^{(0)}_{t}(x) = \mathsf{g_{1}}(\mathcal{H}(x), x, \gamma_t) - \xi^{(0)}_{t}\mathsf{g_{0}}(\mathcal{H}(x), \gamma_t),\\\nonumber\\
	&\Delta\xi^{(1)}_{t}(x) = \mathsf{g_{2}}(\mathcal{H}(x), x, \gamma_t, \xi^{(0)}_{t}) - \xi^{(1)}_{t}\mathsf{g_{0}}(\mathcal{H}(x), \gamma_t).
\end{flalign}

The stochastic recursions which track $\Upsilon_1$ and $\Upsilon_2$ are defined in Equations (\ref{eqn:ce2ndxi0}) and (\ref{eqn:ce2ndxi1}) respectively. Note that the recursion of $\xi^{(0)}_t$ depends on $\gamma_t$, while that of $\xi^{(1)}_{t}$ depends on $\gamma_t$ and $\xi^{(0)}_{t}$. Also note that the recursion of $\gamma_t$ is independent of $\xi^{(0)}_{t}$ and $\xi^{(1)}_{t}$. This implies that there exists only a unilateral coupling between these quantities which thus enables us to use the same learning rate parameter for all the three  recursions. \vspace*{2mm}\\

\noindent
\textbf{$3$. Model Parameter Update: } In the ideal version of CE, we have $\theta_{t+1} = (\Upsilon_1(\theta_t, \dots), \Upsilon_2(\theta_t, \dots))^{\top}$. This can be a large discrete change from $\theta_t$ to $\theta_{t+1}$. But in our algorithm, we adopt a smooth update of the model parameters. The recursion is defined in Equation (\ref{eq:thupd}). This smooth update is practically significant since it guarantees a continuous evolution of the model parameters contrary to a discrete change which might cause large deviations. This further prevents premature convergence of the model sequence to any of the suboptimal solutions. A technical detail which has to be mentioned here is that the proposed algorithm does not update the model parameter $\theta_t$ at each time instant $t$, rather it is updated only along a subsequence of $\{t\}$. We block the update of model parameters by utilizing a delaying mechanism whose technical details we elaborate later in this paper. This delaying mechanism also enables us to reuse the same learning rate for the $\theta_t$ recursion as the one used for $\gamma_t$, $\xi^{(0)}_{t}$ and $\xi^{(1)}_{t}$.  This is indeed significant since the samples processed by the recursions of $\gamma_t$, $\xi^{(0)}_{t}$ and $\xi^{(1)}_{t}$ are controlled by $\theta_t$ and hence a direct implication reveals there is a bilateral coupling between $\theta_t$ and the rest of the quantities. However, because of the delaying mechanism, sufficient averaging of the quantities $\gamma_t$, $\xi^{(0)}_{t}$ and $\xi^{(1)}_{t}$ occurs between any two successive updates of the model parameter. So one can indeed reuse the same learning rate for the $\theta_t$ recursion. \vspace*{2mm}\\
\textbf{$4.$ Learning Rates: }Our algorithm uses a single learning rate $\{\beta_{t}\}$, which satisfies the following condition:
\begin{assm}\label{assm:lnrtce2nd}
	The step-size $\{\beta_{t}\}$ is a deterministic, positive and non-increasing sequence which satisfies
	\begin{equation}\label{eqn:ce2ndlearnrt}
		\sum_{t=1}^{\infty}\beta_{t} = \infty, \hspace*{5mm} \sum_{t=1}^{\infty}\beta^{2}_{t} < \infty.
	\end{equation}
\end{assm}
\noindent
\textbf{$5.$ Mixture Distribution:} The streamlined nature inherent in the stochastic approximation algorithms demands only a single sample $\mathsf{X}_{t+1}$ (generated in Equation (\ref{eq:samplegen})) per iteration.  This is a remarkable improvement in the sense that the algorithm learns by utilizing only a single sample $\mathsf{X}_{t+1}$ per iteration to evolve the variables involved and thus directing the model parameter $\theta_t$ towards the degenerate distribution concentrated on the optimum point $x^{*}$. In our algorithm, we use a mixture distribution $\widehat{f}_{\theta_{t}}$  to generate the sample $\mathsf{X}_{t+1}$, where $\widehat{f}_{\theta_t} = (1-\lambda)f_{\theta_t} + \lambda f_{\theta_0}$ with $\lambda \in [0,1)$ the mixing weight. The initial distribution parameter $\theta_0$ is chosen such that the density function $f_{\theta_0}$ is strictly positive on every point in the solution space $\mathcal{X}$, \emph{i.e.}, $f_{\theta_0}(x) > 0, \forall x \in \mathcal{X}$. The mixture approach in fact facilitates extensive exploration of the solution space and prevents the iterates from getting stranded in suboptimal solutions.\\

\noindent\textbf{Notation: } We denote by $\gamma_{\rho}(\mathcal{H}, \widehat{\theta}))$, the $(1-\rho)$-quantile of $\mathcal{H}(\cdot)$ w.r.t. the mixture distribution $\widehat{f}_{\theta}$ and let $E_{\widehat{\theta}}[\cdot]$ be the expectation w.r.t. the mixture PDF $\widehat{f}_{\theta}$. Also, $\P_{\widehat{\theta}}$ is the probability measure w.r.t the mixture PDF $\widehat{f}_{\theta}$, \emph{i.e.}, for a Borel set $A \subset \bbbr^{m}$, we have $\P_{\widehat{\theta}}(A) = \int_{A}\widehat{f}_{\theta}(x)dx$.\\\\
\noindent
\scalebox{0.9}{
	\begin{minipage}{1.12\linewidth}
		\begin{algorithm}[H]
			\textbf{Data: } $\epsilon_1 \in (0,1), \lambda, c_t \in (0,1)$, $\beta_{t}$, $\theta_{0} = (\mu_0, \Sigma_0)^{\top}$.\\
			\textbf{Init:} $\gamma_0 = 0$, $\xi^{(0)}_0=0_{m \times 1}$, $\xi^{(1)}_{0}=0_{m \times m}$, $T_{0}=0$, $\gamma^{p}_{0} = -\infty$, $c = c_0$, $t = 0$, $\theta^{p} = NULL$.\\\vspace*{4mm}
			\While{stopping criteria not satisfied}{\vspace*{1mm}
				{\setlength{\abovedisplayskip}{-14pt}\begin{flalign}\label{eq:samplegen}
						&\textbf{Sample generation:}\hspace*{11cm}\nonumber\\	
						&\mathsf{X}_{t+1}  \sim \widehat{f}_{\theta_{t}}(\cdot) \hspace*{2mm}\mathrm{ where }\hspace*{2mm} \widehat{f}_{\theta_{t}} = (1-\lambda)f_{\theta_t} + \lambda f_{\theta_0};
				\end{flalign}}\\
				{\setlength{\abovedisplayskip}{-14pt}\begin{flalign}\label{eqn:ce2ndgamma}
						&\textbf{Tracking the $(1-\rho)$-quantile of $\mathcal{H}(\cdot)$ w.r.t. $\widehat{f}_{\theta_{t}}$:}\hspace*{65mm}\nonumber\\
						&\gamma_{t+1} = \gamma_{t} - \beta_{t+1} \Delta \gamma_{t}(\mathsf{X}_{t+1}));
				\end{flalign}}\\
				{\setlength{\abovedisplayskip}{-14pt}\begin{flalign}\label{eqn:ce2ndxi0}
						&\textbf{Tracking $\Upsilon_1$ of Equation (\ref{eq:sigmaideal1}):}\hspace*{80mm}\nonumber\\
						&\xi^{(0)}_{t+1} = \xi^{(0)}_{t} + \beta_{t+1}\Delta \xi^{(0)}_t(\mathsf{X}_{t+1});
				\end{flalign}}\\
				{\setlength{\abovedisplayskip}{-14pt}\begin{flalign}\label{eqn:ce2ndxi1}
						&\textbf{Tracking $\Upsilon_2$ of Equation (\ref{eq:sigmaideal2}):}\hspace*{80mm}\nonumber\\
						&\xi^{(1)}_{t+1} = \xi^{(1)}_{t} + \beta_{t+1} \Delta \xi^{(1)}_t(\mathsf{X}_{t+1});
				\end{flalign}}
				\If{$\theta^{p}$ $\neq$ $NULL$}{
					{\setlength{\abovedisplayskip}{-14pt}\begin{flalign}\label{eqn:ce2ndgammap}
							&{\mathsf{X}}^{p}_{t+1}   \sim \widehat{f}_{\theta^{p}} \textrm{ where }
							\widehat{f}_{\theta^{p}} = (1-\lambda)f_{\theta^{p}} + \lambda f_{\theta_{0}};\nonumber\\
							&\gamma^{p}_{t+1} = \gamma^{p}_{t} - \beta_{t+1} \Delta \gamma^{p}_{t}(\mathsf{X}^{p}_{t+1});\\
							&\mathrm{\underline{Note:}}\hspace*{2mm} \Delta \gamma^{p}_{t}\hspace*{2mm} \mathrm{ is\hspace*{2mm} same\hspace*{2mm} as }\hspace*{2mm} \Delta \gamma_{t} \hspace*{2mm} \mathrm{except} \hspace*{2mm} \gamma^{p}_{t} \hspace*{2mm} \mathrm{ replacing}\hspace*{2mm} \gamma_{t}.\nonumber
					\end{flalign}}
				}
				{\setlength{\abovedisplayskip}{-14pt}\begin{flalign}\label{eqn:ce2ndTt}
						&\textbf{Threshold comparison:}\hspace*{100mm}\nonumber\\
						&T_{t+1} = T_{t} + c \left(\mathbb{I}_{\{\gamma_{t+1} > \gamma^{p}_{t+1}\}} - \mathbb{I}_{\{\gamma_{t+1} \leq \gamma^{p}_{t+1}\}} - T_{t}\right);
				\end{flalign}}
				\eIf{$T_{t+1} > \epsilon_1$}{
					{\setlength{\abovedisplayskip}{-10pt}\setlength{\belowdisplayskip}{4pt}\begin{flalign}
						\hspace*{-12mm}\textbf{Save previous model: } \theta^p = \theta_t;\hspace*{4mm} \gamma^{p}_{t+1} = \gamma_{t};
					\end{flalign}}\\
					{\setlength{\abovedisplayskip}{-12pt}\setlength{\belowdisplayskip}{4pt}\begin{flalign}\label{eq:thupd}
							\textbf{Update model: }
							\theta_{t+1} = \theta_{t} + \beta_{t+1}\left((\xi^{(0)}_{t}, \xi^{(1)}_{t})^{\top} - \theta_{t}\right);
					\end{flalign}}\\
					{\setlength{\abovedisplayskip}{-12pt}\setlength{\belowdisplayskip}{4pt}\begin{flalign}\label{eq:gmstarupd}
						\hspace*{-20mm}\textbf{Reset parameters: }
							\hspace{4mm} T_{t} = 0; \hspace*{4mm} c = c_{t};	
					\end{flalign}}
				}{
					\hspace*{2cm}$\gamma^{p}_{t+1} = \gamma^{p}_{t}$; \hspace{4mm} $\theta_{t+1} = \theta_{t}$;
				}
				$t = t+1$;}
			\caption{CE2-ND}\label{algo:ce2det-nd}
		\end{algorithm}
\end{minipage}}
To better comprehend the algorithm, a pictorial depiction of the algorithm CE2-ND is provided in Fig. \ref{fig:ce2nddet-visual}.
\begin{figure}
	\centering
	\includegraphics[height=52mm, width=120mm]{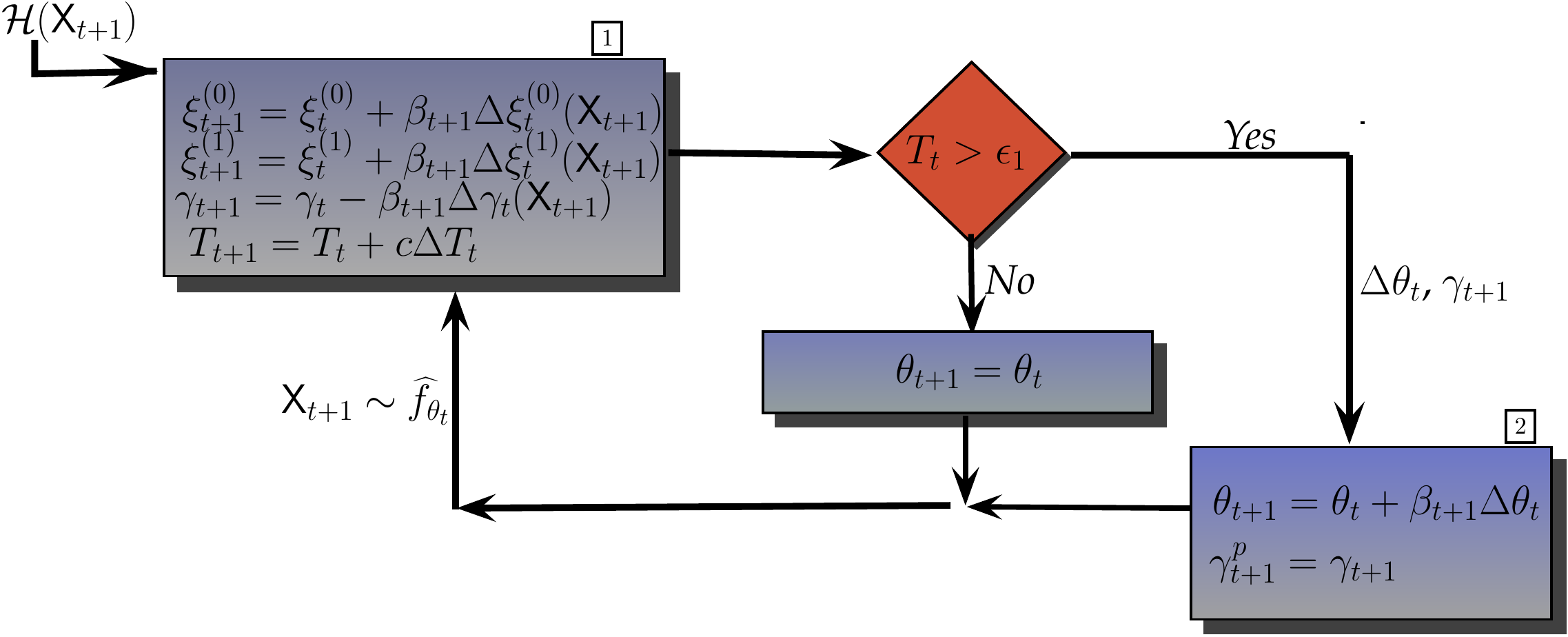}
	\caption{Flowchart representation of the algorithm CE2-ND.}\label{fig:ce2nddet-visual}
\end{figure}
Now it is important to note that the model parameter $\theta_t$ is not updated at each $t$. Rather it is updated every time $T_t$ hits $\epsilon_1$, where $0 < \epsilon_1 < 1$. So the update of $\theta_t$ only happens along a sub-sequence $\{t_{(n)}\}_{n \in \mathbb{N}}$ of $\{t\}_{t \in \mathbb{N}}$. This particular, yet important aspect of the algorithm is demonstrated in the time-line map given in Fig. \ref{fig:ce2nddet-timegraph}. So between $t=t_{(n)}$ and $t=t_{(n+1)}$, the variable $\gamma_{t}$ estimates the quantity $\gamma_{\rho}(\mathcal{H}, \widehat{\theta}_{t_{(n)}})$. Intuitively, one can think of the sequences $\{\theta_{t_{(n)}}\}$ and $\{\gamma_{t_{(n)}}\}$ to be tracking the ideal CE model sequence and the threshold levels respectively. We also maintain two book-keeping variables  $\gamma^{p}_{t}$ and $\theta^{p}$ which hold the previous threshold and the previous model parameter respectively. Thus $\gamma^{p}_{t_{(n)}}$ is the estimate of the $(1-\rho)$-quantile with respect to $\widehat{f}_{\theta_{t_{(n-1)}}}$ which is the previous model PDF. We also update the previous threshold $\gamma^{p}_{t}$ in recursion (\ref{eqn:ce2ndgammap}) using the previous PDF mixture $\widehat{f}_{\theta^{p}}$ to improve the accuracy of the previous threshold.
\begin{figure}
	\centering
	\includegraphics[height=52mm, width=120mm]{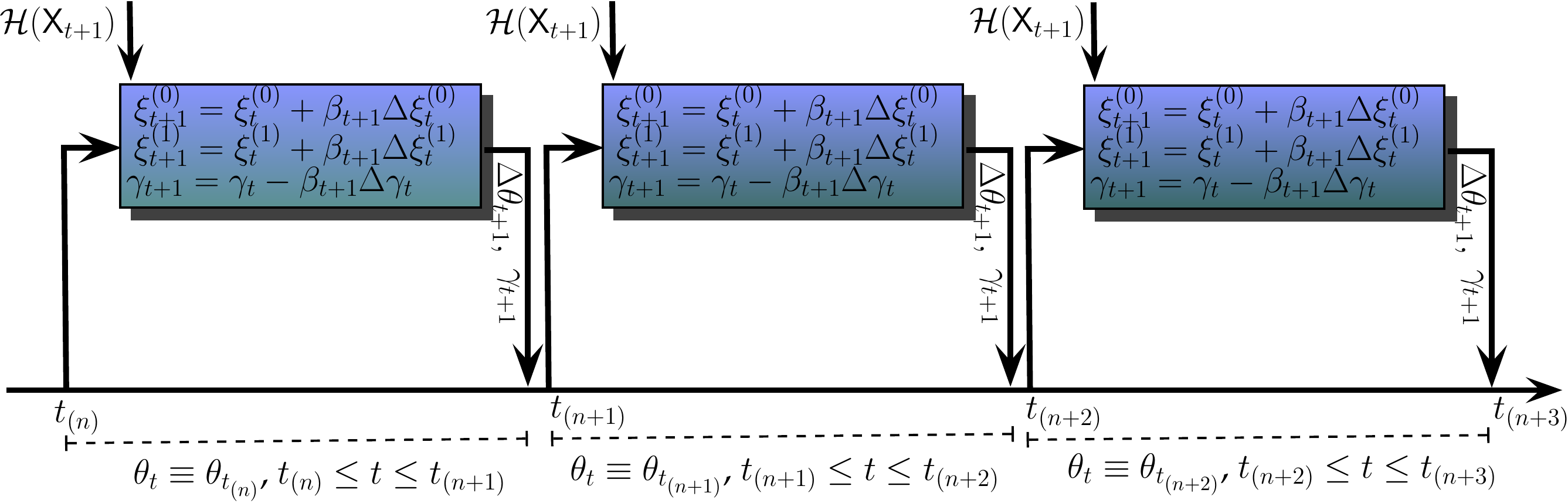}
	\caption{Timeline graph of the algorithm CE2-ND.}\label{fig:ce2nddet-timegraph}
\end{figure}

We maintain a variable $T_t$ and its recursion (\ref{eqn:ce2ndTt}) to determine the moment to update the model parameter. The recursion (\ref{eqn:ce2ndTt}) is an elegant trick to ensure that the current threshold estimate $\gamma_t$ eventually becomes greater than the previous threshold estimate $\gamma^{p}_{t_{(n)}}$, \emph{i.e.}, $\gamma_t > \gamma^{p}_{t_{(n)}}$ for all but finitely many $t$. We have shown in Lemma \ref{lmn:xiconv} that if $\gamma_{\rho}(\mathcal{H},\theta_t) > \gamma_{\rho}(\mathcal{H}, \theta^{p})$, then $T_{t} \rightarrow 1$ as $t \rightarrow \infty$. However, in practice, an algorithm cannot wait infinitely long to determine the order of the current and the previous thresholds. Hence we chose a confidence ceiling $\epsilon_1 \in (0,1)$  and the model parameters are updated when $T_t$ hits the ceiling. From the empirical studies we have conducted, we believe that $\epsilon_1$ in the range $[0.8,0.95]$ is sufficient to obtain a credible comparison of the thresholds. Also note that we reset $T_t$ in Equation (\ref{eq:gmstarupd}) during model parameter update to initiate a new comparison (since the model parameters are changed). Now to justify the comparison step $T_t > \epsilon_1$, one has to ensure that $\sup_{t}\vert T_t \vert < 1$ holds. 
It can be verified  that the random variable $T_t$ indeed belongs to $(-1,1)$, $\forall t > 0$. We state it as a proposition here.
\begin{figure}[!h]
	\centering
	\includegraphics[width=120mm, height=60mm]{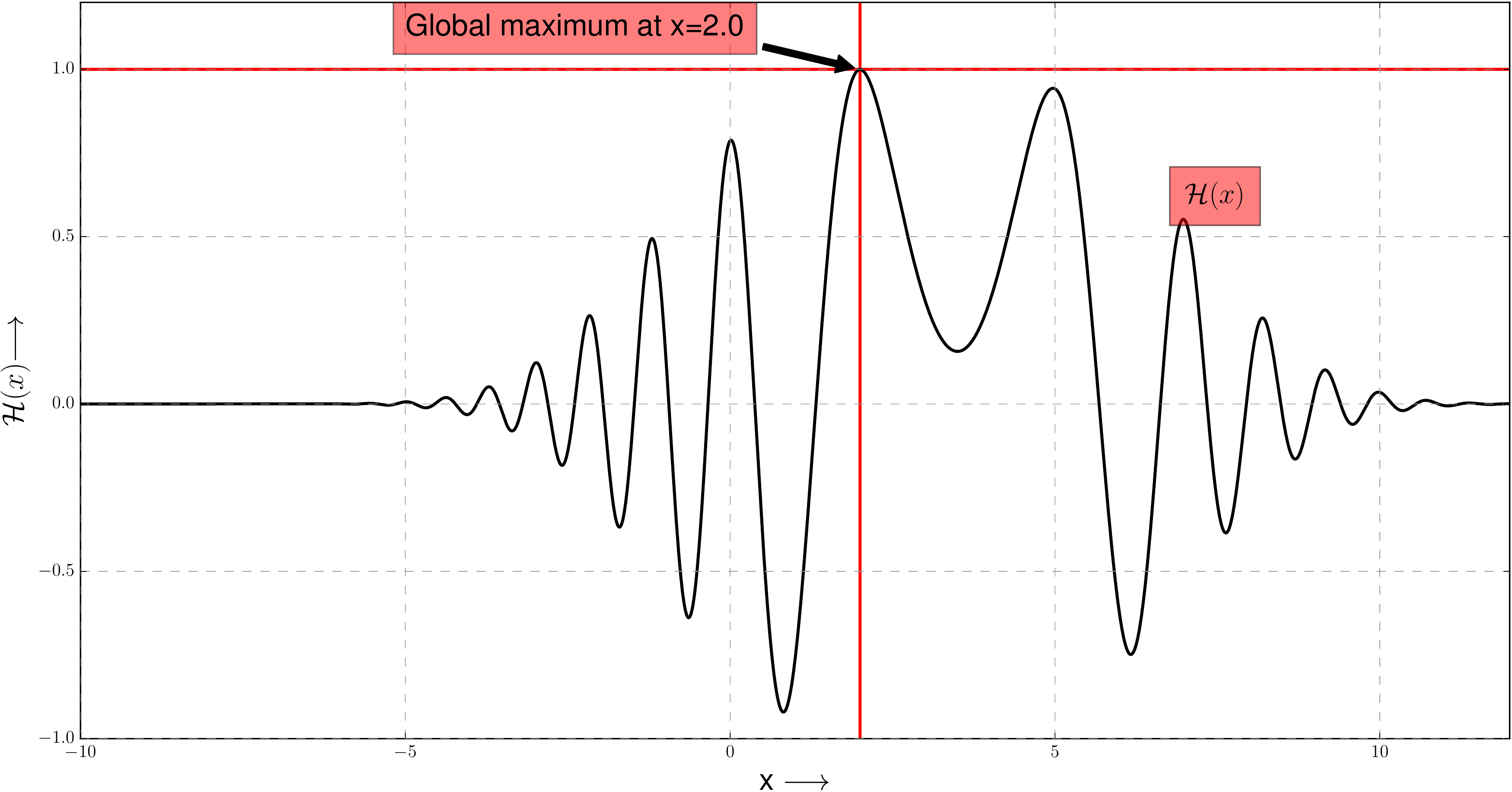}
	\caption{Plot of a real valued objective function defined over $\mathbb{R}$ whose global optimum $x^{*} = 2$.}\label{fig:cedemo1}
\end{figure}
\begin{figure}[!h]
	\centering
	\includegraphics[width=120mm, height=70mm]{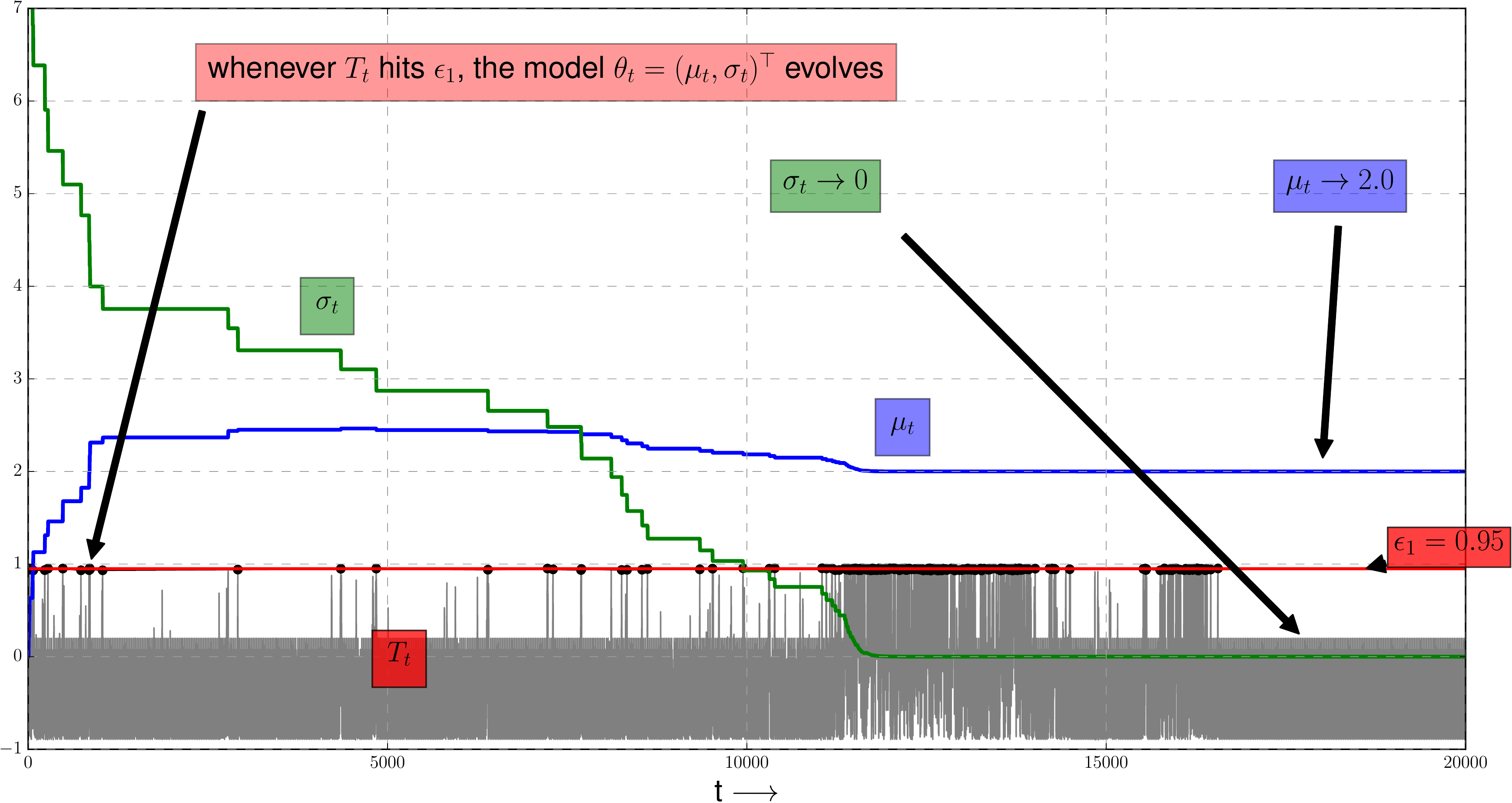}
	\caption{Plot of  $\mu_{t}$,  $\sigma_t$ and $T_t$ when CE2-ND is applied to the objective function from Fig. \ref{fig:cedemo1}. Note that the solution space is a subset of $\mathbb{R}$, and hence we have $\theta_t = (\mu_t, \sigma_t)^{\top} \in \mathbb{R}^{2}$. Now regarding the evolution of the various tracking variables, we have the mean $\mu_t$ converging to $x^{*} = 2$ and the variance $\sigma_t$ converging to $0$ which implies that the model sequence $\theta_t$ is indeed converging to the degenerate distribution concentrated at $x^{*}$. Now note that the $T_t$ variable controls the evolution of the model sequence $\{\theta_t\}$. Indeed the graph clearly illustrates that $\theta_t$ is updated only when $T_t$ hits $\epsilon_1 = 0.95$ ceiling.}\label{fig:cedemo2}
\end{figure}
\begin{proposition}
	For any $T_0 \in (0,1)$, $T_t$ in Equation (\ref{eqn:ce2ndTt}) belongs to $(-1,1)$, $\forall t > 0$.
\end{proposition}
\begin{proof}
	Assume $T_0 \in (0,1)$. Now the Equation (\ref{eqn:ce2ndTt}) can be rearranged as
	\begin{gather*}
	T_{t+1} = \left(1-c\right)T_{t} + c(\mathbb{I}_{\{\gamma_{t+1} > \gamma^{p}_{t+1}\}} - \mathbb{I}_{\{\gamma_{t+1} \leq \gamma^{p}_{t+1}\}}),
	\end{gather*}
	where $c \in (0,1)$. At first, we consider the two worst case scenarios. In the worst case, either $\mathbb{I}_{\{\gamma_{t+1} > \gamma^{p}_{t+1}\}} = 1$, $\forall t$ or $\mathbb{I}_{\{\gamma_{t+1} \leq \gamma^{p}_{t+1}\}} = 1$, $\forall t$.  Since the two events  $\{\gamma_{t+1} > \gamma^{p}_{t+1}\}$ and $\{\gamma_{t+1} \leq \gamma^{p}_{t+1}\}$ are mutually exclusive, we will only consider the former event $\I_{\{\gamma_{t+1} > \gamma^{p}_{t+1}\}} = 1$, $\forall t$. In this case,
	\begin{equation*}
	\begin{aligned}
	\lim_{t \rightarrow \infty}T_{t} &= \lim_{t \rightarrow \infty}\left(c + c(1-c) + c(1-c)^{2} + \dots + c(1-c)^{t-1} + (1-c)^{t}T_0\right) \\
	&= \lim_{t \rightarrow \infty}\frac{c(1-(1-c)^{t})}{c} + T_0(1-c)^{t} \\&= \lim_{t \rightarrow \infty}(1-(1-c)^{t}) + T_0(1-c)^{t}  = 1.
	\end{aligned}
	\end{equation*}
	Similarly for the latter event $\mathbb{I}_{\{\gamma_{t+1} \leq \gamma^{p}_{t+1}\}} = 1$, $\forall t$, one can prove that \\$\lim_{t \rightarrow \infty} T_{t} = -1$.\\
	
	In cases other than the worst case scenarios, both  the events $\{\mathbb{I}_{\{\gamma_{t+1} > \gamma^{p}_{t+1}\}} = 1, t \in \N\}$ and $\{\mathbb{I}_{\{\gamma_{t+1} \leq \gamma^{p}_{t+1}\}} = 1, t \in \N\}$ occur with non-zero probability. Hence $ \vert T_t \vert < 1$. This completes the proof.
\end{proof}

\begin{remark}
	The recursion (\ref{eqn:ce2ndgamma}) of $\gamma_t$ might be slow since the increment term $\Delta \gamma_t$ is small. So to accelerate it one might need to multiply the increment term with a constant $K_{\gamma} > 1.0$. In most practical cases, one can easily induce $K_{\gamma}$ from the knowledge about the bounds ($\inf_{x}\mathcal{H}(x)$ and $\sup_{x}\mathcal{H}(x)$) of the objective function $\mathcal{H}$.
\end{remark}

\section{Convergence Analysis}
To analyze the asymptotic behaviour of the algorithm CE2-ND, we utilize the ODE based analysis which is intuitively pleasing while requiring very less restrictions.
\begin{assm}\label{assm:ce2ndgmbd}
	The sequence $\{\gamma_{t}\}_{t \in \mathbb{N}}$ in Equation (\ref{eqn:ce2ndgamma}) satisfies $\sup_{t \in \N}{\vert \gamma_{t} \vert} < \infty$ with probability one.
\end{assm}
\begin{remark}\label{rem:gmbd1}
	Note that this is a technical requirement to prove the convergence. However, it is not straightforward, but is a prerequisite to establish convergence. Thus, one needs to show that this holds, \emph{i.e.}, the iterates remain uniformly bounded. In most pragmatic scenarios, one imposes this requirement by projecting the updates to a compact and convex set. However, in the case of the quantile estimation recursion (\ref{eqn:ce2ndgamma}),  note that the objective function $\mathcal{H}$ is bounded (\emph{i.e.}, $\mathcal{H}(x) \in [\mathcal{H}_{l}, \mathcal{H}_{u}], \forall x)$ and hence the true quantile $\gamma_{\rho}(\mathcal{H}, \cdot)$ belongs to the closed and bounded interval $[\mathcal{H}_{l}, \mathcal{H}_{u}]$. Therefore, one can easily guarantee the above assumption by  projecting the iterates $\gamma_t$ back to the above interval if they drift too far away from the interval.  The ODE analysis in such a case follows roughly along the same lines as below. For more see \cite{kushner2012stochastic}.\\\\\
\end{remark}

Define the filtration $\{\mathcal{F}_{t}\}_{t \in \mathbb{N}}$, where $\mathcal{F}_{t} \triangleq \sigma(\gamma_0, \gamma_i, \mathsf{X}_{i}, 1 \leq i \leq t)$ is the $\sigma$-field  generated by $\gamma_0$, $\gamma_i$ and $\mathsf{X}_{i}$, $1 \leq i \leq t$.\\

As mentioned earlier, the model parameter $\theta_{t}$ is updated only along a subsequence $\{t_{(n)}\}_{n \in \mathbb{N}}$ of $\{t\}_{t \in \mathbb{N}}$. Between $t = t_{(n)}$ and $t = t_{(n+1)}$, the model parameter $\theta_t$ remains constant. So we can analyze the limiting behaviour of $\gamma_{t}$, $\xi^{(0)}_{t}$ and $\xi^{(1)}_{t}$ by keeping $\theta_{t}$ fixed. We now have the following result for recursion (\ref{eqn:ce2ndgamma}):\\
\begin{lemma} \label{lmn:ce2nd-gmconv}
	Assume $\theta_{t} \equiv \theta, \forall t$. Let Assumption \ref{assm:ce2ndgmbd} hold and also let the learning rate $\{\beta_t\}_{t \in \N}$ satisfy Assumption \ref{assm:lnrtce2nd}. Then the sequence $\{\gamma_t\}_{t \in \N}$ defined in Equation (\ref{eqn:ce2ndgamma}) satisfies
	$\gamma_{t} \rightarrow \gamma_{\rho}(\mathcal{H}, \widehat{\theta})$ as $t \rightarrow \infty$ with probability one, where $\widehat{f}_{\theta} = (1-\lambda)f_{\theta} + \lambda f_{\theta_0}$.
\end{lemma}
\begin{proof}
	First we recall recursion (\ref{eqn:ce2ndgamma}) here:
		{\setlength{\abovedisplayskip}{0pt}\begin{equation}\label{eqn:preq2}
		\begin{aligned}
		\gamma_{t+1} = \gamma_{t} - \beta_{t+1} \Big(-(1-\rho)\mathbb{I}_{\{\mathcal{H}(\mathsf{X}_{t+1}) \geq \gamma_{t}\}} +  \rho \mathbb{I}_{\{\mathcal{H}(\mathsf{X}_{t+1}) \leq \gamma_{t} \}}\Big),
		\end{aligned}
		\end{equation}}
	where $\mathsf{X}_{t+1} \sim \widehat{f}_{\theta}$.\\\\	
	The above equation can be rewritten as
	\begin{equation}\label{eqn:preq1}
	\gamma_{t+1} = \gamma_{t} - \beta_{t+1}\Delta \gamma_{t}(\mathsf{X}_{t+1}),
	\end{equation}
	where $\Delta\gamma_{t}$ is defined in Equation (\ref{eq:gmgradinc}).\\\\	
	The above equation can be further viewed as,
	\begin{flalign*}
	\gamma_{t+1} &= \gamma_{t} - \beta_{t+1} \Delta \gamma_{t}(\mathsf{X}_{t+1})\\
	&= \gamma_{t} + \beta_{t+1}\Big(-\Delta \gamma_{t}(\mathsf{X}_{t+1}) + \mathbb{E}\left[\Delta\gamma_{t} (\mathsf{X}_{t+1}) \vert \mathcal{F}_{t}\right] - \mathbb{E}\left[\Delta\gamma_{t} (\mathsf{X}_{t+1}) \vert \mathcal{F}_{t}\right]\Big)\\	
	&= \gamma_{t} + \beta_{t+1}\left(\mathbb{M}^{(1,0)}_{t+1} - \mathbb{E}\left[\Delta\gamma_{t} (\mathsf{X}_{t+1}) \vert \mathcal{F}_{t}\right]\right),
	\end{flalign*}
	where 
	\begin{flalign}\label{eq:noisetermmt}
	\mathbb{M}^{(1,0)}_{t+1} &\triangleq  \mathbb{E}\left[\Delta{\gamma}_{t} (\mathsf{X}_{t+1}) \vert \mathcal{F}_{t}\right]-\Delta{\gamma}_{t}(\mathsf{X}_{t+1})\nonumber\\
	&= \mathbb{E}_{\widehat{\theta}}\left[\Delta{\gamma}_{t}( \mathsf{X}_{t+1})\right] - \Delta{\gamma}_{t}(\mathsf{X}_{t+1}).
	\end{flalign}
	The above equality follows since $\{\mathsf{X}_{t+1}, t \in \N\}$ is IID.\vspace*{2mm}\\
	Also,
	{\setlength{\abovedisplayskip}{2pt}\begin{flalign}
		-\mathbb{E}\left[\Delta\gamma_{t}(\mathsf{X}_{t+1}) \vert \mathcal{F}_{t}\right] &\in -\mathbb{E}\left[\partial_{\gamma} \psi(\mathcal{H}(\mathsf{X}_{t+1}), \gamma_t) \vert \mathcal{F}_{t}\right] \nonumber\\&=  -\mathbb{E}_{\widehat{\theta}}\left[\partial_{\gamma} \psi(\mathcal{H}(\mathsf{X}_{t+1}), \gamma_t) \right],
		\end{flalign}}
	\noindent
	where $\psi(\cdot,\cdot)$ is as in Lemma \ref{lma:qnopt} and $\partial_{\gamma}\psi$ (the sub-differential of $\psi(\cdot,\gamma)$ \emph{w.r.t.} $\gamma$) is a set function and is defined as follows:
	\begin{equation}\label{eq:psisubdiff}
	\hspace*{-3mm}\partial_{\gamma}\psi(\mathcal{H}(x), \gamma) = \left\{
	\begin{array}{ll}
	\{-(1-\rho)\I_{\{\mathcal{H}(x) \geq \gamma\}} + \rho \I_{\{\mathcal{H}(x) \leq \gamma\}}, 
	\hspace{1mm} \textrm{for } \hspace{0mm} \gamma \neq \mathcal{H}(x),\vspace*{5mm}\\
	\left[  -(1-\rho), \rho\ \right], \hspace{1mm} \textrm{for } \hspace{0mm} \gamma = \mathcal{H}(x),
	\end{array}
	\right.
	\end{equation}
	
			For brevity, let $h^{(1,0)}(\gamma) \triangleq -\mathbb{E}_{\widehat{\theta}}\left[\partial_{\gamma} \psi(\mathcal{H}(\mathsf{X}), \gamma) \right]$, where $\mathsf{X} \sim \widehat{f}_{\theta}$ (Note that we consider the random variable $\mathsf{X}$ instead of $\mathsf{X}_{t+1}$ for notational convenience. This indeed makes sense, since $\{\mathsf{X}_t, t \in \N\}$ is IID and $\mathsf{X}_{t+1} \sim \widehat{f}_{\theta}$). The set function $h^{(1,0)}:[\mathcal{H}_{l}, \mathcal{H}_{u}] \rightarrow \{$subsets of $\mathbb{R}\}$ satisfies the following properties:
	\begin{enumerate}
		\item For each $\gamma \in [\mathcal{H}_{l}, \mathcal{H}_{u}]$, $h^{(1,0)}(\gamma)$ is convex and compact.\vspace*{2mm}\\		
		\hspace*{4mm}Indeed, it follows directly from Equation (\ref{eq:psisubdiff}).  For each $\gamma \in [\mathcal{H}_{l}, \mathcal{H}_{u}]$, note that $-h^{(1,0)}(\gamma)$ is either a singleton or the closed interval $[-(1-\rho),\rho]$.\\
		\item For each $\gamma \in [\mathcal{H}_{l}, \mathcal{H}_{u}]$, we have \\
		$\sup_{y \in h^{(1,0)}(\gamma)} \vert y \vert < K_{1,0}(1+\vert \gamma \vert)$, for some  $0 < K_{1,0}  < \infty$. \vspace*{2mm}\\
		\hspace*{4mm}Indeed, for each $\gamma \in [\mathcal{H}_{l}, \mathcal{H}_{u}]$, note that $-h^{(1,0)}(\gamma)$ is either the scalar \\ $\mathbb{E}_{\widehat{\theta}}\left[-(1-\rho)\I_{\{\mathcal{H}(\mathsf{X}) \geq \gamma\}} +  \rho \I_{\{\mathcal{H}(\mathsf{X}) \leq \gamma\}}\right]$ or the  bounded closed interval $[-(1-\rho),\rho]$. Hence the above bound exists.\\
		\item $h^{(1,0)}$ is upper semi-continuous.\vspace*{2mm}\\	
		\hspace*{4mm}To prove this, one has to show the following: if the sequence $\{\gamma_n\}$ converges to $\bar{\gamma}$ and $\{y_n\}$ converges to $\bar{y}$ with $y_n \in h^{(1,0)}(\gamma_n)$, then $\bar{y} \in h^{(1,0)}(\bar{\gamma})$. 
		Note that for each $\gamma \in [\mathcal{H}_{l}, \mathcal{H}_{u}]$, there are two possibilities for $-h^{(1,0)}(\gamma)$. It is either $\mathbb{E}_{\widehat{\theta}}\left[-(1-\rho)\I_{\{\mathcal{H}(\mathsf{X}) \geq \gamma\}} +  \rho \I_{\{\mathcal{H}(\mathsf{X}) \leq \gamma\}}\right]$ or the closed interval $[-(1-\rho), \rho]$. Also,
		\begin{flalign}\label{eq:h10limcd}
		\mathbb{E}_{\widehat{\theta}}&\left[-(1-\rho)\I_{\{\mathcal{H}(\mathsf{X}) \geq \gamma\}} +  \rho \I_{\{\mathcal{H}(\mathsf{X}) \leq \gamma\}}\right]\nonumber\\ &= -(1-\rho)\P_{\widehat{\theta}}(\mathcal{H}(\mathsf{X}) \geq \gamma) + \rho \P_{\widehat{\theta}}(\mathcal{H}(\mathsf{X}) \leq \gamma)\nonumber\\ 
		&\in [-(1-\rho), \rho].
		\end{flalign}
		Now consider the case when $-y_n = -h^{(1,0)}(\gamma_n) =$\\ $\mathbb{E}_{\widehat{\theta}}\left[-(1-\rho)\I_{\{\mathcal{H}(\mathsf{X}) \geq \gamma_n\}} +  \rho \I_{\{\mathcal{H}(\mathsf{X}) \leq \gamma_n\}}\right]$, then $-y_n = -(1-\rho)\P_{\widehat{\theta}}(\mathcal{H}(\mathsf{X}) \geq \gamma_n) + \rho \P_{\widehat{\theta}}(\mathcal{H}(\mathsf{X}) \leq \gamma_n)$ converges to $-\bar{y} = -(1-\rho)\P_{\widehat{\theta}}(\mathcal{H}(\mathsf{X}) \geq \bar{\gamma}) + \rho \P_{\widehat{\theta}}(\mathcal{H}(\mathsf{X}) \leq \bar{\gamma})$. This follows from the continuity of probability measures. Now from Equation (\ref{eq:h10limcd}), we have $-\bar{y} \in -h^{(1,0)}(\bar{\gamma})$, \emph{i.e.}, $\bar{y} \in h^{(1,0)}(\bar{\gamma})$.\\
		
		Now consider the case when $-y_n \in $ $[-(1-\rho), \rho]$ and $-\bar{y} = $ \\ $\mathbb{E}_{\widehat{\theta}}\left[-(1-\rho)\I_{\{\mathcal{H}(\mathsf{X}) \geq \bar{\gamma}\}} +  \rho \I_{\{\mathcal{H}(\mathsf{X}) \leq \bar{\gamma}\}}\right]$.
		This implies that $\psi(\cdot, \gamma)$ is differentiable at $\gamma = \bar{\gamma}$, while only sub-differentials exist at $\gamma = \gamma_n, \forall n \in \N$. This particular scenario is not possible. The reason being $\psi$ is piece-wise linear in $\gamma$ and $\psi(\cdot, \gamma)$ is differentiable at $\gamma = \bar{\gamma}$.  Therefore, there exists a neighbourhood around $\bar{\gamma}$  such that $\psi(\cdot, \gamma)$ is linear. However, by hypothesis $\{\gamma_n\} \rightarrow \bar{\gamma}$ which is impossible due to the linear behaviour of $\psi$ around $\bar{\gamma}$ and the non-differentiability of $\psi$ at each $\gamma_n$.\\\\
	\end{enumerate}	
Now, regarding the noise term $\mathbb{M}^{(1,0)}_{t}$ (defined in Equation (\ref{eq:noisetermmt})), observe that $\mathbb{M}^{(1,0)}_{t}$ is $\mathcal{F}_{t}$-measurable $\forall t \in \N \setminus 
\{0\}$ and is integrable. Also, it is not hard to verify that $\{\mathbb{M}^{(1,0)}_{t}$, $t \in \N \setminus \{0\} \}$ is a martingale difference noise sequence. Indeed, almost surely,
{\setlength{\belowdisplayskip}{4pt}\begin{flalign*}
\mathbb{E}[\mathbb{M}^{(1,0)}_{t+1} | \mathcal{F}_{t}] &= \mathbb{E}\Big[\mathbb{E}_{\widehat{\theta}}\left[\Delta{\gamma}_{t} (\mathsf{X}_{t+1})\right] - \Delta{\gamma}_{t}(\mathsf{X}_{t+1}) \Big| \mathcal{F}_{t}\Big]\\
&= \mathbb{E}\Big[\mathbb{E}_{\widehat{\theta}}[\Delta{\gamma}_{t} (\mathsf{X}_{t+1})] \Big| \mathcal{F}_{t}\Big] - \mathbb{E}\left[\Delta{\gamma}_{t}(\mathsf{X}_{t+1})\Big| \mathcal{F}_{t}\right] \\
&= \mathbb{E}_{\widehat{\theta}}[\Delta{\gamma}_{t}(\mathsf{X}_{t+1})] -\mathbb{E}_{\widehat{\theta}}[\Delta{\gamma}_{t}(\mathsf{X}_{t+1})] \\
&= 0.
\end{flalign*}}
The third equality above holds since $\{\mathsf{X}_{t+1}$, $t \in \N\}$ is IID.\vspace*{2mm}\\
Also, since $\Delta{\gamma}_{t}(\mathsf{X}_{t+1})$ is bounded almost surely, we find that $\Delta{\gamma}_{t}(\mathsf{X}_{t+1})$ has finite first and second order moments. Hence, 
{\setlength{\belowdisplayskip}{4pt}\begin{flalign}\label{eqn:bdeq}
\mathbb{E}\left[\vert \mathbb{M}^{(1,0)}_{t+1} \vert^{2}|\mathcal{F}_{t}\right] 
&= \mathbb{E}\Big[\Big(\mathbb{E}_{\widehat{\theta}}\left[\Delta{\gamma}_{t}(\mathsf{X}_{t+1})\right] - \Delta{\gamma}_{t}(\mathsf{X}_{t+1}) \Big)^{2} \Big| \mathcal{F}_{t}\Big] \nonumber\\
&\leq K_{1,1}(1+\vert \gamma_{t} \vert^{2}),
\end{flalign}}
for some  $0 < K_{1,1} < \infty$.

Also, the stability of the sequence $\{\gamma_t\}$ is guaranteed by  Assumption \ref{assm:ce2ndgmbd}, where we assume the almost sure boundedness of the sequence $\{\gamma_{t}\}$. Now, by appealing to Theorem $2$  in Chapter $5$ of \cite{borkar2008stochastic}, we deduce that the stochastic sequence $\{\gamma_{t}\}$ asymptotically tracks the following differential inclusion 
\begin{flalign} \label{eqn:dtode}
\frac{d}{dt}\gamma(t) \in h^{(1,0)}(\gamma(t)) &= -\mathbb{E}_{\widehat{\theta}}\left[\partial_{\gamma} \psi(\mathcal{H}(\mathsf{X}), \gamma(t))\right] \nonumber\\ &= -\partial_{\gamma}\mathbb{E}_{\widehat{\theta}}\left[\psi(\mathcal{H}(\mathsf{X}), \gamma(t))\right].
\end{flalign}
Note that the interchange of $\mathbb{E}_{\widehat{\theta}}[\cdot]$ and $\partial_{\gamma}$ in the above differential inclusion follows by appealing to the Dominated Convergence Theorem.\\

Now we analyze the  stability of the above differential inclusion. Note that by Lemma \ref{lma:qnopt}, we know that $\gamma^{*} \triangleq \gamma_{\rho}(\mathcal{H}, \widehat{\theta})$ is the unique root of the function $h^{(1,0)}(\cdot)$ and hence it is a fixed point of the flow induced by the above differential inclusion. Now, define $V(\gamma) \triangleq \mathbb{E}_{\widehat{\theta}}\left[\psi(\mathcal{H}(\mathsf{X}), \gamma)\right] - \mathbb{E}_{\widehat{\theta}}\left[\psi(\mathcal{H}(\mathsf{X}), \gamma^{*})\right]$. It is easy to verify that $V$ is continuously differentiable and  $\mathbb{E}_{\widehat{\theta}}\left[\psi(\mathcal{H}(\mathsf{X}), \gamma)\right]$ is a convex function and hence $\gamma^{*}$ is its global minimum. Hence $V(\gamma) > 0$, $\forall \gamma \in \mathbb{R} \backslash \{\gamma^{*}\}$. Further $V(\gamma^{*}) = 0$ and $V(\gamma) \rightarrow \infty$ as $\vert \gamma \vert \rightarrow \infty$. So $V(\cdot)$ is a Lyapunov function. Also note that $\nabla V(\gamma)^{\top}h^{(1,0)}(\gamma) \leq 0$. So $\gamma^{*}$ is the global attractor of the flow induced by the differential inclusion defined in Equation (\ref{eqn:dtode}). Thus by appealing to Corollary 4 in Chapter 5 of \cite{borkar2008stochastic}, we obtain that the iterates $\gamma_{t}$ converge almost surely to $\gamma^{*} = \gamma_{\rho}(\mathcal{H}, \widehat{\theta})$. This completes the proof of Lemma \ref{lmn:ce2nd-gmconv}.
			
\end{proof}
Lemma \ref{lmn:ce2nd-gmconv} claims that if the model parameter is held constant, i.e., $\theta_{t} \equiv \theta, \forall t$, then $\gamma_{t}$ successfully tracks $\gamma_{\rho}(\mathcal{H}, \widehat{\theta})$: the $(1-\rho)$-quantile of $\mathcal{H}$ w.r.t. the mixture PDF $\widehat{f}_{\theta}$.

Now we define the filtration $\{\mathcal{F}_{t}\}_{t \in \mathbb{N}}$ where the $\sigma$-field \\$\mathcal{F}_t$ = $\sigma\left(\gamma_i, \gamma^{p}_i, \xi^{(0)}_i, \xi^{(1)}_i, \theta_i, 0 \leq i \leq t; \mathsf{X}_{i}, 1 \leq i \leq t\right)$, $t \in \mathbb{N}$. 

\begin{lemma}\label{lmn:xiconv}
	Assume $\theta_{t} \equiv \theta, \forall t$.  Let Assumptions  \ref{assm:lnrtce2nd} and \ref{assm:ce2ndgmbd} hold. Then almost surely,
	\begin{flalign}
	&1.\hspace*{4mm}\lim_{t \rightarrow \infty} \xi^{(0)}_{t} = \xi^{(0)}_{*} = \frac{\mathbb{E}_{\widehat{\theta}}\left[ \mathsf{g}_{1}\left(\mathcal{H}(\mathsf{X}), \mathsf{X}, \gamma_{\rho}(\mathcal{H}, \widehat{\theta})\right)\right]}{\mathbb{E}_{\widehat{\theta}}\left[\mathsf{g}_{0}\left(\mathcal{H}(\mathsf{X}), \gamma_{\rho}(\mathcal{H}, \widehat{\theta})\right)\right]},\hspace*{3cm}\label{eqn:xi0limit}\\
	&2.\hspace*{4mm}\lim_{t \rightarrow \infty} \xi^{(1)}_{t} = \xi^{(1)}_{*} =  \frac{\mathbb{E}_{\widehat{\theta}}\left[\mathsf{g}_{2}\left(\mathcal{H}(\mathsf{X}), \mathsf{X}, \gamma_{\rho}(\mathcal{H}, \widehat{\theta}), \xi^{(0)}_{*}\right)\right]}{\mathbb{E}_{\widehat{\theta}}\left[\mathsf{g}_{0}\left(\mathcal{H}(\mathsf{X}), \gamma_{\rho}(\mathcal{H}, \widehat{\theta})\right)\right]}.\label{eqn:xi1limit}
	\end{flalign}
	
	$3$. If $\gamma_{\rho}(\mathcal{H}, \widehat{\theta}) > \gamma_{\rho}(\mathcal{H}, \widehat{\theta^{p}})$, then $\{T_{t}\}_{t \in \mathbb{N}}$ in Equation (\ref{eqn:ce2ndTt}) satisfies \\\hspace*{10mm} $\lim_{t \rightarrow \infty} T_{t} = 1$ a.s.
\end{lemma}
\begin{proof}
	$1.$ First, we recall Equation (\ref{eqn:ce2ndxi0}) below:
	\begin{equation}\label{eq:c0e1}
	\begin{aligned}
	\xi^{(0)}_{t+1} = \xi^{(0)}_{t} +  \beta_{t+1} \Big(\mathsf{g}_{1}(\mathcal{H}(\mathsf{X}_{t+1}), \mathsf{X}_{t+1}, \gamma_{t})& - \xi^{(0)}_{t}\mathsf{g}_{0}\left(\mathcal{H}(\mathsf{X}_{t+1}), \gamma_{t}\right)\Big),\\
	&\textrm{ where } \mathsf{X}_{t+1} \sim \widehat{f}_{\theta_{t}}.
	\end{aligned}
	\end{equation}
	Note that the above recursion of $\xi^{(0)}_{t}$ depends on $\gamma_{t}$, but not the other way. This implies that we can replace $\gamma_{t}$ by its limit point $\gamma_{\rho}(\mathcal{H}, \widehat{\theta})$ and a bias term which goes to zero as $t \rightarrow \infty$. We denote the decaying bias term using the notation $o(1)$.
	Further, using the hypothesis that $\theta_{t} = \theta$ and from Equation (\ref{eq:c0e1}), we get,
	\begin{equation} \label{eq:c0}
	\xi^{(0)}_{t+1} = \xi^{(0)}_{t} + \beta_{t+1} \left(h^{(2,0)}(\xi^{(0)}_{t}) + \mathbb{M}^{(2,0)}_{t+1} + o(1)\right),\hspace*{3cm}
	\end{equation}
	\begin{equation} \label{eq:hm0}
	\begin{aligned}
	\mathrm{where} \hspace{1mm}
	h^{(2,0)}(x) \triangleq -\mathbb{E}\left[x \mathsf{g}_{0}\left(\mathcal{H}(\mathsf{X}_{t+1}), \gamma_{\rho}(\mathcal{H}, \widehat{\theta})\right) \Big| \mathcal{F}_{t}\right] + \\ \mathbb{E}\left[ \mathsf{g}_{1}\left(\mathcal{H}(\mathsf{X}_{t+1}), \mathsf{X}_{t+1}, \gamma_{\rho}(\mathcal{H}, \widehat{\theta})\right) \Big|  \mathcal{F}_{t}\right],
	\end{aligned}
	\end{equation}
	\begin{equation*}
	\begin{aligned}
	\mathbb{M}^{(2,0)}_{t+1} &\triangleq \mathsf{g}_{1}\left(\mathcal{H}(\mathsf{X}_{t+1}), \mathsf{X}_{t+1}, \gamma_{\rho}(\mathcal{H}, \widehat{\theta})\right) - \\ \mathbb{E}&\left[\mathsf{g}_{1}\left(\mathcal{H}(\mathsf{X}_{t+1}), \mathsf{X}_{t+1}, \gamma_{\rho}(\mathcal{H}, \widehat{\theta})\right) \Big|  \mathcal{F}_{t}\right] - \xi^{(0)}_{t}\mathsf{g}_{0}\left(\mathcal{H}(\mathsf{X}_{t+1}), \gamma_{\rho}(\mathcal{H}, \widehat{\theta})\right) + \\  \mathbb{E}&\left[\xi^{(0)}_{t} \mathsf{g}_{0}\left(\mathcal{H}(\mathsf{X}_{t+1}), \gamma_{\rho}(\mathcal{H}, \widehat{\theta})\right) \Big| \mathcal{F}_{t}\right] \textrm{ and } \mathsf{X}_{t+1} \sim \widehat{f}_{\theta}.
	\end{aligned}
	\end{equation*}
	Since $\mathsf{X}_{t+1}$ is independent of the $\sigma$-field $\mathcal{F}_{t}$,
	the function $h^{(2,0)}(\cdot)$ in Equation (\ref{eq:hm0}) can be rewritten as
	\begin{equation*}
	\begin{aligned}
	h^{(2,0)}(x) = -\mathbb{E}_{\widehat{\theta}}\left[x \mathsf{g}_{0}\left(\mathcal{H}(\mathsf{X}), \gamma_{\rho}(\mathcal{H}, \widehat{\theta})\right)\right] + \mathbb{E}_{\widehat{\theta}}\left[\mathsf{g}_{1}\left(\mathcal{H}(\mathsf{X}), \mathsf{X}, \gamma_{\rho}(\mathcal{H}, \widehat{\theta})\right)\right],
	\end{aligned}
	\end{equation*}
	where $\mathsf{X} \sim \widehat{f}_{\theta}$. It is easy to verify that $\mathbb{M}^{(2,0)}_{t}$, $t \in \mathbb{N}$ is a martingale difference sequence, \emph{i.e.}, $\mathbb{M}^{(2,0)}_{t}$ is $\mathcal{F}_{t}$-measurable, integrable and $\mathbb{E}[\mathbb{M}^{(2,0)}_{t+1} | \mathcal{F}_t] = 0$ \emph{a.s.}, $\forall t \in \mathbb{N}$. It is also easy to verify that $h^{(2,0)}(\cdot)$ is Lipschitz continuous. Also since $S(\cdot)$ is bounded above and $\widehat{f}_{\theta}(\cdot)$ has finite first and second moments we have almost surely,
	\[\mathbb{E}\left[\Vert \mathbb{M}^{(2,0)}_{t+1} \Vert^{2} \vert \mathcal{F}_{t}\right] \leq K_{2,0}(1+\Vert \xi^{(0)}_t \Vert^{2}), \forall t \geq 0, \textrm{ for some } 0 < K_{2,0} < \infty. \]
	Now consider the ODE
	\begin{equation}\label{eq:xi0ode2}
	\frac{d}{dt}\xi^{(0)}(t) = h^{(2,0)}(\xi^{(0)}(t)).
	\end{equation} 
	We may rewrite the above ODE as,
	\[\frac{d}{dt}\xi^{(0)}(t) = A\xi^{(0)}(t) + b^{(0)},\]
	where $A$ is a  diagonal matrix with $A_{ii} = -\mathbb{E}_{\widehat{\theta}}\left[\mathsf{g}_{0}\left(\mathcal{H}(\mathsf{X}), \gamma_{\rho}(\mathcal{H}, \widehat{\theta})\right)\right]$, $0 \leq i < m$ and  $b^{(0)} = \mathbb{E}_{\widehat{\theta}}\left[\mathsf{g}_{1}\left(\mathcal{H}(\mathsf{X}), \mathsf{X}, \gamma_{\rho}(\mathcal{H}, \widehat{\theta})\right)\right]$. In \cite{borkar2000ode} and Chapter $3$ of \cite{borkar2008stochastic}, an ODE based analysis has been developed to assure the stability (boundedness almost surely) of stochastic approximation recursions under general conditions. We apply the result from there for our case. Indeed, consider the following ODE: 
	\begin{flalign}\label{eqn:xi0infty}
	\frac{d}{dt}\xi^{(0)}(t) = \lim_{\eta \rightarrow \infty}\frac{h^{(2,0)}(\eta\xi^{(0)}(t))}{\eta} = A \xi^{(0)}(t).
	\end{flalign}
	Since the matrix $A$ has the same value for all the diagonal elements, $A$ has only one eigenvalue: $\lambda(A)$ = $-\mathbb{E}_{\widehat{\theta}}\left[\mathsf{g}_{0}\left(\mathcal{H}(\mathsf{X}), \gamma_{\rho}(\mathcal{H}, \widehat{\theta})\right)\right]$ with multiplicity $m$. Also observe that $\lambda(A) < 0$. Hence the ODE (\ref{eqn:xi0infty}) is  globally asymptotically stable to the origin. Using Theorem 7, Chapter 3 of \cite{borkar2008stochastic}, the iterates $\{\xi^{(0)}_{t}\}_{t \in \mathbb{N}}$ are stable \emph{a.s.}, \emph{i.e.}, $\sup_{t \in \mathbb{N}}{\Vert \xi^{(0)}_{t} \Vert} < \infty$ \emph{a.s.}   
	
	Again, by using the earlier argument that the eigenvalues $\lambda(A)$ of $A$ are negative and identical, the point $-A^{-1}b^{(0)}$ can be seen to be a globally asymptotically stable equilibrium of the ODE (\ref{eq:xi0ode2}). By using Corollary 4, Chapter 2 of \cite{borkar2008stochastic}, we can conclude that
	\begin{equation*}
	\lim_{t \rightarrow \infty} \xi^{(0)}_{t} = -A^{-1}b^{(0)} \emph{a.s.} = \frac{E_{\widehat{\theta}}\left[\mathsf{g}_{1}\left(\mathcal{H}(\mathsf{X}), \mathsf{X}, \gamma_{\rho}(\mathcal{H}, \widehat{\theta})\right)\right]}{E_{\widehat{\theta}}\left[\mathsf{g}_{0}\left(\mathcal{H}(\mathsf{X}), \gamma_{\rho}(\mathcal{H}, \widehat{\theta})\right)\right]}\hspace*{4mm} \emph{a.s.}
	\end{equation*}
	This completes the proof of Equation (\ref{eqn:xi0limit}).\\\\
	$2.$ First, we recall the matrix recursion (\ref{eqn:ce2ndxi1}) below:
	\begin{equation}\label{eq:c0e2}
	\begin{aligned}
	\xi^{(1)}_{t+1} = \xi^{(1)}_{t} + \beta_{t+1}\Big(\mathsf{g}_{2}(\mathcal{H}(\mathsf{X}_{t+1}), \mathsf{X}_{t+1}, \gamma_{t}, \xi^{(0)}_t)& - \xi^{(1)}_{t}\mathsf{g}_{0}\left(\mathcal{H}(\mathsf{X}_{t+1}), \gamma_{t}\right)\Big)\\
	&\textrm{ where } \mathsf{X}_{t+1} \sim \widehat{f}_{\theta_{t}}.
	\end{aligned}
	\end{equation}
	As in the earlier proof, we also assume $\theta_{t} = \theta$. Also note that $\xi^{(1)}_{t}$, $\xi^{(0)}_{t}$ and $\gamma_{t}$ are on the same timescale. However, the recursion of $\gamma_{t}$ proceeds independently and in particular does not depend on $\xi^{(0)}_{t}$ and $\xi^{(1)}_{t}$. Also, there is a unilateral coupling of $\xi^{(1)}_{t}$ on $\xi^{(0)}_{t}$ and $\gamma_{t}$, but not the other way. Hence, while analyzing recursion (\ref{eq:c0e2}), one may replace $\gamma_t$ and $\xi^{(0)}_{t}$ in Equation (\ref{eq:c0e2}) with their limit points $\gamma_{\rho}(\mathcal{H}, \theta)$  and $\xi^{(0)}_{*}$ respectively and a decaying bias term which is $o(1)$. Now, by considering all the above observations, we rewrite the Equation (\ref{eq:c0e2}) as, 
	\begin{equation} \label{eqn:c1}
	\begin{aligned}
	\xi^{(1)}_{t+1} = \xi^{(1)}_{t} + \beta_{t+1} \left(h^{(2,1)}(\xi^{(1)}_{t}) + \mathbb{M}^{(2,1)}_{t+1} + o(1)\right),\hspace*{3cm}
	\end{aligned}
	\end{equation}
	\begin{equation}\label{eqn:hmc1}
	\begin{aligned}
	\hspace*{-6mm}\textrm{ where }  h^{(2,1)}(x) \triangleq
	\mathbb{E}\left[\mathsf{g}_{2}\left(\mathcal{H}(\mathsf{X}_{t+1}), \mathsf{X}_{t+1}, \gamma_{\rho}(\mathcal{H}, \widehat{\theta}), \xi^{(0)}_{*}\right) \Big| \mathcal{F}_{t}\right] -\\ \mathbb{E}\left[ x \mathsf{g}_{0}\left(\mathcal{H}(\mathsf{X}_{t+1}), \gamma_{\rho}(\mathcal{H},\widehat{\theta})\right) \Big|  \mathcal{F}_{t}\right] \textrm{ and } 
	\end{aligned}
	\end{equation}
	\begin{equation}
	\begin{aligned}
	\mathbb{M}^{(2,1)}_{t+1}& \triangleq \mathbb{E}\left[\xi^{(1)}_{t}\mathsf{g}_{0}\left(\mathcal{H}(\mathsf{X}_{t+1}), \gamma_{\rho}(\mathcal{H},\widehat{\theta})\right) \Big\vert \mathcal{F}_t\right] - \\ \xi^{(1)}_{t}&\mathsf{g}_{0}\left(\mathcal{H}(\mathsf{X}_{t+1}), \gamma_{\rho}(\mathcal{H},\widehat{\theta})\right)   -  \mathbb{E}\left[\mathsf{g}_{2}\left(\mathcal{H}(\mathsf{X}_{t+1}), \mathsf{X}_{t+1}, \gamma_{\rho}(\mathcal{H}, \widehat{\theta}), \xi^{(0)}_{*}\right) \Big| \mathcal{F}_{t}\right] + \\ &\mathsf{g}_{2}\left(\mathcal{H}(\mathsf{X}_{t+1}), \mathsf{X}_{t+1}, \gamma_{\rho}(\mathcal{H}, \widehat{\theta}), \xi^{(0)}_{*}\right), \textrm{ where } \mathsf{X}_{t+1} \sim \widehat{f}_{\theta}.
	\end{aligned}
	\end{equation}
	Since $\mathsf{X}_{t+1}$ is independent of the $\sigma$-field $\mathcal{F}_{t}$,
	the function $h^{(2,1)}(\cdot)$ in Equation (\ref{eqn:hmc1}) can be rewritten as
	\begin{equation}
	\begin{aligned}
	h^{(2,1)}(x) = 
	\mathbb{E}_{\widehat{\theta}}\left[\mathsf{g}_{2}\left(\mathcal{H}(\mathsf{X}), \mathsf{X}, \gamma_{\rho}(\mathcal{H}, \widehat{\theta}), \xi^{(0)}_{*}\right)\right] - \mathbb{E}_{\widehat{\theta}}\left[x \mathsf{g}_{0}\left(\mathcal{H}(\mathsf{X}), \gamma_{\rho}(\mathcal{H},\widehat{\theta})\right)\right],
	\end{aligned}
	\end{equation}
	where $\mathsf{X} \sim \widehat{f}_{\theta}(\cdot)$. It is not difficult to verify that $\mathbb{M}^{(2,1)}_{t+1}$, $t \in \mathbb{N}$ is a martingale difference noise sequence and $h^{(2,1)}(\cdot)$ is Lipschitz continuous. Also since $S(\cdot)$ is bounded and $\widehat{f}_{\theta}(\cdot)$ has finite first and second moments we get,
	\[\mathbb{E}\left[\Vert \mathbb{M}^{(2,1)}_{t+1} \Vert^{2} \vert \mathcal{F}_{t}\right] \leq K_{2,1}(1+\Vert \xi^{(1)}_t \Vert^{2}), \forall t \in \mathbb{N}, \textrm{ for some } 0 < K_{2,1} < \infty. \]
	Now consider the ODE given by
	\begin{equation}\label{eq:ode3}
	\begin{aligned}
	\frac{d}{dt}\xi^{(1)}(t) = h^{(2,1)}(\xi^{(1)}(t)), \hspace*{5mm} t \in \bbbr_{+}.
	\end{aligned}
	\end{equation} 
	By rewriting the above equation we get,
	\[\frac{d}{dt}\xi^{(1)}(t) = A\xi^{(1)}(t) + b^{(1)}, \hspace*{5mm} t \in \bbbr_{+},\]
	where $A$ is a diagonal matrix as before, \emph{i.e.},  $A_{ii} = -\mathbb{E}_{\widehat{\theta}}\left[\mathsf{g}_{0}\left(\mathcal{H}(\mathsf{X}), \gamma_{\rho}(\mathcal{H}, \widehat{\theta})\right)\right]$, $\forall i, 0 \leq i < k$ and $b^{(1)} = \mathbb{E}_{\widehat{\theta}}\left[\mathsf{g}_{2}\left(\mathcal{H}(\mathsf{X}), \mathsf{X}, \gamma_{\rho}(\mathcal{H}, \widehat{\theta}), \xi^{(0)}_{*}\right)\right]$. Now consider the ODE in the $\infty$-system 
	\[\frac{d}{dt}{\xi}^{(1)}(t) = \lim_{\eta \rightarrow \infty}\frac{1}{\eta}h^{(2,1)}(\eta{\xi}^{(1)}(t)) = A{\xi}^{(1)}(t).\]
	Again, the eigenvalue $\lambda(A)$ =  $-\mathbb{E}_{\widehat{\theta}}\left[\mathsf{g}_{0}\left(\mathcal{H}(\mathsf{X}), \gamma_{\rho}(\mathcal{H}, \widehat{\theta})\right)\right]$ of $A$ is negative and is of multiplicity $m$ and hence origin is the unique globally asymptotically stable equilibrium of the $\infty$-system. Therefore it follows that the iterates $\{{\xi}^{(1)}_{t}\}_{t \in \mathbb{N}}$ are almost surely stable, \emph{i.e.}, $\sup_{t \in \mathbb{N}}{\Vert {\xi}^{(0)}_{t} \Vert} < \infty$ \emph{a.s.}, see Theorem 7, Chapter 3 of \cite{borkar2008stochastic}.
	
	Again, by using the earlier argument that the eigenvalues $\lambda(A)$ of $A$ are negative and identical, the point $-A^{-1}b^{(1)}$ can be seen to be a globally asymptotically stable equilibrium of the ODE (\ref{eq:ode3}). By Corollary 4, Chapter 2 of \cite{borkar2008stochastic},  it follows that 
	\begin{equation*}
	\lim_{t \rightarrow \infty}{\xi}^{(1)}_t = -A^{-1}b^{(1)}\hspace*{2mm}a.s. = \frac{\mathbb{E}_{\widehat{\theta}}\left[ \mathsf{g}_{2}\left(\mathcal{H}(\mathsf{X}), \mathsf{X}, \gamma_{\rho}(\mathcal{H}, \widehat{\theta}), \xi^{(0)}_{*}\right)\right]}{\mathbb{E}_{\widehat{\theta}}\left[\mathsf{g}_{0}\left(\mathcal{H}(\mathsf{X}), \gamma_{\rho}(\mathcal{H}, \widehat{\theta})\right)\right]}\hspace*{2mm} a.s.\\
	\end{equation*}
	This completes the proof of Equation (\ref{eqn:xi1limit}).\\\\
	$3.$ Here also we assume $\theta_t \equiv \theta$. Then $\gamma_{t}$ in recursion (\ref{eqn:ce2ndgamma}) and $\gamma^{p}_{t}$ in recursion (\ref{eqn:ce2ndgammap}) converge to $\gamma_{\rho}(\mathcal{H}, \widehat{\theta})$  and  $\gamma_{\rho}(\mathcal{H}, \widehat{\theta^{p}})$ respectively. So if $\gamma_{\rho}(\mathcal{H}, \widehat{\theta}) > \gamma_{\rho}(\mathcal{H}, \widehat{\theta^{p}})$, then $\gamma_{t} > \gamma^{p}_{t}$ eventually, \emph{i.e.}, $\gamma_{t} > \gamma^{p}_{t}$ for all but finitely many $t$. So almost surely $T_t$ in  Equation (\ref{eqn:ce2ndTt}) will converge to $\mathbb{E}\left[\mathbb{I}_{\{\gamma_{t+1} > \gamma^{p}_{t+1}\}} - \mathbb{I}_{\{\gamma_{t+1} \leq \gamma^{p}_{t+1}\}}\right]$ = $\P(\gamma_{t+1} > \gamma^{p}_{t+1}) - \P(\gamma_{t+1} \leq \gamma^{p}_{t+1}) = 1-0 = 1$.
\end{proof}

\noindent\textbf{Notation: }For the subsequence $\{t_{(n)}\}_{n > 0}$ of $\{t\}_{t \geq 0}$, we denote $t^{-}_{(n)} \triangleq t_{(n)}-1$ for $n > 0$.\\
As mentioned earlier, $\theta_t$ is updated only along a subsequence $\{t_{(n)}\}_{n \geq 0}$ of $\{t\}_{t \geq 0}$ with $t_0 = 0$ as follows:
\begin{equation}\label{eqn:thetareal}
\theta_{t_{(n+1)}} = \theta_{t_{(n)}} + \beta_{t_{(n+1)}}\left(({\xi}^{(0)}_{t^{-}_{(n+1)}}, {\xi}^{(1)}_{t^{-}_{(n+1)}})^{\top} - \theta_{t_{(n)}}\right).
\end{equation} 
Now we define $\Psi(\theta) = (\Psi_1(\theta), \Psi_2(\theta))^{\top}$, where 
\begin{flalign}\label{eqn:psice2-nd}
&\Psi_1(\theta) \triangleq \frac{\mathbb{E}_{\widehat{\theta}}\left[ \mathsf{g}_{1}\left(\mathcal{H}(\mathsf{X}), \mathsf{X}, \gamma_{\rho}(\mathcal{H}, \widehat{\theta})\right)\right]}{\mathbb{E}_{\widehat{\theta}}\left[\mathsf{g}_{0}\left(\mathcal{H}(\mathsf{X}), \gamma_{\rho}(\mathcal{H}, \widehat{\theta})\right)\right]},\\
&\Psi_2(\theta) \triangleq  \frac{\mathbb{E}_{\widehat{\theta}}\left[\mathsf{g}_{2}\left(\mathcal{H}(\mathsf{X}), \mathsf{X}, \gamma_{\rho}(\mathcal{H}, \widehat{\theta}), \Psi_1(\theta)\right)\right]}{\mathbb{E}_{\widehat{\theta}}\left[\mathsf{g}_{0}\left(\mathcal{H}(\mathsf{X}), \gamma_{\rho}(\mathcal{H}, \widehat{\theta})\right)\right]}.
\end{flalign}

We now state our main theorem. The theorem states that the model sequence $\{\theta_{t}\}$ generated by Algorithm \ref{algo:ce2det-nd} converges to $\theta^{*} = (x^{*}, 0_{m \times m})^{\top}$, which is the degenerate distribution concentrated at $x^{*}$.
\begin{theorem}\label{thm:ce2nddetmain}
	Let $S(x) = exp(rx)$, $r \in \mathbb{R}_{+}$.  Let $\rho \in (0,1)$ and $\lambda \in (0,1)$. Let $\theta_0 = (\mu_0, q\I_{m \times m})^{\top}$, where $q \in \mathbb{R}_{+}$. Let the step-size sequence $\{\beta_t\}$ satisfy the Assumption \ref{assm:lnrtce2nd}. Also let $c_t \rightarrow 0$ as $t \rightarrow \infty$. Assume that both the solution space $\mathcal{X}$ and the parameter space $\Theta$ are compact. Let $\{\theta_t = (\mu_t, \Sigma_t)^{\top}\}_{t \in \mathbb{N}}$ be the sequence generated by CE2-ND (Algorithm \ref{algo:ce2det-nd}) and assume $\theta_{t} \in \Theta$, $\forall t \in \mathbb{N}$. Also, let  Assumption \ref{assm:ce2ndgmbd} hold. Assume that the objective function $\mathcal{H} \in \mathcal{C}^{2}$. Further, we assume that there exists a continuously differentiable function $V:\Theta \rightarrow \mathbb{R}_{+}$ s.t. $\nabla V(\theta)^{\top}\Psi(\theta) < 0$, $\forall \theta \in \Theta\smallsetminus\{\theta^{*}\}$ and $\nabla V(\theta^{*})^{\top}\Psi(\theta^{*}) = 0$. Then, there exists $q^{*} \in \mathbb{R}_{+}$ and $r^{*} \in \mathbb{R}_{+}$ s.t. $\forall q > q^{*}$ and $\forall r > r^{*}$,
	\begin{gather*}
	\hspace{1mm} \lim_{t \rightarrow \infty} \mathcal{H}(\mu_{t}) = \mathcal{H}(x^{*})  \hspace{3mm} and \hspace{2mm} \lim_{t \rightarrow \infty}\theta_{t} = \theta^{*} = (x^{*}, 0_{m \times m})^{\top} \textrm{ almost surely},
	\end{gather*}
	where $x^{*}$ is defined in Equation (\ref{eqn:detoptprb}).
\end{theorem}
\begin{proof}
	Rewriting the Equation (\ref{eq:thupd}) along the subsequence $\{t_{(n)}\}_{n \in \mathbb{N}}$, we have for $n \in \mathbb{N}$,
	\begin{equation}\label{eqn:thetap1}
	\theta_{t_{(n+1)}} = \theta_{t_{(n)}} + \beta_{t_{(n+1)}}\left(({\xi}^{(0)}_{t^{-}_{(n+1)}}, {\xi}^{(1)}_{t^{-}_{(n+1)}})^{\top} - \theta_{t_{(n)}}\right).
	\end{equation} 
	The iterates $\theta_{t_{(n)}}$ are stable, \emph{i.e.}, $\sup_{n}{\Vert \theta_{t_{(n)}}\Vert} < \infty$ \emph{a.s.} It is directly implied from the hypothesis that $\theta_{t_{(n)}} \in \Theta$ and $\Theta$ is a compact set.\\\\
	Rearranging the Equation (\ref{eqn:thetap1}) we get, for $n \in \mathbb{N}$,
	\begin{equation}
	\theta_{t_{(n+1)}} = \theta_{t_{(n)}} + \beta_{t_{(n+1)}}\left(\Psi(\theta_{t_{(n)}}) + \mathit{o}(1)\right).
	\end{equation}
	This easily follows from the fact that, for $t_{(n)} < t \leq t_{(n+1)}$, the random variables ${\xi}^{(0)}_t$ and ${\xi}^{(1)}_t$ estimate the quantities $\Psi_1(\theta_{t_{(n)}})$ and $\Psi_2(\theta_{t_{(n)}})$ respectively. Since $c_t \rightarrow 0$, the estimation error decays to $0$. This accounts for the $\mathit{o}(1)$ term.\\
	
	The limit points of the above recursion are indeed the roots of $\Psi$. Hence by equating $\Psi_1(\theta)$ to $0_{m \times 1}$, we get,
	\begin{flalign*} 
	&\frac{\mathbb{E}_{\widehat{\theta}}\left[\mathsf{g_{1}}\bm{\big{(}}\mathcal{H}(\mathsf{X}), \mathsf{X},  \gamma_{\rho}(\mathcal{H}, \widehat{\theta})\bm{\big{)}}\right]}{\mathbb{E}_{\widehat{\theta}}\left[\mathsf{g_{0}}\bm{\big{(}}\mathcal{H}(\mathsf{X}), \gamma_{\rho}(\mathcal{H}, \widehat{\theta})\bm{\big{)}}\right]} - \mu = 0_{m \times 1}.\\
	&\Rightarrow (1-\lambda)\mathbb{E}_{\theta}\left[\mathsf{g_{1}}\bm{\big{(}}\mathcal{H}(\mathsf{X}), \mathsf{X},  \gamma_{\rho}(\mathcal{H}, \widehat{\theta})\bm{\big{)}}\right] + \lambda\mathbb{E}_{\theta_0}\left[\mathsf{g_{1}}\bm{\big{(}}\mathcal{H}(\mathsf{X}), \mathsf{X},  \gamma_{\rho}(\mathcal{H}, \widehat{\theta})\bm{\big{)}}\right] - \\&\hspace*{10mm}(1-\lambda)\mu\mathbb{E}_{\theta}\left[\mathsf{g_{0}}\bm{\big{(}}\mathcal{H}(\mathsf{X}), \gamma_{\rho}(\mathcal{H}, \widehat{\theta})\bm{\big{)}}\right] - \lambda\mu\mathbb{E}_{\theta_0}\left[\mathsf{g_{0}}\bm{\big{(}}\mathcal{H}(\mathsf{X}), \gamma_{\rho}(\mathcal{H}, \widehat{\theta})\bm{\big{)}}\right] = 0_{m \times 1}.\\
	&\Rightarrow (1-\lambda)\mathbb{E}_{\theta}\left[(X - \mu)\mathsf{g_{0}}\bm{\big{(}}\mathcal{H}(\mathsf{X}), \mathsf{X},  \gamma_{\rho}(\mathcal{H}, \widehat{\theta})\bm{\big{)}}\right] + \\&\hspace*{20mm}\lambda\big(\mathbb{E}_{\theta_0}\left[\mathsf{g_{1}}\bm{\big{(}}\mathcal{H}(\mathsf{X}), \mathsf{X},  \gamma_{\rho}(\mathcal{H}, \widehat{\theta})\bm{\big{)}}\right] - \mu\mathbb{E}_{\theta_0}\left[\mathsf{g_{0}}\bm{\big{(}}\mathcal{H}(\mathsf{X}), \gamma_{\rho}(\mathcal{H}, \widehat{\theta})\bm{\big{)}}\right]\big) = 0_{m \times 1}.
	\end{flalign*}
	By applying the ``integration by parts" rule for multivariate Gaussian, we obtain
	\begin{flalign*}
	&(1-\lambda)\Sigma\mathbb{E}_{\theta}\left[S(\mathcal{H}(\mathsf{X}))\nabla\mathcal{H}(\mathsf{X}))\mathbb{I}_{\{\mathcal{H}(\mathsf{X}) \geq \gamma_{\rho}(\mathcal{H}, \widehat{\theta})\}}\right] + \\&\hspace*{10mm}\lambda\Big(q\mathbb{E}_{\theta_0}\left[S(\mathcal{H}(\mathsf{X}))\nabla\mathcal{H}(\mathsf{X}))\mathbb{I}_{\{\mathcal{H}(\mathsf{X}) \geq \gamma_{\rho}(\mathcal{H}, \widehat{\theta})\}}\right] - \mu\mathbb{E}_{\theta_0}\left[\mathsf{g_{0}}\bm{\big{(}}\mathcal{H}(\mathsf{X}), \gamma_{\rho}(\mathcal{H}, \widehat{\theta})\bm{\big{)}}\right]\Big) = 0_{m \times 1}.
	\end{flalign*}\\
	For brevity, we define 
	\begin{equation}\label{eq:quanstar}
	\gamma_{\rho}(\theta) \triangleq \gamma_{\rho}(\mathcal{H}, \theta) \hspace*{2mm} \textrm{and }\mathsf{\hat{g}_{0}}(x, \theta) \triangleq \mathsf{g_{0}}\bm{\big{(}}\mathcal{H}(x), \gamma_{\rho}(\theta)\big{)}.
	\end{equation}
	Therefore, the above equation becomes
	\begin{flalign}\label{eq:withbsline}
	&(1-\lambda)\Sigma\mathbb{E}_{\theta}\left[S(\mathcal{H}(\mathsf{X}))\nabla\mathcal{H}(\mathsf{X}))\mathbb{I}_{\{\mathcal{H}(\mathsf{X}) \geq \gamma_{\rho}(\widehat{\theta})\}}\right] + \nonumber\\&\hspace*{20mm}\lambda\Big(q\mathbb{E}_{\theta_0}\left[S(\mathcal{H}(\mathsf{X}))\nabla\mathcal{H}(\mathsf{X}))\mathbb{I}_{\{\mathcal{H}(\mathsf{X}) \geq \gamma_{\rho}(\widehat{\theta})\}}\right] - \mu\mathbb{E}_{\theta_0}\left[\mathsf{\hat{g}_{0}}(\mathsf{X}, \widehat{\theta})\right]\Big) = 0_{m \times 1}.
	\end{flalign}\\
	By adding the baseline $\xi^{(0)}_t\mathsf{\widehat{g}_{0}}(\mathsf{X}, \widehat{\theta})$ with $\mathsf{X} \sim f_{\theta_0}$ to the recursion of $\xi^{(0)}_t$, one can drop the component $\mu\mathbb{E}_{\theta_0}\left[\mathsf{\widehat{g}_{0}}(\mathsf{X}, \widehat{\theta})\right]$ from the above equation. This indeed simplifies the analysis and does not affect the asymptotic behaviour. However, in practical cases, instead of adding the baseline, one can choose $\lambda$ small enough and $q$ large enough to nullify the effect of the component $\mu\mathbb{E}_{\theta_0}\left[\mathsf{\widehat{g}_{0}}(\mathsf{X}, \widehat{\theta})\right]$. Hence, in the analysis, we consider the following equation instead of Equation (\ref{eq:withbsline}), \emph{i.e.},
	\begin{flalign}\label{eq:wobsline}
	(1-\lambda)\Sigma\mathbb{E}_{\theta}&\left[S(\mathcal{H}(\mathsf{X}))\nabla\mathcal{H}(\mathsf{X}))\mathbb{I}_{\{\mathcal{H}(\mathsf{X}) \geq \gamma_{\rho}(\widehat{\theta})\}}\right] + \nonumber\\&\lambda q\mathbb{E}_{\theta_0}\left[S(\mathcal{H}(\mathsf{X}))\nabla\mathcal{H}(\mathsf{X}))\mathbb{I}_{\{\mathcal{H}(\mathsf{X}) \geq \gamma_{\rho}(\widehat{\theta})\}}\right] = 0_{m \times 1}.
	\end{flalign}\\\\
	Similarly, by equating $\Psi_2(\theta)$ to $\mathbb{O}$ $(= 0_{m \times m})$, we get,
	\begin{flalign}\label{eqn:sgmsol2}
	&\frac{\mathbb{E}_{\widehat{\theta}}\left[\mathsf{g_{2}}\big{(}\mathcal{H}(\mathsf{X}), \mathsf{X}, \gamma_{\rho}(\mathcal{H},\widehat{\theta}), \mu\big{)}\right]}{\mathbb{E}_{\widehat{\theta}}\left[\mathsf{g_{0}}\big{(}\mathcal{H}(\mathsf{X}), \gamma_{\rho}(\mathcal{H}, \widehat{\theta})\big{)}\right]} -  \Sigma = \mathbb{O}.\nonumber\\
	&\Longrightarrow (1-\lambda)\mathbb{E}_{\theta}\left[\mathsf{g_{2}}\big{(}\mathcal{H}(\mathsf{X}), \mathsf{X}, \gamma_{\rho}(\mathcal{H},\widehat{\theta}), \mu\big{)}\right] + \lambda{E}_{\theta_0}\left[\mathsf{g_{2}}\big{(}\mathcal{H}(\mathsf{X}), \mathsf{X}, \gamma_{\rho}(\mathcal{H},\widehat{\theta}), \mu\big{)}\right] - \nonumber\\&\hspace*{10mm} (1-\lambda)\Sigma\mathbb{E}_{\theta}\left[\mathsf{g_{0}}\big{(}\mathcal{H}(\mathsf{X}), \gamma_{\rho}(\mathcal{H}, \widehat{\theta})\big{)}\right] - \lambda\Sigma\mathbb{E}_{\theta_0}\left[\mathsf{g_{0}}\big{(}\mathcal{H}(\mathsf{X}), \gamma_{\rho}(\mathcal{H}, \widehat{\theta})\big{)}\right]= \mathbb{O}.\nonumber\\
	&\Longrightarrow (1-\lambda)\mathbb{E}_{\theta}\left[(\mathsf{X}-\mu)(\mathsf{X}-\mu)^{\top}\widehat{\mathsf{g_{0}}}(\mathsf{X}, \widehat{\theta})\right] + \nonumber\\&\hspace*{10mm}  \lambda{E}_{\theta_0}\left[\mathsf{g_{2}}\big{(}\mathcal{H}(\mathsf{X}), \mathsf{X}, \gamma_{\rho}(\mathcal{H},\widehat{\theta}), \mu\big{)}\right] - \Sigma\mathbb{E}_{\widehat{\theta}}\left[\mathsf{g_{0}}\big{(}\mathcal{H}(\mathsf{X}), \gamma_{\rho}(\mathcal{H}, \widehat{\theta})\big{)}\right] = \mathbb{O}.\nonumber\\
	&\Longrightarrow (1-\lambda)\mathbb{E}_{\theta}\left[(\mathsf{X}-\mu)(\mathsf{X}-\mu)^{\top}\widehat{\mathsf{g_{0}}}(\mathsf{X}, \widehat{\theta})\right] + \lambda{E}_{\theta_0}\left[\mathsf{g_{2}}\big{(}\mathcal{H}(\mathsf{X}), \mathsf{X}, \gamma_{\rho}(\mathcal{H},\widehat{\theta}), \mu\big{)}\right] - \nonumber\\&\hspace*{20mm} (1-\lambda)\mathbb{E}_{\theta}\left[(\mathsf{X}-\mu)(\mathsf{X}-\mu)^{\top}\right]\mathbb{E}_{\theta}\left[\widehat{\mathsf{g_{0}}}(\mathsf{X}, \widehat{\theta})\right] - \lambda\Sigma\mathbb{E}_{\theta_0}\left[\widehat{\mathsf{g_{0}}}(\mathsf{X}, \widehat{\theta})\right]= \mathbb{O}.\nonumber\\
	&\Longrightarrow (1-\lambda)\mathbb{E}_{\theta}\left[(\mathsf{X}-\mu)(\mathsf{X}-\mu)^{\top}\Big(\widehat{\mathsf{g_{0}}}(\mathsf{X}, \widehat{\theta})-\mathbb{E}_{\theta}\left[\widehat{\mathsf{g_{0}}}(\mathsf{X}, \widehat{\theta})\right]\Big)\right] + \nonumber\\&\hspace*{15mm}\lambda{E}_{\theta_0}\left[\mathsf{X}\mathsf{X}^{\top}\Big(\widehat{\mathsf{g_{0}}}(\mathsf{X}, \widehat{\theta})-\mathbb{E}_{\theta_0}\left[\widehat{\mathsf{g_{0}}}(\mathsf{X}, \widehat{\theta})\right]\Big)\right] - \nonumber\\&\hspace*{20mm} \lambda\Bigg(\mu\mathbb{E}_{\theta_0}\left[\mathsf{X}\Big(\widehat{\mathsf{g_{0}}}(\mathsf{X}, \widehat{\theta})-\mathbb{E}_{\theta_0}\left[\widehat{\mathsf{g_{0}}}(\mathsf{X}, \widehat{\theta})\right]\Big)\right]^{\top} + \nonumber\\&\hspace*{30mm} \mathbb{E}_{\theta_0}\left(\mathsf{X}\Big(\widehat{\mathsf{g_{0}}}(\mathsf{X}, \widehat{\theta})-\mathbb{E}_{\theta_0}\left[\widehat{\mathsf{g_{0}}}(\mathsf{X}, \widehat{\theta})\right]\Big)\right]\mu^{\top}\Bigg) = \mathbb{O}.
	\end{flalign}
	Similar to the earlier case, by arguing along the same line, one can discard the term $\lambda\Bigg(\mu\mathbb{E}_{\theta_0}\left[\mathsf{X}\Big(\widehat{\mathsf{g_{0}}}(\mathsf{X}, \widehat{\theta})-\mathbb{E}_{\theta_0}\left[\widehat{\mathsf{g_{0}}}(\mathsf{X}, \widehat{\theta})\right]\Big)\right]^{\top} + \mathbb{E}_{\theta_0}\left(\mathsf{X}\Big(\widehat{\mathsf{g_{0}}}(\mathsf{X}, \widehat{\theta})-\mathbb{E}_{\theta_0}\left[\widehat{\mathsf{g_{0}}}(\mathsf{X}, \widehat{\theta})\right]\Big)\right]\mu^{\top}\Bigg)$  and consider the following equation instead of Equation  (\ref{eqn:sgmsol2}) for analysis, \emph{i.e.},
	\begin{flalign}\label{eq:wobssigma}
	&(1-\lambda)\mathbb{E}_{\theta}\left[(\mathsf{X}-\mu)(\mathsf{X}-\mu)^{\top}\Big(\widehat{\mathsf{g_{0}}}(\mathsf{X}, \widehat{\theta})-\mathbb{E}_{\theta}\left[\widehat{\mathsf{g_{0}}}(\mathsf{X}, \widehat{\theta})\right]\Big)\right] + \nonumber\\&\hspace*{15mm}\lambda{E}_{\theta_0}\left[\mathsf{X}\mathsf{X}^{\top}\Big(\widehat{\mathsf{g_{0}}}(\mathsf{X}, \widehat{\theta})-\mathbb{E}_{\theta_0}\left[\widehat{\mathsf{g_{0}}}(\mathsf{X}, \widehat{\theta})\right]\Big)\right] = \mathbb{O}.
	\end{flalign}
	Again, by applying the ``integration by parts" rule for multivariate Gaussian in Equation (\ref{eq:wobssigma}) and using the hypothesis $S(x) = exp(rx)$, we obtain
	\begin{flalign}\label{eqn:zeq}
	&(1-\lambda)\Sigma^{2}\mathbb{E}_{\theta}\left[\nabla_{x}^{2}\mathsf{\hat{g}_{0}}(\mathsf{X}, \widehat{\theta})\right] + \lambda q^{2}\mathbb{E}_{\theta_{0}}\left[\nabla_{x}^{2}\mathsf{g_{0}}(\mathsf{X}, \widehat{\theta})\right] = \mathbb{O}.\nonumber\\
	&\Longrightarrow \hspace*{4mm}(1-\lambda)\Sigma^{2} \mathbb{E}_{\theta}\left[S(\mathcal{H}(\mathsf{X}))G^{r}(\mathsf{X})\mathbb{I}_{\{\mathcal{H}(\mathsf{X}) \geq \gamma_{\rho}(\widehat{\theta})\}}\right] + \nonumber\\&\hspace*{3cm}\lambda q^{2} \mathbb{E}_{\theta_0}\left[S(\mathcal{H}(\mathsf{X}))G^{r}(\mathsf{X})\mathbb{I}_{\{\mathcal{H}(\mathsf{X}) \geq \gamma_{\rho}(\widehat{\theta})\}}\right] = \mathbb{O},
	\end{flalign}
	where $G^{r}(x) \triangleq r^{2}\nabla\mathcal{H}(x)\nabla\mathcal{H}(x)^{\top} + r\nabla^{2}\mathcal{H}(x)$. Note that for each $x \in \mathcal{X}$, $G^{r}(x) \in \mathbb{R}^{m \times m}$. Hence we denote $G^{r}(x)$ as $\left[G^{r}_{ij}(x)\right]_{i=1,j=1}^{m}$. For brevity, we also define
	\begin{gather}
	F^{r, \rho}(x, \theta) \triangleq S(\mathcal{H}(x))G^{r}(x)\mathbb{I}_{\{\mathcal{H}(x) \geq \gamma_{\rho}(\theta)\}},
	\end{gather}
	where $F^{r, \rho}(x, \theta) \in \mathbb{R}^{m \times m}$ which is also denoted as $\left[F^{r, \rho}_{ij}(x)\right]_{i=1,j=1}^{m}$.\\
	Hence Equation (\ref{eqn:zeq}) becomes,
	\begin{flalign}\label{eqn:zeq2}
	(1-\lambda)\Sigma^{2} \mathbb{E}_{\theta}\left[F^{r, \rho}(\mathsf{X}, \widehat{\theta})\right] + \lambda q^{2} \mathbb{E}_{\theta_0}\left[F^{r, \rho}(\mathsf{X}, \widehat{\theta})\right] = \mathbb{O}.
	\end{flalign}
	Note that $(\nabla_{i}\mathcal{H})^{2} \geq 0$. Hence we can find an $r^{*} > 0$ such that $G^{r}_{ii}(x) > 0$, $\forall r > r^{*}$, $1 \leq i \leq m$, $\forall x \in \mathcal{X}$. Also, by hypothesis, $\Theta$ is compact. Hence we can find $q^{*} > 0$ such that 
	\begin{flalign}
	&(1-\lambda)\left(\Sigma^{2}\mathbb{E}_{\theta}\left[F^{r, \rho}(\mathsf{X}, \widehat{\theta})\right]\right)_{ii} + \lambda q^{2} \mathbb{E}_{\theta_0}\left[F^{r, \rho}_{ii}(\mathsf{X}, \widehat{\theta})\right] \geq 0,\\&\hspace*{32mm}\forall r > r^{*}, \forall q > q^{*}, \forall \theta \in \Theta, 1 \leq \forall i \leq m.\nonumber
	\end{flalign}
	This contradicts the equality in Equation (\ref{eqn:zeq2}) for $q > q^{*}$ and $r > r^{*}$.  Hence for such choices of $q$ and $r$, each of the terms in Equation (\ref{eqn:zeq2}) is $0$, \emph{i.e.},
	\begin{flalign}
	&\Sigma^{2} \mathbb{E}_{\theta}\left[F^{r, \rho}(\mathsf{X}, \widehat{\theta})\right] = \mathbb{O}\hspace*{10mm}\textrm{ and }\label{eq:zeq3}\\ 
	&q^{2} \mathbb{E}_{\theta_0}\left[F^{r, \rho}(\mathsf{X}, \widehat{\theta})\right]  = \mathbb{O}.\label{eq:zeq4}
	\end{flalign}
	Now from Equation (\ref{eq:zeq4}), we have  
	\begin{equation}\label{eq:degdistres}
	\begin{aligned}
	\mathbb{E}_{\theta_0}\left[F^{r, \rho}(\mathsf{X}, \widehat{\theta})\right]  = \mathbb{O} \hspace*{3mm} \Longrightarrow& \hspace*{3mm}\gamma_{\rho}(\widehat{\theta}) = \mathcal{H}(x^{*}),\\ &\hspace*{2mm} \forall r > r^{*}, \forall \rho \in(0, \rho^{*}), \forall q > q^{*}.
	\end{aligned}
	\end{equation}
	The above implication is trivial, since for all thresholds $\gamma \in \left[\inf_{x \in \mathcal{X}}\mathcal{H}(x), \sup_{x \in \mathcal{X}}\mathcal{H}(x)\right]$ which are strictly less that $\mathcal{H}(x^{*}), $ we have $\left(\mathbb{E}_{\theta_0}\left[S(\mathcal{H}(\mathsf{X}))G^{r}(\mathsf{X})\mathbb{I}_{\{\mathcal{H}(\mathsf{X}) \geq \gamma\}}\right]\right)_{ii} > 0$, $\forall 1 \leq i \leq m$, $\forall r > r^{*}$.
	
	Now from Equation (\ref{eq:zeq3}), we have  $\Sigma = \mathbb{O}$. Indeed, if  $\Sigma \neq \mathbb{O}$, then $\Sigma^{2}$ is invertible (follows because $\Sigma^{2}$ is positive definite). Hence we have $\mathbb{E}_{\theta}\left[F^{r, \rho}(\mathsf{X}, \widehat{\theta})\right] = \mathbb{O}$. However, this is a contradiction, since $\forall r > r^{*}$, we have $F^{r, \rho}_{ii}(x, \widehat{\theta}) > 0$. 
	
	It is also easy to verify that the solution so obtained, \emph{i.e.}, $\Sigma = \mathbb{O}$ and $\gamma_{\rho}(\widehat{\theta}) = \mathcal{H}(x^{*})$ also satisfies the initial Equation  (\ref{eq:wobsline}). 
	
	This shows that for any $x \in \mathcal{X}$, the degenerate distribution concentrated on $x$ given by $\theta_x = (x, 0_{m \times m})^{\top}$ is a potential limit point of the recursion (\ref{eqn:thetap1}).\vspace*{0mm}\\
	
	Now we prove the following claim which effectively establishes that the only limiting distribution of the recursion (\ref{eqn:thetap1}) is indeed the degenerate distribution concentrated on $x^{*}$.\vspace*{2mm}\\
	\textbf{Claim (C1):} The only degenerate distribution which satisfies the condition $\gamma_{\rho}(\widehat{\theta}) = \mathcal{H}(x^{*})$ is $\theta^{*} = (x^{*}, 0_{m \times m})^{\top}$.\\
	The above claim can be verified as follows: if there exists $x^{\prime} (\in \mathcal{X}) \neq x^{*}$ s.t. $\gamma_{\rho}(\widehat{\theta}_{x^{\prime}}) = \mathcal{H}(x^{*})$ is satisfied, then from the definition of $\gamma_{\rho}(\cdot)$ in Equation (\ref{eq:quanstar}) and from Equation (\ref{eq:quantdef}), we can find an increasing sequence $\{l_i\}$, where $l_i > \mathcal{H}(x^{\prime})$ \emph{s.t.} the following property is satisfied:
	\begin{equation}\label{eq:sqeq}
	\lim_{i \rightarrow \infty}l_i = \mathcal{H}(x^{*}) \textrm{ and } \P_{\widehat{\theta_{x^{\prime}}}}(\mathcal{H}(\mathsf{X}) \geq l_i) \geq \rho.
	\end{equation}
	But $\P_{\widehat{\theta_{x^{\prime}}}}(\mathcal{H}(\mathsf{X}) \geq l_i) = (1-\lambda)\P_{\theta_{x^{\prime}}}(\mathcal{H}(\mathsf{X}) \geq l_i) + \lambda \P_{\theta_{0}}(\mathcal{H}(\mathsf{X}) \geq l_i)$ and $\P_{\theta_{x^{\prime}}}(\mathcal{H}(\mathsf{X}) \geq l_i) = 0$, $\forall i$. Therefore from Equation (\ref{eq:sqeq}), we get,
	\begin{gather*}
	\begin{aligned}
	&\P_{\widehat{\theta_{x^{\prime}}}}(\mathcal{H}(\mathsf{X}) \geq l_i) \geq \rho\\
	&\Rightarrow (1-\lambda)\P_{\theta_{x^{\prime}}}(\mathcal{H}(\mathsf{X}) \geq l_i) + \lambda \P_{\theta_{0}}(\mathcal{H}(\mathsf{X}) \geq l_i) \geq \rho\\
	&\Rightarrow \lambda \P_{\theta_{0}}(\mathcal{H}(\mathsf{X}) \geq l_i) \geq \rho\\
	&\Rightarrow \P_{\theta_{0}}(\mathcal{H}(\mathsf{X}) \geq l_i) \geq \frac{\rho}{\lambda}.
	\end{aligned}
	\end{gather*}
	In fact
	\begin{flalign}\label{eq:fcontr}
	\P_{\theta_{0}}(\mathcal{H}(\mathsf{X}) \geq l_i) \geq \min{(\frac{\rho}{\lambda}, 1)}.
	\end{flalign}
	Recall that $l_i \rightarrow \mathcal{H}(x^{*})$. Hence, by the continuity of probability measures and from Equation (\ref{eq:fcontr}), we get 
	\begin{equation*}
	0 = \P_{\theta_{0}}(\mathcal{H}(\mathsf{X}) \geq \mathcal{H}(x^{*})) = \lim_{i \rightarrow \infty}\P_{\theta_{0}}(\mathcal{H}(\mathsf{X}) \geq l_i) \geq \min{(\frac{\rho}{\lambda}, 1)} > 0,
	\end{equation*}
	which is a contradiction. This proves the Claim (C1). Now the only remaining task is to prove that $\theta^{*}$ is a stable attractor. This easily follows from the hypothesis regarding the existence of the Lyapunov function $V$ in the statement of the theorem.
\end{proof}
\subsection{Exogesis of Theorem \ref{thm:ce2nddetmain}}
Theorem \ref{thm:ce2nddetmain} provides a few insights into the nature of the algorithm CE2-ND. The theorem is more existential in nature which fundamentally claims the existence of lower bounds $q^{*}$ and $r^{*}$ for the parameters $q$ (the covariance parameter of the initial distribution) and $r$ (the scaling parameter of the weight function $S(x) = \exp{(rx)}$) respectively, which successfully drive the algorithm towards the global optimum. At first, we explore the nature of our algorithm by contrasting it with respect to the standard CE algorithm. Note that in the standard CE method, which is the initially proposed CE method, the weight function $S(\cdot)$ is literally not involved, \emph{i.e.}, $S \equiv 1$. In this case, the update procedure is given by
\begin{flalign}\label{eq:stdce}
&\mu_{t+1} = \frac{\sum_{i=1}^{N}\mathbb{I}_{\{\mathcal{H}(\mathsf{X}_i) \geq \gamma_{t+1}\}}\mathsf{X}_{i}}{\sum_{i=1}^{N}\mathbb{I}_{\{\mathcal{H}(\mathsf{X}_i) \geq \gamma_{t+1}\}}},\\
&\Sigma_{t+1} = \frac{\sum_{i=1}^{N}\mathbb{I}_{\{\mathcal{H}(\mathsf{X}_i) \geq \gamma_{t+1}\}}(\mathsf{X}_{i}-\mu_{t+1})(\mathsf{X}_{i}-\mu_{t+1})^{\top}}{\sum_{i=1}^{N}\mathbb{I}_{\{\mathcal{H}(\mathsf{X}_i) \geq \gamma_{t+1}\}}}.
\end{flalign}
To demonstrate the differences more vividly, we consider the following example:
\begin{example}\label{exp:exothmce2mn}
	\[  \mathcal{H}(x) = \left\{
	\begin{array}{ll}
	0 & x < -\delta \\          \frac{3}{\delta}x + 3 & -\delta \leq x \leq 0 \\          -\frac{3}{\delta}x+3 & 0 < x \leq \delta \\
	0 & x > \delta,
	\end{array} 
	\right. \]
	where $\delta > 0$.
	\begin{figure}[h]
		\centering
		\includegraphics[scale=0.85]{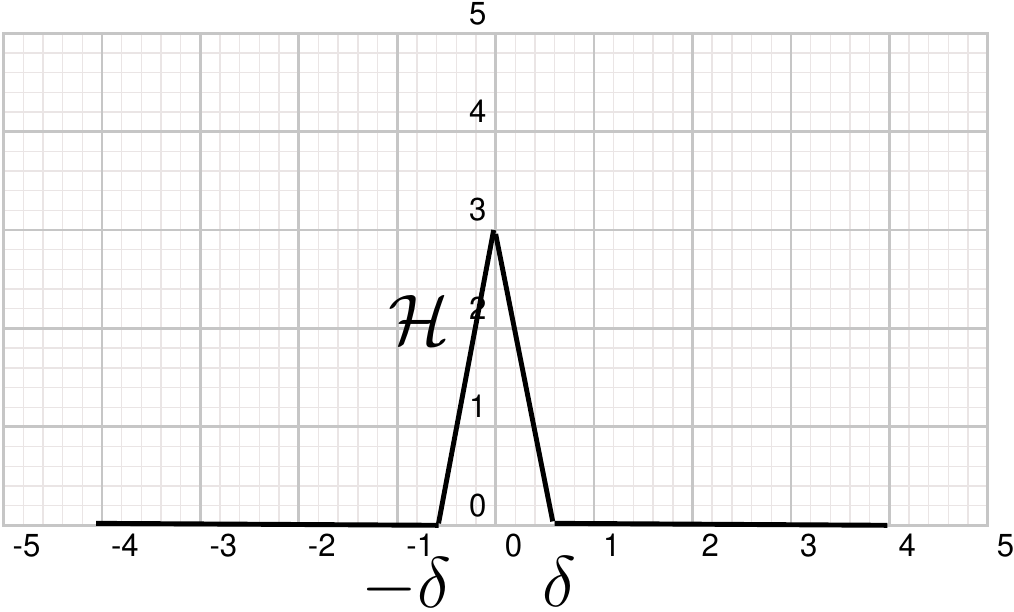}
	\end{figure}
\end{example}
Let us assume that the initial density parameter $\theta_0$ and the quantile parameter $\rho$ are already chosen. Now consider a $\delta > 0$ such that $\P_{\theta_0}\left(-\delta \leq \mathsf{X} \leq \delta\right) < \rho$. For such choice of $\delta$, we have $\gamma_{\rho}(\mathcal{H}, \theta_0) = 0$ and hence $\{x \vert \mathcal{H}(x) \geq \gamma_{\rho}(\mathcal{H}, \theta_0)\} = \mathcal{X}$. This follows directly from the definition of $\gamma_{\rho}(\cdot, \cdot)$. The above condition can be satisfied by taking $\rho=0.1$ and $\delta = 0.4$. In this situation, for the standard CE, we get $\theta_t = \theta_0$, $\forall t > 0$, \emph{i.e.}, the density parameters remain constant. So one has to be careful in choosing the quantile parameter $\rho$ to avoid such scenarios since the shape of the objective function is critical in seeking the global optimum. Thus there is a strong dependency between the standard CE and the quantile parameter $\rho$. So the primary objective for incorporating $S(\mathcal{H}(x))$ into the update rule of CE2-ND is to decouple this dependency. Thus the model parameter is updated by additionally conferring due consideration to the shape of the objective function. This non-dependency of our approach on $\rho$ is also corroborated by Theorem \ref{thm:ce2nddetmain} which does not propose any bounds on the quantile parameter $\rho$. Now consider the earlier example with $\gamma_1 = 0$ and same $\rho$ as before. We find that CE2-ND does show remarkable performance as illustrated in Fig. \ref{fig:rdiff1}.
\begin{figure}[h]
	\centering
	\includegraphics[scale=0.4]{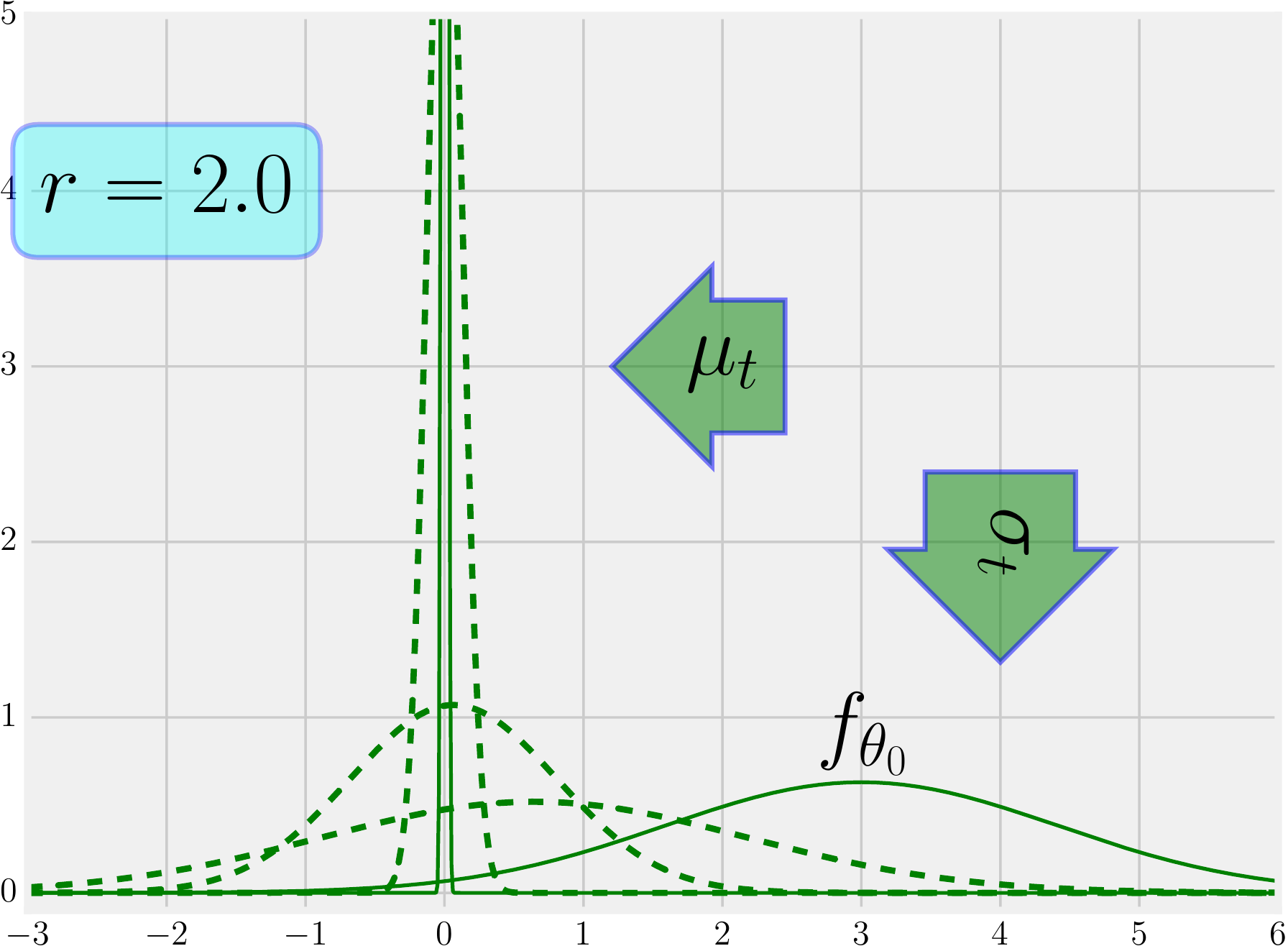}
	\caption{Example which illustrates the non-dependency of CE2-ND on $\rho$. Here, CE2-ND converges to the global optimum even with $\gamma_t = 0$, $\forall t$.}\label{fig:rdiff1}
\end{figure}

Even though the dependency on $\rho$ is being relaxed, we find that with $S(x) = \exp{(rx)}$, an additional dependency on the scaling parameter $r$ has emerged which seems to influence the evolutionary trajectory of the model parameters. This fact is already highlighted in Theorem \ref{thm:ce2nddetmain}, where the existence of the lower bound $r^{*}$ is emphasized. Additionally, we empirically illustrate this particular aspect of the algorithm. This is demonstrated in Fig. \ref{fig:rdiff2}, where the objective is same as earlier, but $r=1.0$. This choice of $r$ is in contrast to the earlier optimal behaviour shown in Fig. \ref{fig:rdiff1}, where $r=2.0$.
\begin{figure}[h]
	\centering
	\includegraphics[scale=0.4]{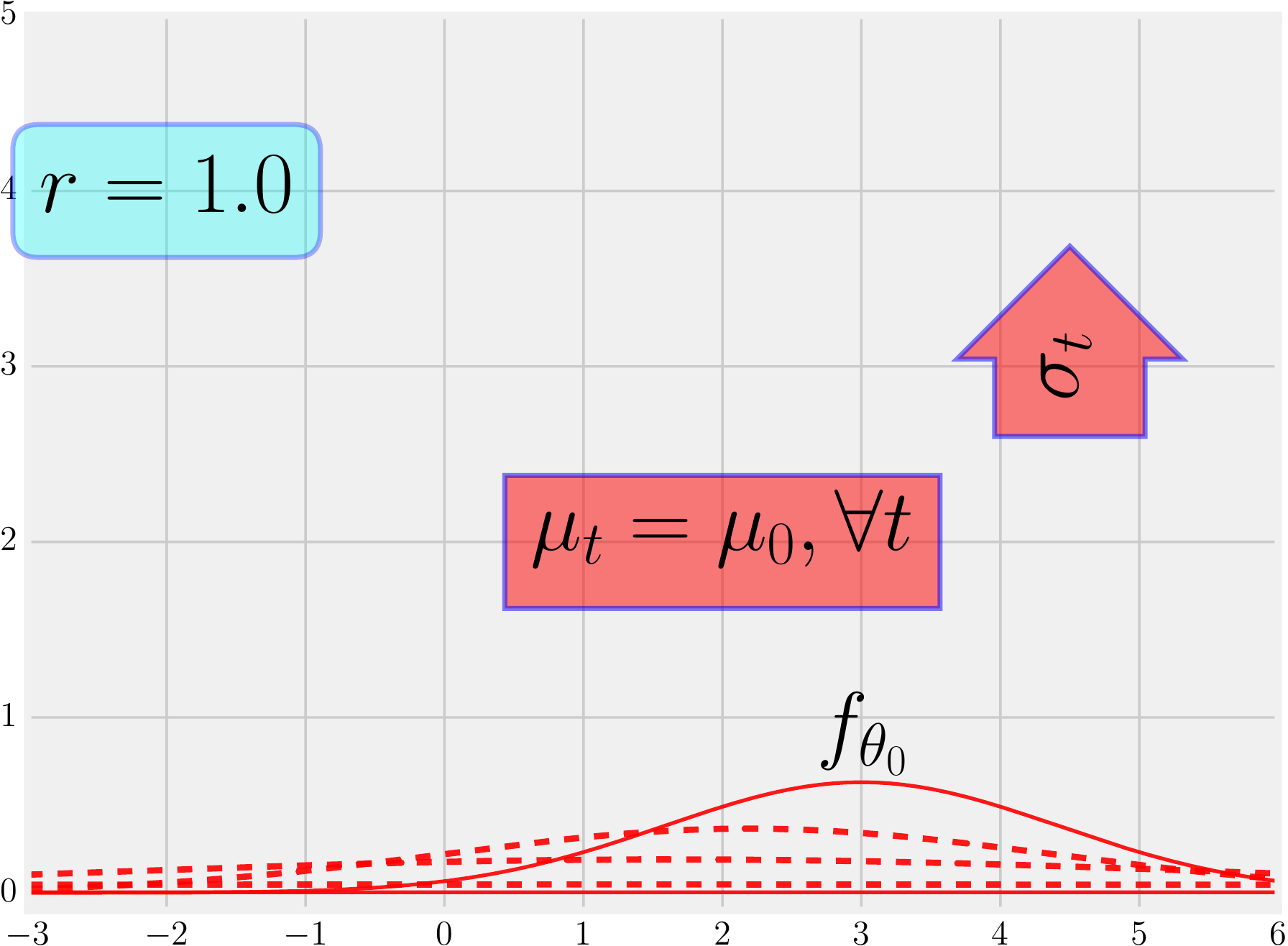}
	\caption{Example which illustrates the dependency of CE2-ND on $r$. Here, with $r=1.0$, CE2-ND explodes.}\label{fig:rdiff2}
\end{figure}

Another pertinent parameter highlighted in Theorem \ref{thm:ce2nddetmain} is the initial distribution parameter $\theta_0$. We assume the initial distribution to be a zero-mean Gaussian distribution with the co-ordinates being mutually independent and the covariance matrix is of the form $q\mathbb{I}_{m \times m}$, where $q > 0$. The theorem emphasizes the existence of a lower bound $q^{*}$ for $q$ to successfully seek the global optimum. This is intuitive since the initial distribution which is being mixed with the current model distribution during each iteration of CE2-ND is necessary to promote sufficient exploration of the solution space $\mathcal{X}$ and this prevents the algorithm from the premature convergence to any of the sub-optimal solutions. We also illustrate it empirically in Fig. \ref{fig:qdiff1}. Here, we again consider the same setting from  Example \ref{exp:exothmce2mn}. Here, with $q = 0.8$,  the algorithm CE2-ND exhibits sub-optimal behaviour. This is in contrast to the optimal behaviour illustrated in Fig. \ref{fig:rdiff1}, where $q = 1.0$. 
\begin{figure}[h]
	\centering
	\includegraphics[scale=0.4]{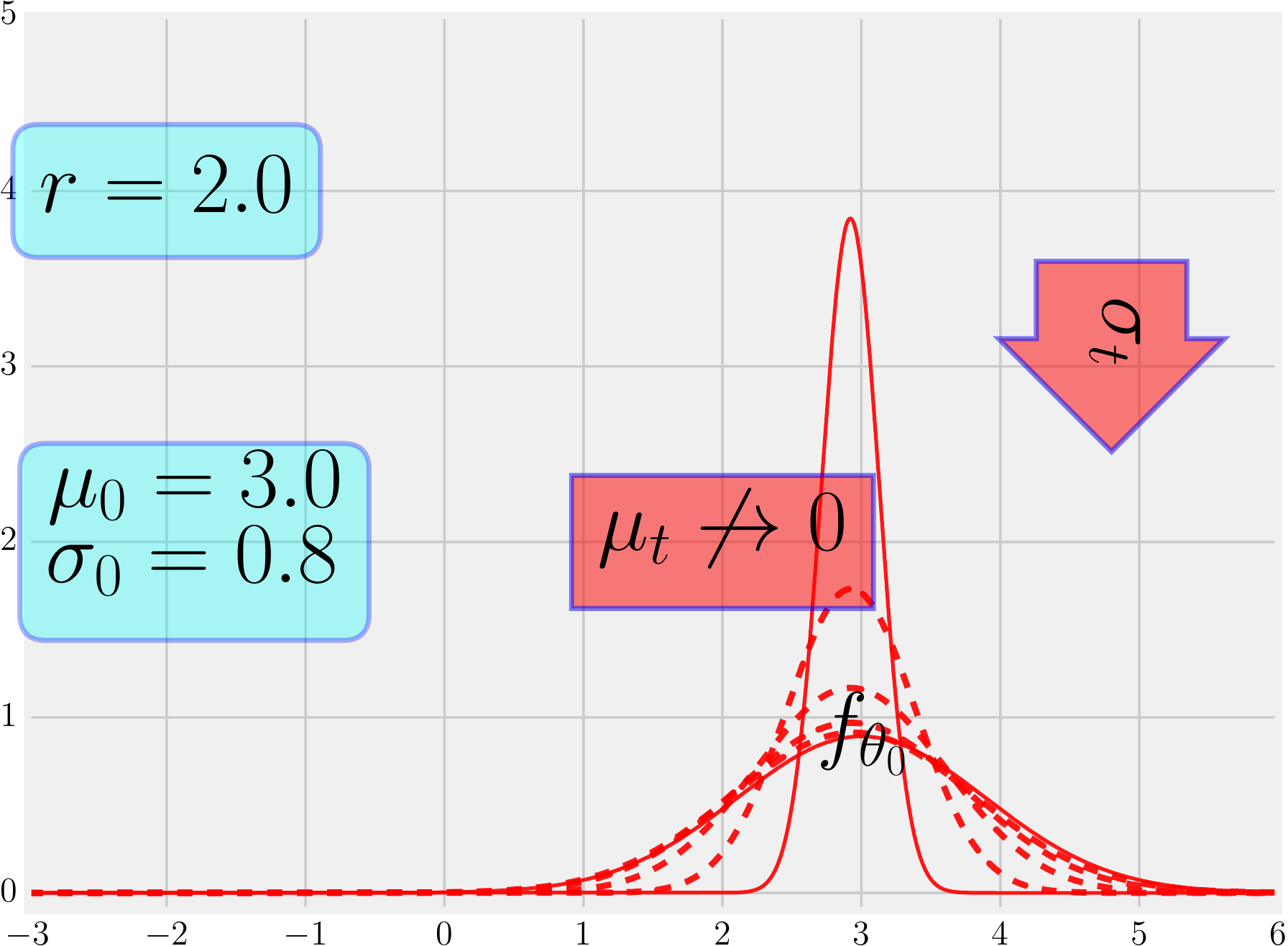}
	\caption{Example which illustrates the dependency of CE2-ND on $q$. Here, with $q=0.8$, CE2-ND converges to a sub-optimal solution.}\label{fig:qdiff1}
\end{figure}

However, the choice of $\rho$ does indeed affect the rate of convergence of the algorithm since it influences the rate of contraction of the search space. Indeed, one can easily observe that during a successful search, the search space contracts probabilistically during iterations and finally converges to the singleton $\{x^{*}\}$. So, for a given non-degenerate PDF $f_{\theta}$,  the threshold which is the $(1-\rho)$-quantile of $\mathcal{H}$ w.r.t $f_{\theta}$ is monotonically decreasing with respect to the parameter $\rho$ and hence for larger values of $\rho$ ($\rho$ close to $1$), the thresholds might rise very slowly which inversely affects the contraction rate of the search space. Also, for small values of $\rho$ ($\rho$  close to $0$) with the PDF $f_{\theta}$ being heavy tailed, it might occur that the variance might be very high in estimating $\gamma_{\rho}(\mathcal{H}, \theta)$ and this will negatively impact the optimal evolution of the model sequence. So an intermediate value of $\rho$ is always recommended for the optimal performance of CE2-ND. This is empirically demonstrated in the experimental section. 

Another important parameter even though non-tunable is the threshold levels $\gamma_{t_{(n)}}$, whose role is critical while updating the model parameter. Recall that the samples whose function values are greater than the current threshold are only considered for updating the model parameter. In CE2-ND, the value of $\gamma_{t_{(n)}}$ is the $(1-\rho)$-quantile of $\mathcal{H}$ w.r.t. the PDF $f_{\theta_{t_{(n)}}}$ (disregarding the mixture distribution for the time being). For brevity, let us drop the sub-sequence notation and use $t$ instead of $t_{(n)}$. In order to better comprehend the dynamics of the algorithm, it is imperative to explore the nature of the evolution of $\gamma_t$. One might intuitively think that the sequence $\{\gamma_t\}$ should be monotonically increasing, since the evolution of the model sequence $\{\theta_t\}$ is primarily guided towards increasing the probability of the high quality solutions. But one will need to rigorously prove this claim. Before we do that, we provide some existing results from the literature which will be of assistance to the analysis. For the standard CE method (update rule defined in Equation (\ref{eq:stdce})), we know from Lemma 4 of \cite{hu2007model} that
\begin{equation}\label{eq:stdcegminc}
\P_{\theta_{t+1}}(\mathcal{H}(\mathsf{X}) \geq \gamma_{\rho}(\mathcal{H}, \theta_t)) \geq \P_{\theta_t}((\mathcal{H}(\mathsf{X}) \geq \gamma_{\rho}(\mathcal{H}, \theta_t)) \geq \rho, \hspace*{5mm} \forall t \geq 0.
\end{equation}
This further implies that $\gamma_{\rho}(\mathcal{H}, \theta_{t+1}) \geq \gamma_{\rho}(\mathcal{H}, \theta_t)$, $\forall t \geq 0$. This establishes the monotonically ascending nature of the threshold sequence $\{\gamma_t\}$ for the standard CE. However in CE2-ND, the update of model parameters involves weighting with $S(\mathcal{H}(x))$. Also for the weighted case, we have the following result from Theorem 2 of \cite{hu2007model}.
\begin{equation}\label{eq:excegminc}
\mathbb{E}_{\theta_{t+1}}\left[S(\mathcal{H}(\mathsf{X}))\mathbb{I}_{\{\mathcal{H}(\mathsf{X}) \geq \gamma_{\rho}(\mathcal{H}, \theta_t)\}}\right] \geq \mathbb{E}_{\theta_{t}}\left[S(\mathcal{H}(\mathsf{X}))\mathbb{I}_{\{\mathcal{H}(\mathsf{X}) \geq \gamma_{\rho}(\mathcal{H}, \theta_t)\}}\right].
\end{equation}
The result shows that the expected behaviour of the subsequent model $\theta_{t+1}$ in the region $\{\mathcal{H}(x) \geq \gamma_{\rho}(\mathcal{H}, \theta_t)\}$ is superior to the expected behaviour of the current model $\theta_t$ in the same region. Even though this result provides quite an insight into the expected behaviour of the model sequence $\{\theta_t\}$, it is not trivial to deduce whether the threshold sequence $\{\gamma_t\}$ should improve over subsequent iterations. It requires slightly  deeper analysis which we provide here.
\begin{proposition}
	Let $\{\theta_t\}$ be the model sequence generated by the update rule (\ref{eq:opt1}). Further assume that both $S$ and $\mathcal{H}$ are Borel measurable. Then
	$\gamma_{\rho}(\mathcal{H}, \theta_{t+1}) \geq \gamma_{\rho}(\mathcal{H}, \theta_{t})$.
\end{proposition}
\begin{proof}
	Rewriting the update rule (\ref{eq:opt1}) in a generalized form as follows:
	\begin{flalign}\label{eq:rwopt1}
	\theta^{L}_{t+1} = \argmax_{\theta \in \Theta}\mathbb{E}_{\theta^{L}_{t}}\left[L(\mathsf{X}))\mathbb{I}_{\{\mathcal{H}(\mathsf{X}) \geq \gamma_{t+1}\}}\log{f_\theta(\mathsf{X})}\right],
	\end{flalign} 
	where $L:\mathcal{X} \rightarrow \bbbr$. Note that in the case of weighted CE, we have $L(x) = (S \circ \mathcal{H})(x)$. Now consider the case when $L$ is a characteristic function, \emph{i.e.}, $L$ is of the form $\mathbb{I}_{A}$ where $A \subseteq \mathbb{R}^{m}$ is a Borel set. Then, by Equation (\ref{eq:excegminc}), we have
	\begin{flalign}
	&\mathbb{E}_{\theta^{A}_{t+1}}\left[\mathbb{I}_{A}\mathbb{I}_{\{\mathcal{H}(\mathsf{X}) \geq \gamma_{\rho}(\mathcal{H}, \theta^{A}_t)\}}\right] \geq \mathbb{E}_{\theta^{A}_{t}}\left[
	\mathbb{I}_{A}\mathbb{I}_{\{\mathcal{H}(\mathsf{X}) \geq \gamma_{\rho}(\mathcal{H}, \theta^{A}_t)\}}\right].\nonumber\\\nonumber\\
	&\Longrightarrow \P_{\theta^{A}_{t+1}}(\{\mathcal{H}(\mathsf{X}) \geq \gamma_{\rho}(\mathcal{H}, \theta^{A}_t)\} \cap A) \geq \P_{\theta^{A}_t}(\{\mathcal{H}(\mathsf{X}) \geq \gamma_{\rho}(\mathcal{H}, \theta^{A}_t)\} \cap A),
	\end{flalign}
	where $\{\theta^{A}_{t}\}$ is the model sequence generated using Equation (\ref{eq:rwopt1}) with $L(x) = \mathbb{I}_{A}(x)$.
	
	Now note that $S \circ \mathcal{H}$ is Borel measurable (since $S$ and $\mathcal{H}$ are Borel measurable). Also $S > 0$. Hence there exists a sequence of simple functions $\{s_n\}$ such that $s_{n}(x) \rightarrow (S \circ \mathcal{H})(x)$, $\forall x \in \mathcal{X}$. Note that the simple function $s_n$ is of the form $s_n(x) = \sum_{i=1}^{d}a_i\mathbb{I}_{A_i}(x)$, where $a_i > 0$, $1 \leq \forall i \leq d$ and each $A_i \subseteq \mathbb{R}^{m}$ is a Borel set and $A_i \cap A_j = \phi$, $1\leq \forall i, j, \leq d$, $i \neq j$.
	
	Now let $\{\theta^{s_n}_{t}\}$ be sequence genenerated using Equation (\ref{eq:rwopt1}) with $L(x) = s_{n}(x)$. Hence by Equation (\ref{eq:excegminc}), we have
	\begin{flalign}\label{eq:sngeq}
	&\mathbb{E}_{\theta^{s_n}_{t+1}}\left[s_n(\mathsf{X})\mathbb{I}_{\{\mathcal{H}(\mathsf{X}) \geq \gamma_{\rho}(\mathcal{H}, \theta^{s_n}_t)\}}\right] \geq \mathbb{E}_{\theta^{s_n}_{t}}\left[s_n(\mathsf{X}))\mathbb{I}_{\{\mathcal{H}(\mathsf{X}) \geq \gamma_{\rho}(\mathcal{H}, \theta^{s_n}_t)\}}\right].\nonumber\\
	&\Longrightarrow \sum_{i=1}^{d}a_i\P_{\theta^{s_n}_{t+1}}(\{\mathcal{H}(\mathsf{X}) \geq \gamma_{\rho}(\mathcal{H}, \theta^{s_n}_t)\} \cap A_i) \geq \nonumber\\&\hspace*{30mm}\sum_{i=1}^{d}a_i\P_{\theta^{s_n}_t}(\{\mathcal{H}(\mathsf{X}) \geq \gamma_{\rho}(\mathcal{H}, \theta^{s_n}_t)\} \cap A_i).\nonumber\\
	&\Longrightarrow \sum_{i=1}^{d}\P_{\theta^{s_n}_{t+1}}(\{\mathcal{H}(\mathsf{X}) \geq \gamma_{\rho}(\mathcal{H}, \theta^{s_n}_t)\} \cap A_i) \geq \nonumber\\&\hspace*{30mm}\sum_{i=1}^{d}\P_{\theta^{s_n}_t}(\{\mathcal{H}(\mathsf{X}) \geq \gamma_{\rho}(\mathcal{H}, \theta^{s_n}_t)\} \cap A_i).\nonumber\\
	&\Longrightarrow \P_{\theta^{s_n}_{t+1}}(\mathcal{H}(\mathsf{X}) \geq \gamma_{\rho}(\mathcal{H}, \theta^{s_n}_t)) \geq \P_{\theta^{s_n}_t}(\mathcal{H}(\mathsf{X}) \geq \gamma_{\rho}(\mathcal{H}, \theta^{s_n}_t)).
	\end{flalign}
	
	Let $\{\theta_t\}$ be the model sequence generated using Equation (\ref{eq:rwopt1}) with $L(x) = (S \circ \mathcal{H})(x)$. Now, since the function in Equation (\ref{eq:rwopt1}) is concave, it is not hard to verify that $\theta^{s_n}_{t} \rightarrow \theta_{t}$ as $n \rightarrow \infty$. Hence, from Equation (\ref{eq:sngeq}) and by the hypothesis that $\mathcal{X}$ is compact and further using the Bounded Convergence Theorem, we get
	\begin{flalign*}
	&\P_{\theta_{t+1}}(\mathcal{H}(\mathsf{X}) \geq \gamma_{\rho}(\mathcal{H}, \theta_t)) \geq \P_{\theta_t}(\mathcal{H}(\mathsf{X}) \geq \gamma_{\rho}(\mathcal{H}, \theta_t)) \geq \rho.\\
	&\Longrightarrow \gamma_{\rho}(\mathcal{H}, \theta_{t+1}) \geq \gamma_{\rho}(\mathcal{H}, \theta_{t}).
	\end{flalign*}
	This completes the proof.
\end{proof}
It is important to note that the above claim which confirms the monotonically increasing nature of the sequence $\{\gamma_t\}$ is established for the case which does not involve the mixture distribution. Now for the mixture distribution case, we have the earlier result (Theorem \ref{thm:ce2nddetmain}), where we confirmed the convergence of the model sequence to the degenerate distribution concentrated on the global optimum $x^{*}$. Combining these two results, we obtain the following corollary.
\begin{corollary}
	Let the assumptions of Theorem \ref{thm:ce2nddetmain} hold. Then $\gamma_{\rho}(\mathcal{H}, \widehat{\theta}_{t_{(n)}}) \rightarrow \mathcal{H}(x^{*})$ as $n \rightarrow \infty$ with probability 1.
\end{corollary}
To illustrate this particular aspect of CE2-ND, we again consider the setting from Example \ref{exp:exothmce2mn}. The results obtained are shown in Fig. \ref{fig:gdiff1}.
\begin{figure}[h]
	\centering
	\includegraphics[scale=0.4]{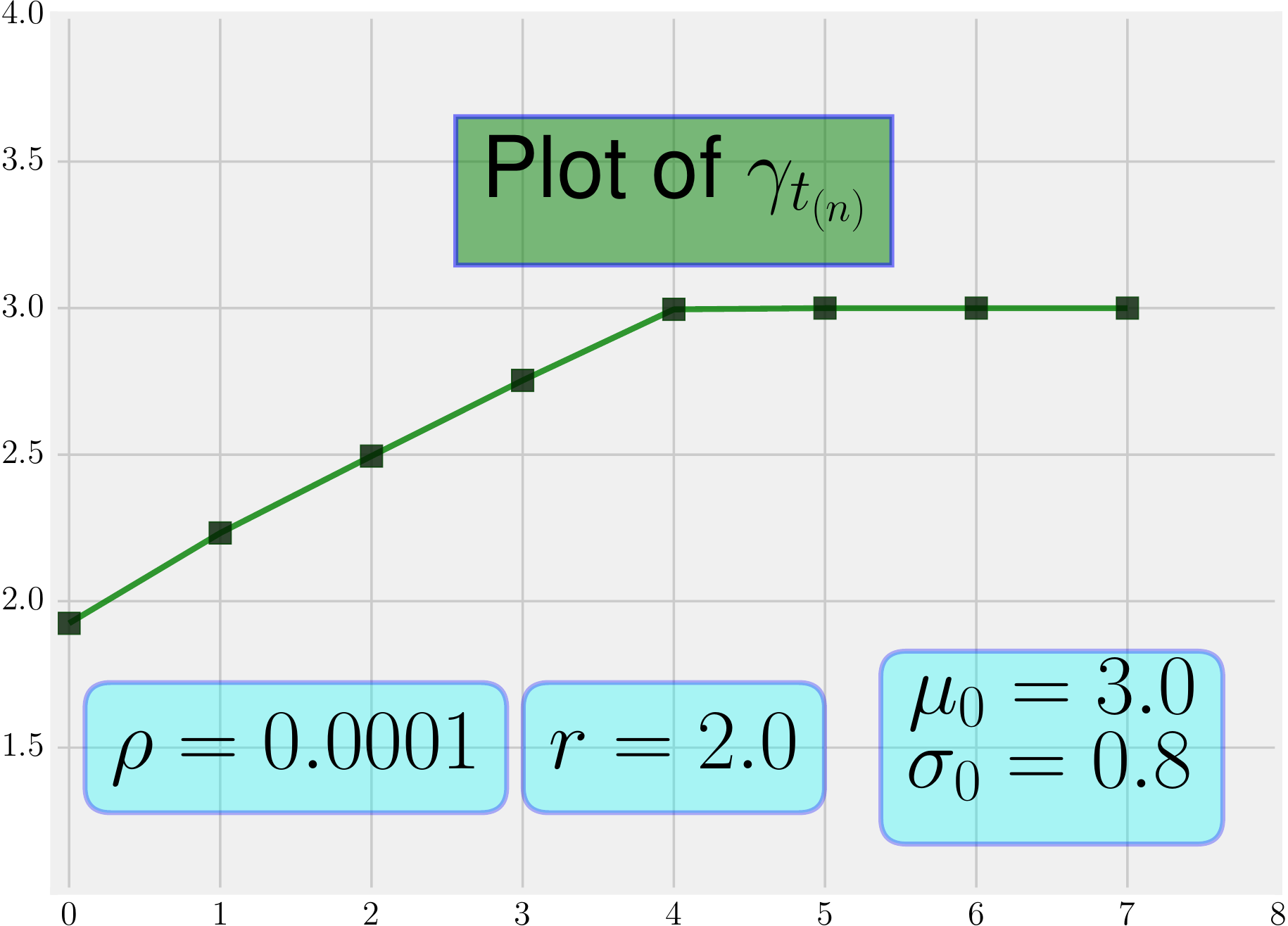}
	\caption{Example which illustrates the evolution of $\gamma_{\rho}(\mathcal{H}, \theta_{t(n)})$.}\label{fig:gdiff1}
\end{figure}
\section{Experimental Illustrations}
We tested CE2-ND on several global optimization benchmark functions from \cite{jamil2013literature}. The benchmark functions that we consider exhibit an uneven and rough landscape with many local optima. To evaluate the performance of the algorithm, we compare it against the naive Monte-Carlo CE (MCCE) from Algorithm \ref{algo:cemc} and the state-of-the-art gradient based Monte-Carlo CE (GMCCE) \cite{hu2012stochastic}, which is a modified version of the Monte-Carlo CE. The model parameter update of GMCCE is given by
\begin{flalign}\label{eq:gmcce}
&\mu_{t+1} = \alpha_{t}\frac{\frac{1}{N_{t}}\sum_{i=1}^{N_{t}}\mathsf{g_{1}}(\mathcal{H}(\mathsf{X}_{i}), \mathsf{X}_{i}, \gamma_{t+1})}{\frac{1}{N_{t}}\sum_{i=1}^{N_{t}}\mathsf{g_{0}}(\mathcal{H}(\mathsf{X}_{i}), \gamma_{t+1})} + (1-\alpha_t)\mu_{t},\\
&\Sigma_{t+1} = \alpha_{t}\frac{\frac{1}{N_{t}}\sum_{i=1}^{N_{t}}\mathsf{g_{2}}(\mathcal{H}(\mathsf{X}_{i}), \mathsf{X}_{i}, \gamma_{t+1}, \mu_{t+1})}{\frac{1}{N_{t}}\sum_{i=1}^{N_{t}}\mathsf{g_{0}}\bm{(}\mathcal{H}(\mathsf{X}_{i}), \gamma_{t+1})} + \nonumber\\ &\hspace*{2cm}(1-\alpha_t)(\Sigma_{t} + (\mu_t - \mu_{t+1})(\mu_t - \mu_{t+1})^{\top}),
\end{flalign}
where $\alpha_t \in (0,1)$ and $\gamma_{t+1}$ is computed using Equation (\ref{eqn:quantmcest}).\\

In each of the plots shown in this section, the solid graph represents the trajectory of $\mathcal{H}(\mu_{t})$, while the dotted horizontal line is the global maximum $\mathcal{H}^{*}$ of the objective function $\mathcal{H}(\cdot)$. The $x$-axis represents the real time in seconds relative to the start of the algorithm. This particular unit is pertinent due to the contrasting nature of the algorithms (a few being incremental and online, while others are batch based). Hence, by considering the $x$-axis to be the relative time in seconds, we obtain a common basis for comparison. All the three algorithms use the same initial distribution $\theta_0$. This helps to compare the algorithms independent of any initial  bias. The results shown are averages over $10$ independent simulations obtained with the same initial distribution $\theta_{0}$. In this section we take $S(x) = \exp(rx), r > 0$.\\

We consider the following benchmark functions for evaluating the performance of our algorithm:\\
\begin{enumerate}
	\item
	\textbf{Griewank function} [$m = 200$][Continuous, Differentiable, Non-Separable, Scalable, Multimodal]
	\begin{equation}\label{eqn:griewank}
	\mathcal{H}_{1}(x) = -1-\frac{1}{4000}\sum_{i=1}^{m}x_{i}^{2}+\prod_{i=1}^{m}\cos{(x_i/\sqrt{i})}. 
	\end{equation}\\
	\item
	\textbf{Levy function} [$m = 50$][Continuous, Differentiable, Multimodal]
	\begin{gather*}
	\mathcal{H}_{2}(x) = -1-\sin^{2}{(\pi y_{1})}-(y_{m}-1)^{2}(1 + \sin^{2}{(2\pi y_{m})})- \\
	\sum_{i=1}^{m}[(y_{i}-1)^{2}(1+10\sin^{2}{(\pi y_i+1)})], \\
	\mathrm{where } \hspace{3mm} y_i = 1 + \frac{x_i-1}{4}. \nonumber
	\end{gather*}\\
	\item
	\textbf{Trigonometric function} [$m = 30$][Continuous, Differentiable, Non-Separable, Scalable, Multimodal]
	\begin{gather*}
	\mathcal{H}_{3}(x) = -1-\sum_{i=1}^{m}[8\sin^{2}{(7(x_i-0.9)^{2})} + 6\sin^{2}{(14(x_i-0.9)^{2})}-(x_i-0.9)^{2}].
	\end{gather*}\\
	\item
	\textbf{Rastrigin function} [$m$ = $30$][Continuous, Differentiable, Scalable, Multimodal]
	\begin{gather*}
	\mathcal{H}_{4}(x) = -\sum_{i=1}^{m}(x_{i}^{2}-10\cos{(2\pi x_{i})})-10m.
	\end{gather*}\\
	\item
	\textbf{Qing function} [$m$ = $30$][Continuous, Differentiable, Separable, Scalable, Multimodal]
	\begin{gather*}
	\mathcal{H}_{5}(x) = -\sum_{i=1}^{m}(x_{i}^{2}-i)^{2}.
	\end{gather*}\\
	\item
	\textbf{Bukin function} [$m=2$][Multimodal, Continuous, Non-Differentiable, Non-Separable, Non-Scalable]
	\begin{gather*}
	\mathcal{H}_{6}(x) = -100\sqrt{x_2-0.01x_1^{2}}-0.01\vert x_1 + 10\vert-20.
	\end{gather*}\\
	\item
	\textbf{Salomon function} [$m=20$][Multimodal, Continuous, Differentiable,  Non-Separable, Scalable]
	\begin{gather*}
		\mathcal{H}_{7}(x) = 10\left(-1+\cos{\left(2\pi\sqrt{\sum_{i=1}^{m}x_i^{2}}\right)}-0.1\sqrt{\sum_{i=1}^{m}x_i^{2}}\right).
	\end{gather*}\\
	\item
	\textbf{Rosenbrock function} [$m=10$][Unimodal, Continuous, Differentiable, Non-Separable, Scalable]
	\begin{gather*}
	\mathcal{H}_{8}(x) = 
	-0.0001\left(\sum_{i=1}^{m}100(x_{2i+1}-x_{2i}^2)^2 + (1-x_{2i})^2\right).
	\end{gather*}\\
	\item
	\textbf{Plateau function} [$m=100$][Multimodal, Continuous, Non-Differentiable]
	\begin{gather*}
	\mathcal{H}_{9}(x) = 
	-0.1\left(30+\sum_{i=1}^{m} \lfloor\vert x_i \vert \rfloor\right).
	\end{gather*}\\
	\item
	\textbf{Pathological function} [$m=50$][Multimodal, Continuous, Non-Differentiable, Non-Separable, Non-Scalable]
	\begin{gather*}
	\mathcal{H}_{10}(x) = -0.1
\sum_{i=1}^{m-1}\left(\frac{\sin^{2}{\sqrt{100x_{i}^{2}+x_{i+1}^2}}-0.5}{0.001\left(x_{i}-x_{i+1}\right)^{4}+1}+0.5\right).
	\end{gather*}
\end{enumerate}

\vspace*{10mm}
The results of the numerical experiments are shown in Fig. \ref{fig:mnres}. The various parameter values used in the experiments are shown in Table \ref{tab:parvalues} and Table \ref{tab:initvalues}. To demonstrate the advantages of our algorithm with regards to memory utilization, we plot the real time memory usage of our algorithm and GMCCE. The comparison is shown in Fig. \ref{fig:memcompare}. To understand the behaviour of our algorithm with respect to the quantile parameter $\rho$, we plot the performance of the algorithm for various values of $\rho$. The results are shown in Fig. \ref{fig:rhocompare}. 

\begin{table}[h]
	\vspace*{4mm}
	\caption{The parameter values used in the experiments.}\label{tab:parvalues}
	\centering
	\begin{adjustbox}{max width=\textwidth}
		\begin{tabular}{| c | c | c | c | c | c | c | c | c | c | c |}
			\specialrule{.2em}{.02em}{.02em} 
			\multicolumn{7}{|c|}{\textbf{CE2-ND}} & \multicolumn{4}{|c|}{\textbf{GMCCE}}\\
			\specialrule{.1em}{.02em}{.02em} 
						\\[-1em]
			$\mathcal{H}$ & $r$ & $\beta_{t}$ & $\lambda$ & $c_t$ & $\epsilon_1$ & $\rho$ & $r$ & $\alpha_{t}$ & $\rho$ & $N_t$ \\ \hline
						\\[-1em]
			$\mathcal{H}_{1}$ & $1.0$ & $t^{-0.52}$ & $t_{(n)}^{-3.0}$ & $0.06$ & $0.9$ & $0.001$ & $0.1$ & $0.1$ & $0.001$ & $N_{t+1} = 1.03N_{t}$, $N_0=700$    \\ \hline
						\\[-1em]
			$\mathcal{H}_{2}$ & $0.001$ & $0.1$ & $t_{(n)}^{-3.0}$ & $0.06$ & $0.9$ & $0.1$ & $0.001$ & $0.1$ & $0.1$ &  $N_{t+1} = 1.001N_{t}$, $N_0=700$  \\ \hline
						\\[-1em]
			$\mathcal{H}_{3}$ & $0.001$ & $0.03$ & $t_{(n)}^{-3.0}$ & $0.06$ & $0.9$ & $0.001$ & $0.001$ & $0.001$ & $0.1$ & $N_{t+1} = 1.001N_{t}$, $N_0=700$   \\ \hline
			\\[-1em]
			$\mathcal{H}_{4}$ & $0.01$ & $0.2$ & $t_{(n)}^{-3.0}$ & $0.06$ & $0.9$ & $0.1$ & $0.001$ & $0.2$ & $0.01$ &  $N_{t+1} = 1.001N_{t}$, $N_0=800$   \\ \hline
			\\[-1em]
			$\mathcal{H}_{5}$ & $0.00001$ & $0.05$ & $t_{(n)}^{-3.0}$ & $0.06$ & $0.9$ & $0.01$ & $0.001$ & $0.2$ & $0.01$ & $N_{t+1} = 1.001N_{t}$, $N_0=1000$   \\ \hline
			\\[-1em]
			$\mathcal{H}_{6}$ & $0.1$ & $t_{(n)}^{-0.52}$ & $t_{(n)}^{-3.0}$ & $0.06$ & $0.9$ & $0.01$ & $0.1$ & $0.1$ & $0.01$ & $N_{t+1} = 1.001N_{t}$, $N_0=2000$   \\ \hline
			\\[-1em]
			$\mathcal{H}_{7}$ & $0.5$ & $0.4$ & $t_{(n)}^{-3.0}$ & $0.08$ & $0.9$ & $0.1$ & $0.5$ & $0.5$ & $0.1$ & $N_{t+1} = 1.005N_{t}$, $N_0=2000$   \\ \hline
			\\[-1em]
			$\mathcal{H}_{8}$ & $0.001$ & $0.1$ & $t_{(n)}^{-4.0}$ & $0.06$ & $0.9$ & $0.01$ & $0.001$ & $0.4$ & $0.01$ & $N_{t+1} = 1.001N_{t}$, $N_0=1000$   \\ \hline
			\\[-1em]
			$\mathcal{H}_{9}$ & $0.05$ & $0.22$ & $0.01$ & $0.05$ & $0.9$ & $0.02$ & $0.05$ & $0.2$ & $0.02$ & $N_{t+1} = 1.001N_{t}$, $N_0=1500$   \\ \hline
			\\[-1em]
			$\mathcal{H}_{10}$ & $0.04$ & $0.2$ & $0.2$ & $0.05$ & $0.9$ & $0.1$ & $0.04$ & $0.2$ & $0.1$ & $N_{t+1} = 1.001N_{t}$, $N_0=1200$   \\ \hline									
			\specialrule{.1em}{.02em}{.02em} 
		\end{tabular}
	\end{adjustbox}
\end{table}

\normalsize
\begin{table}[h]
	\vspace*{2mm}
	\caption{The initial distribution $\theta_0$ used in the various cases and the global maximum $\mathcal{H}^{*}$ of the respective functions.}\label{tab:initvalues}
	\begin{center}
		\begin{tabular}{| c | c | c |}
			\specialrule{.2em}{.02em}{.02em} 
			$\mathcal{H}(\cdot)$ & $\theta_0$ & $\mathcal{H}^{*}$ \\ \hline
			\\[-0.9em]			
			$\mathcal{H}_1$ & $(50.0, 50.0, \dots, 50.0)^{\top}$, $100\I_{200 \times 200}$ & $0$\\\hline
			\\[-0.9em]
			$\mathcal{H}_2$ & $(30.0, 30.0, \dots, 30.0)^{\top}$, $250\I_{50 \times 50}$ & $-1$\\\hline
			\\[-0.9em]
			$\mathcal{H}_3$ & $(10.0, 10.0, \dots, 10.0)^{\top}$, $100\I_{30 \times 30}$ & $-1$\\\hline
			\\[-0.9em]
			$\mathcal{H}_4$ & $(25.0, 25.0, \dots, 25.0)^{\top}$, $100\I_{30 \times 30}$ & $0$\\\hline
			\\[-0.9em]
			$\mathcal{H}_5$ & $(20.0, 20.0, \dots, 20.0)^{\top}$, $200\I_{30 \times 30}$ & $0$\\\hline
			\\[-0.9em]		
			$\mathcal{H}_6$ & $(30.0, 30.0)^{\top}$, $250\I_{2 \times 2}$ & $0$\\\hline
			\\[-0.9em]
			$\mathcal{H}_7$ & $(10.0, 10.0, \dots, 10.0)^{\top}$, $10\I_{20 \times 20}$ & $0$\\\hline			
			\\[-0.9em]			
			$\mathcal{H}_8$ & $(10.0, 10.0, \dots, 10.0)^{\top}$, $10\I_{10 \times 10}$ & $0$\\\hline
			\\[-0.9em]
			$\mathcal{H}_9$ & $(20.0, 20.0, \dots, 20.0)^{\top}$, $400\I_{100 \times 100}$ & $-3$\\\hline			
			\\[-0.9em]
			$\mathcal{H}_{10}$ & $(20.0, 20.0, \dots, 20.0)^{\top}$, $100\I_{50 \times 50}$ & $0$\\\hline						
		\end{tabular}
	\end{center}
\end{table}
\begin{figure}[!h]
	\begin{subfigure}[h]{0.5\textwidth}	 
		\includegraphics[height=52mm, width=60mm]{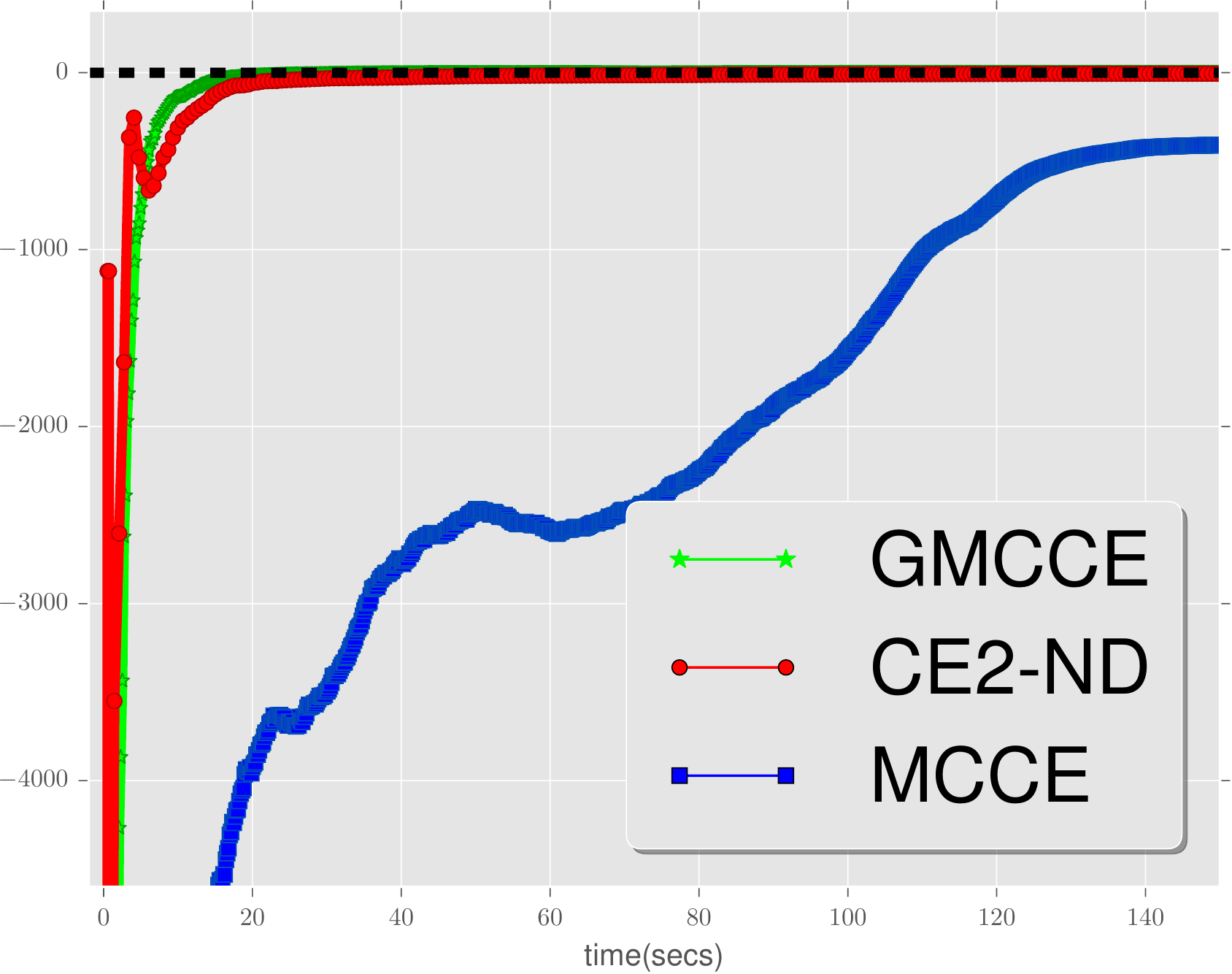}
		\subcaption{Levy function}
	\end{subfigure}%
	\begin{subfigure}[h]{0.5\textwidth}	 
		\includegraphics[width=60mm, height=52mm]{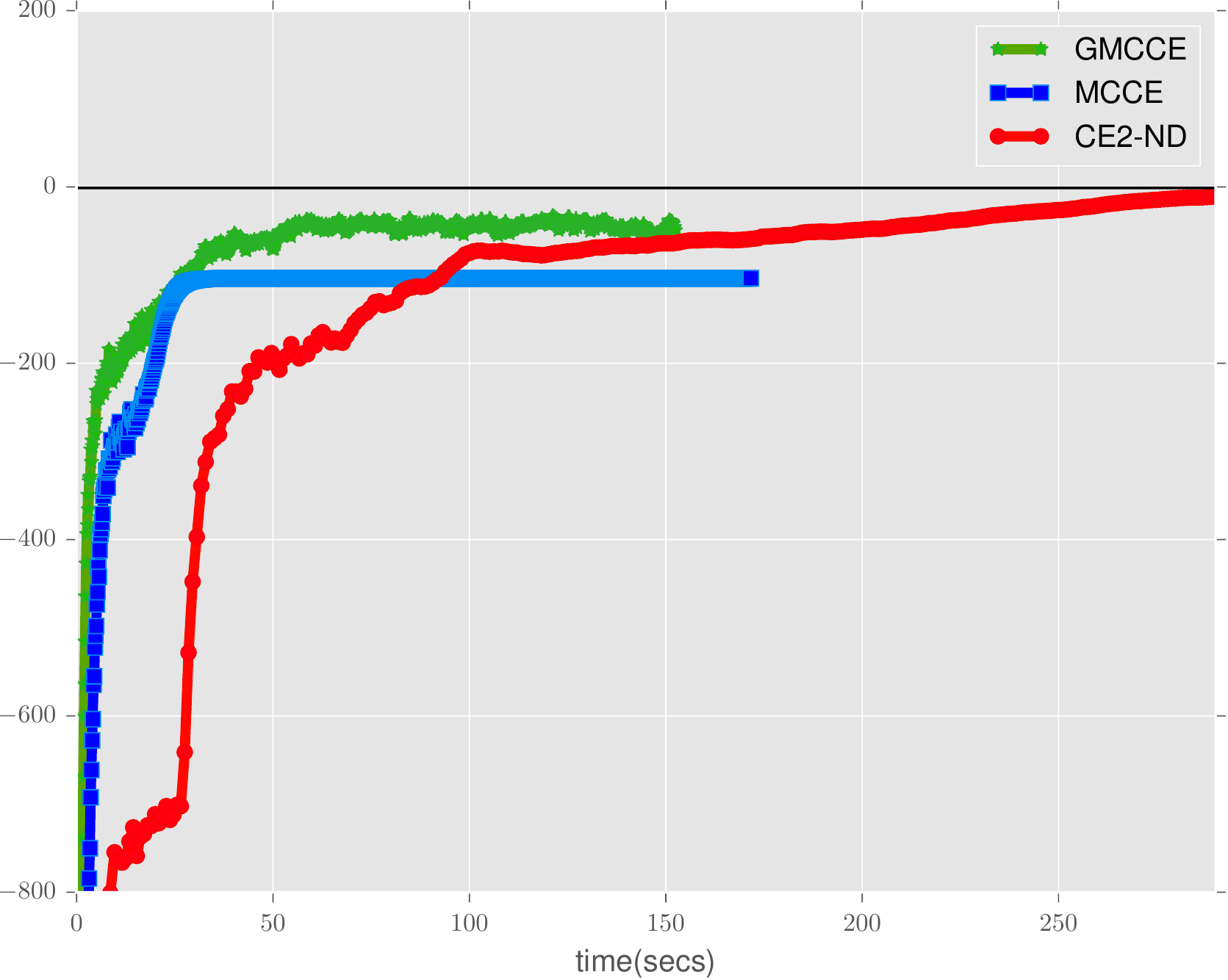}
		\subcaption{Trigonometric function}
	\end{subfigure}\vspace*{10mm}\\
	\begin{subfigure}[h]{0.5\textwidth}
		\includegraphics[width=60mm, height=52mm]{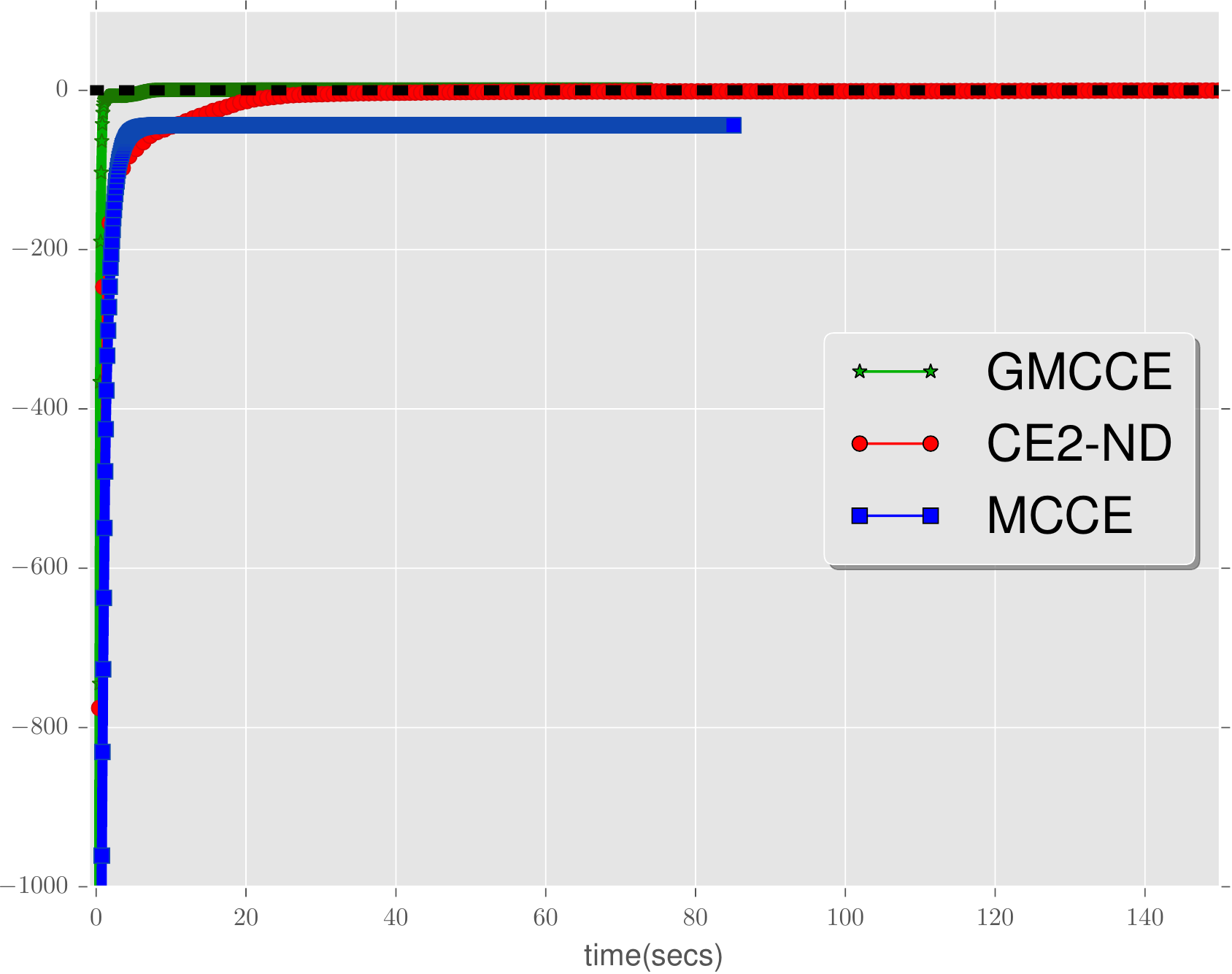}
		\subcaption{Qing function}
	\end{subfigure}
	\begin{subfigure}[h]{0.5\textwidth}
		\includegraphics[height=52mm, width=60mm]{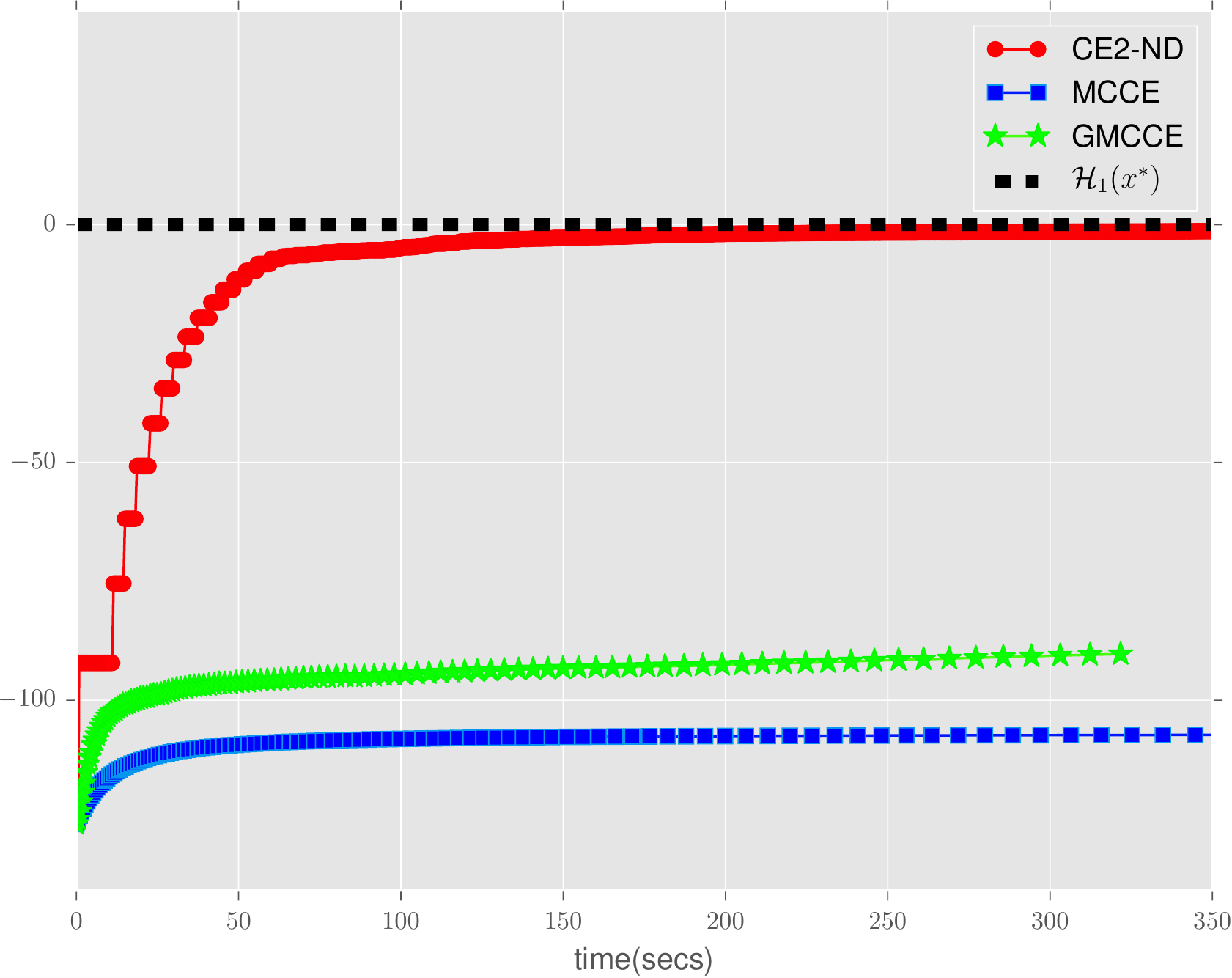}
		\subcaption{Griewank function}
	\end{subfigure}\vspace*{10mm}\\
	\begin{subfigure}[h]{0.5\textwidth}
		\includegraphics[width=60mm, height=52mm]{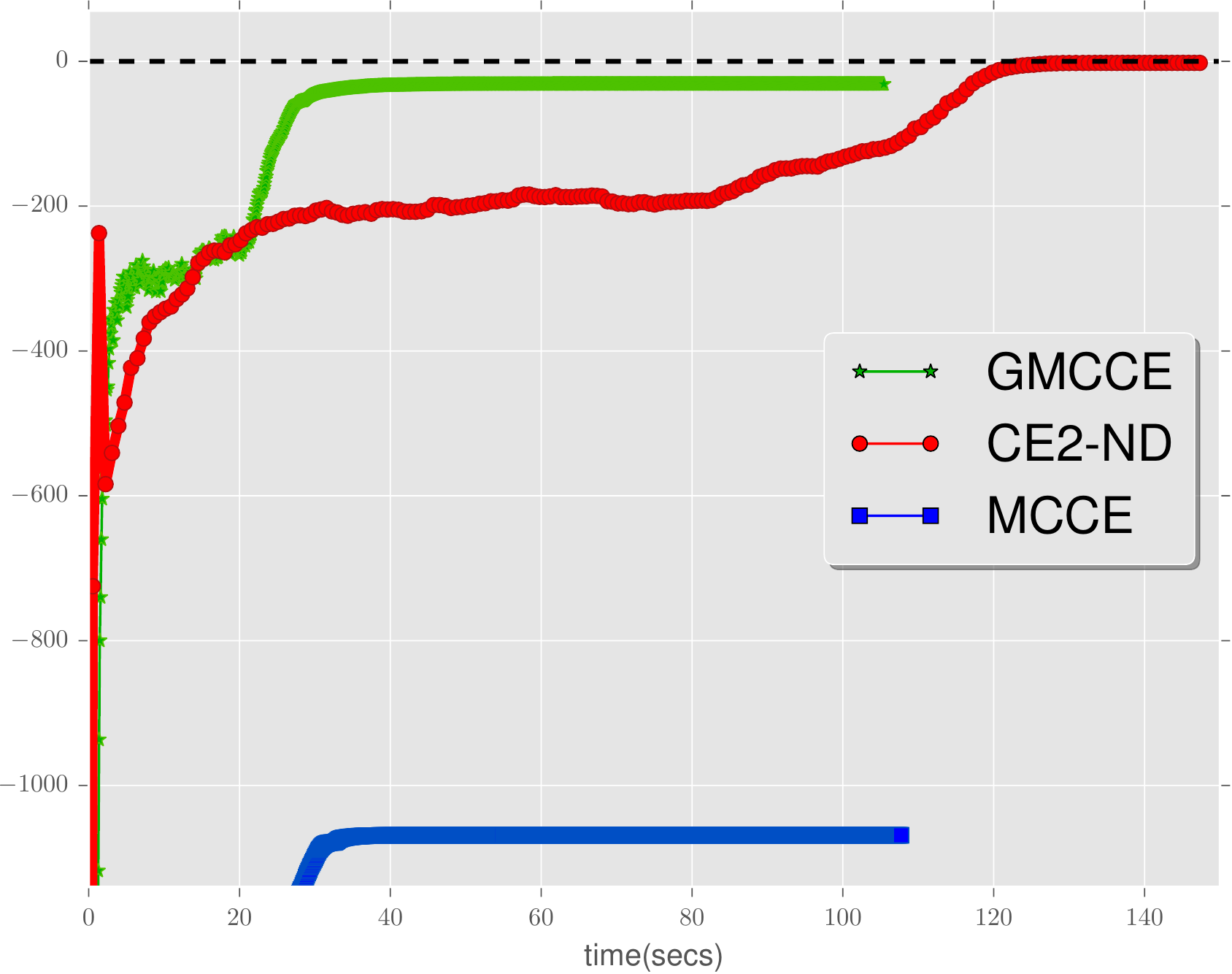}
		\subcaption{Rastrigin function}
	\end{subfigure}
	\begin{subfigure}[h]{0.5\textwidth}
		\includegraphics[width=60mm, height=52mm]{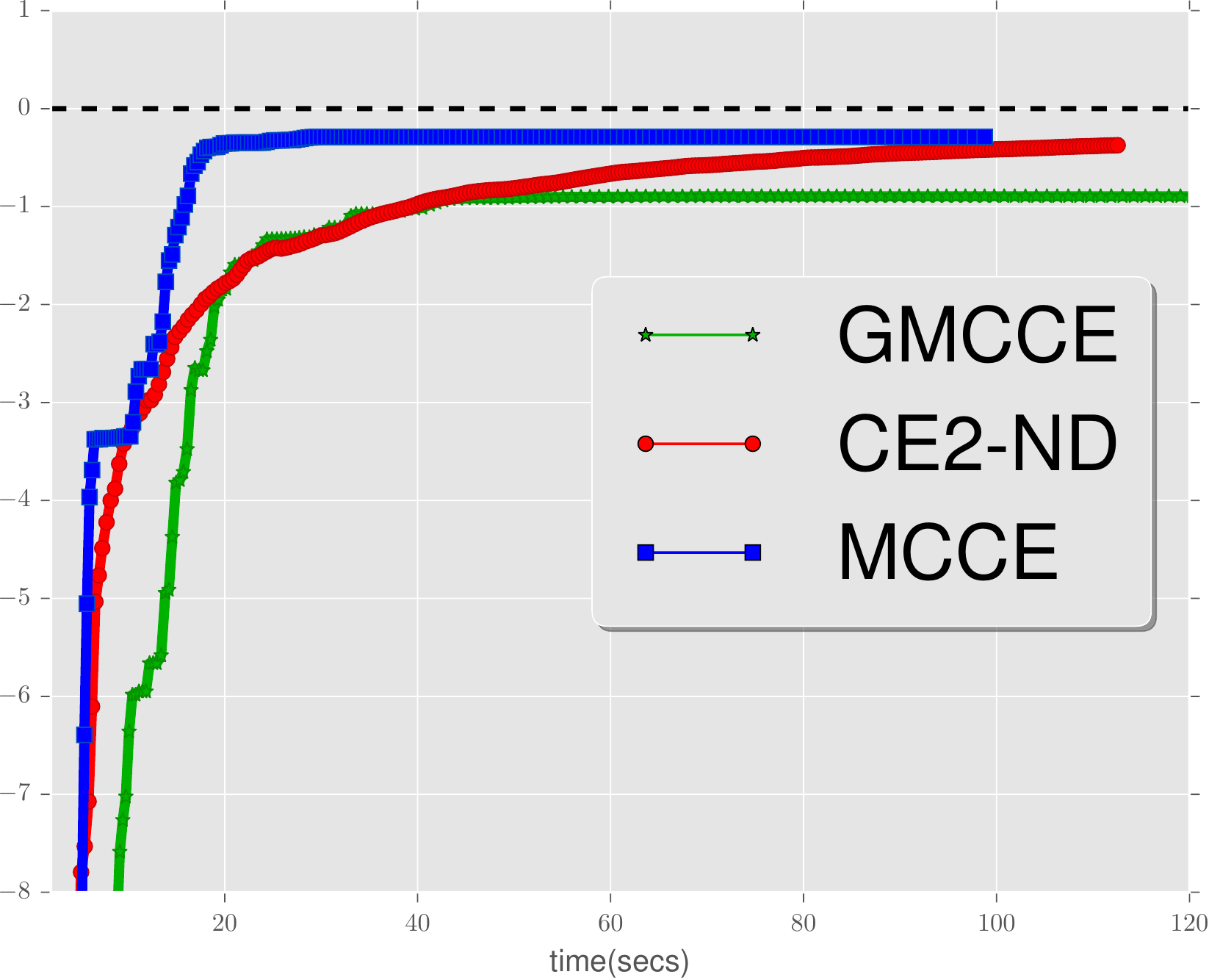}
		\subcaption{Bukin function}
	\end{subfigure}
	\caption{The performance comparison of CE2-ND against GMCCE and MCCE. Here $y$-axis is $\mathcal{H}(\mu_t)$ and $x$-axis is the time in secs relative to the start of the algorithm.}\label{fig:mnres}
\end{figure}

\begin{figure}[!h]
	\begin{subfigure}[h]{0.5\textwidth}
		\includegraphics[height=52mm, width=60mm]{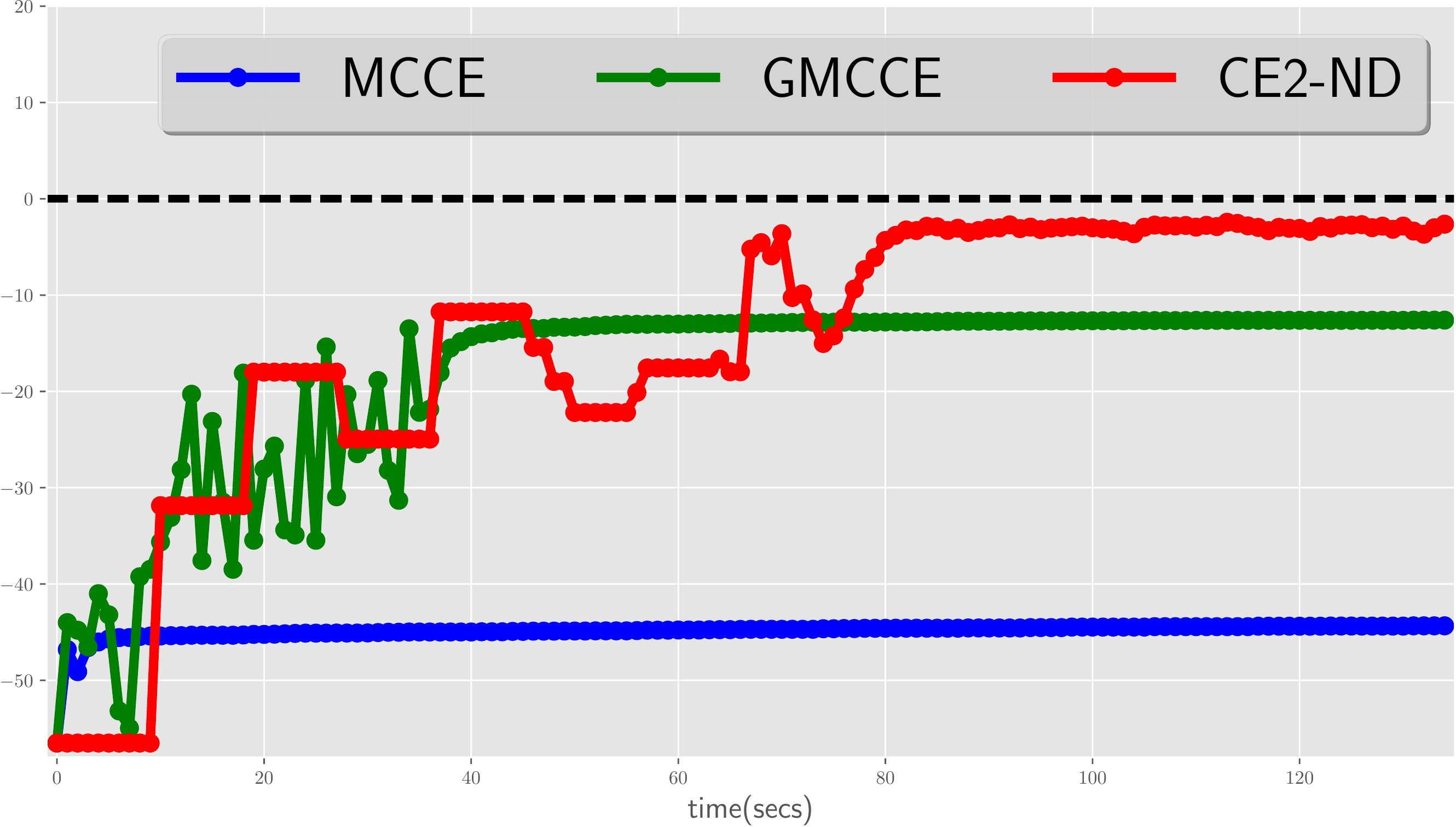}
		\subcaption{Salomon function}
	\end{subfigure}
	\begin{subfigure}[h]{0.5\textwidth}
		\includegraphics[height=52mm, width=60mm]{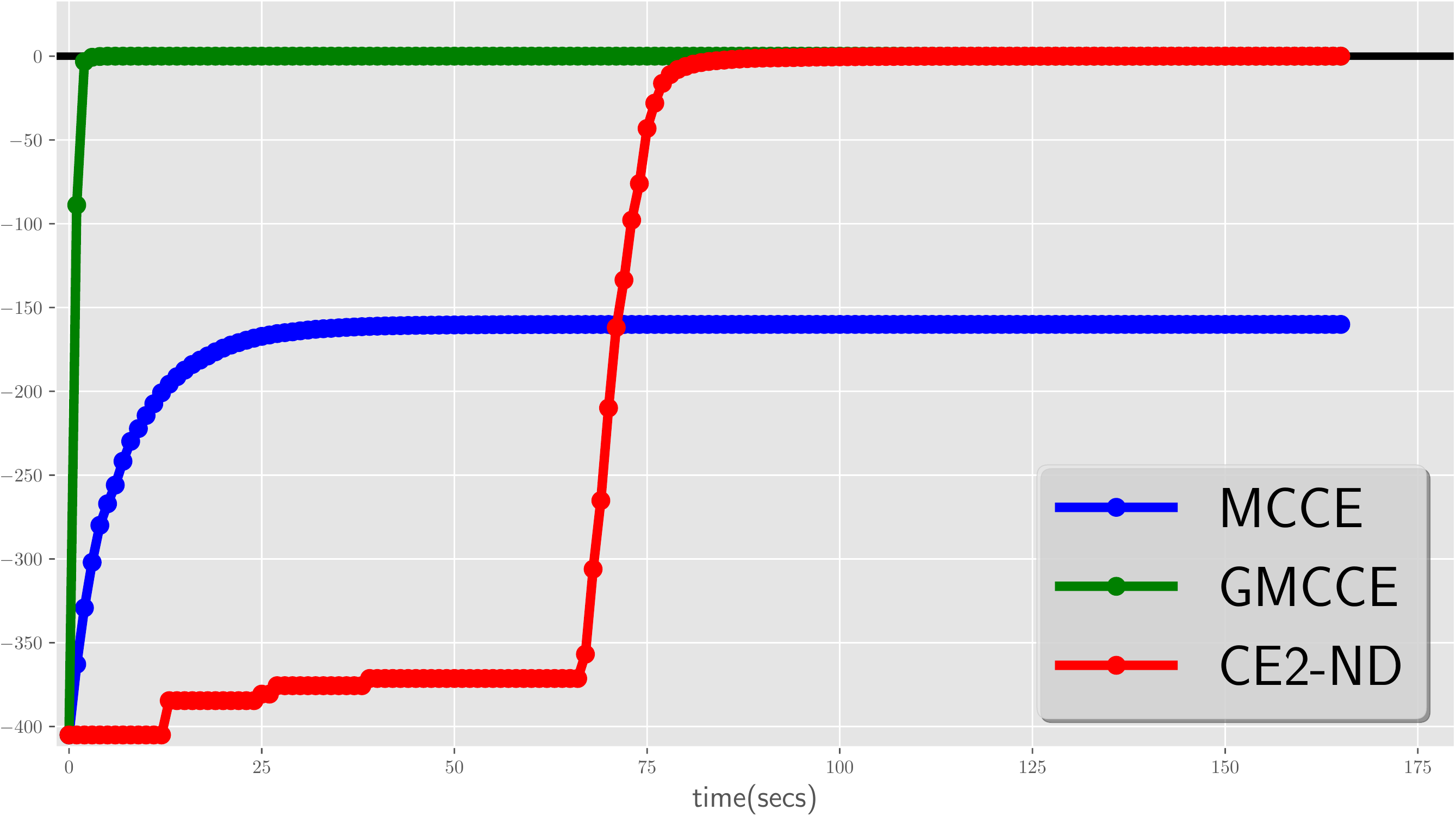}
		\subcaption{Rosenbrock function}
	\end{subfigure}\vspace*{10mm}\\
	\begin{subfigure}[h]{0.5\textwidth}
		\includegraphics[height=52mm, width=60mm]{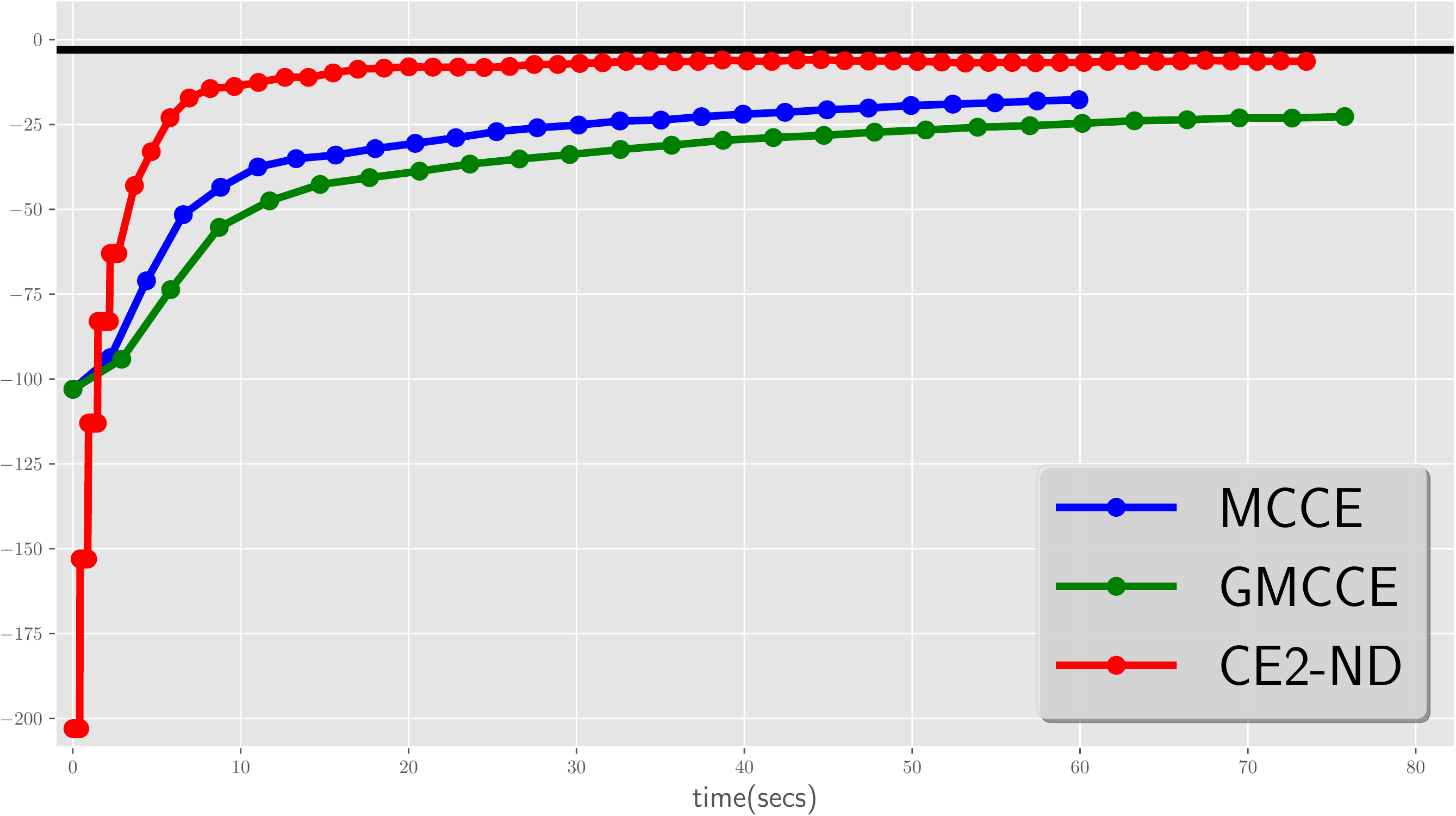}
		\subcaption{Plateau function}
	\end{subfigure}
	\begin{subfigure}[h]{0.5\textwidth}
		\includegraphics[height=52mm, width=60mm]{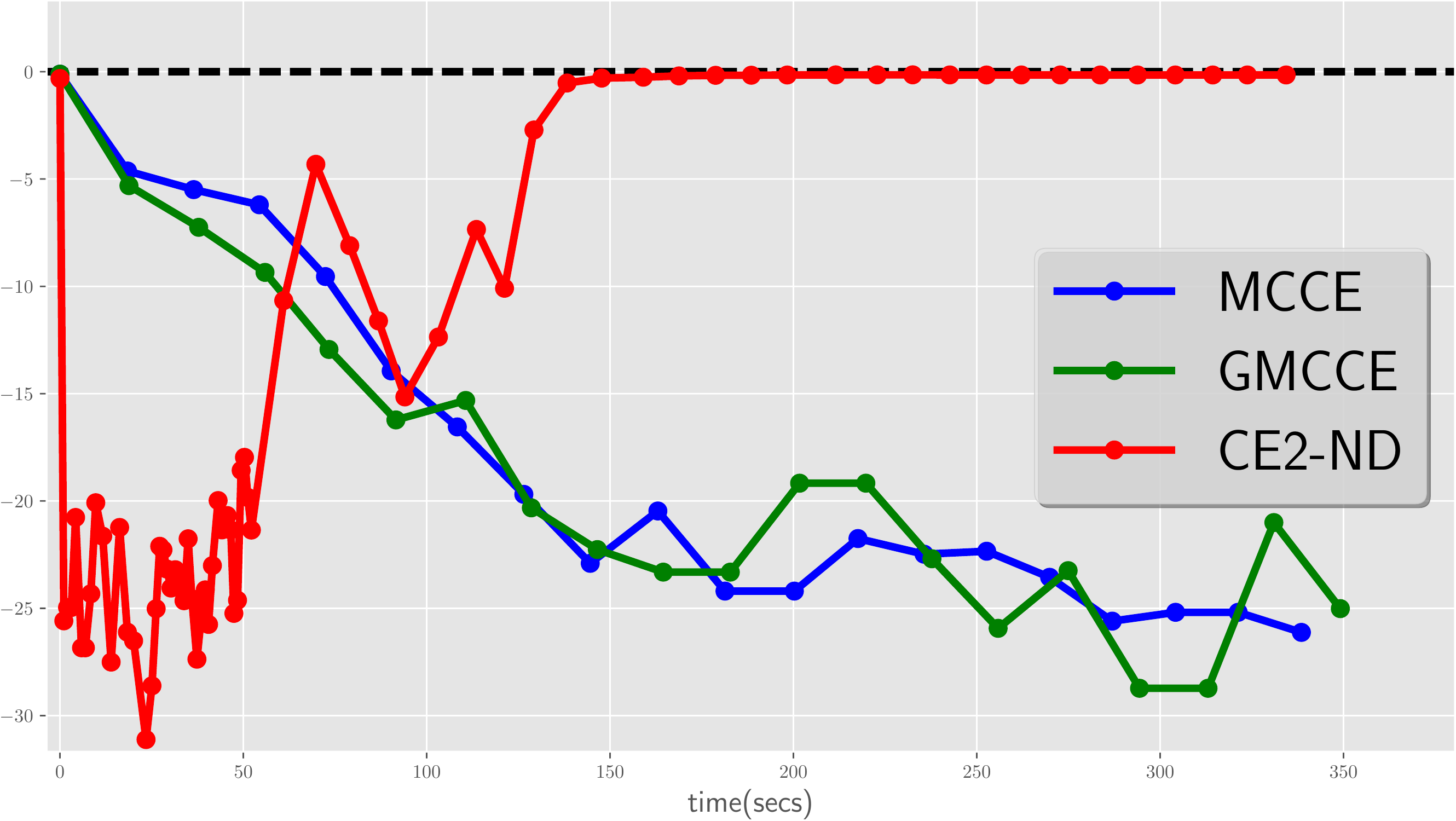}
		\subcaption{Pathological function}\label{fig:pathologicalresult}
	\end{subfigure}
	\caption{More comparisons}
\end{figure}%

\clearpage
\noindent
From the experiments, we make the following observations:\vspace*{0mm}\\
\begin{enumerate}
	\item
	The algorithm CE2-ND shows good performance compared to GMCCE and MCCE in all the test cases that we consider. The algorithm CE2-ND also exhibits good global optimum convergence behaviour when applied to all the above benchmark functions. The benchmark functions we consider for the empirical evaluation of the algorithm possess diverse and rigorous landscape. For example, in the non-differentiable Plateau function $\mathcal{H}_{9}$ (3D plot provided in Fig. \ref{fig:plateau3dplot}), the landscape involves numerous plateaus and ridges which makes the effective navigation of any gradient-based methods quite impossible. However, CE2-ND is able to tread both uphill and downhill across this inaccessible landscape to generate the global optimum with good accuracy at a reasonable rate.
	\begin{figure}[hp]
			\centering
			\includegraphics[height=60mm, width=65mm]{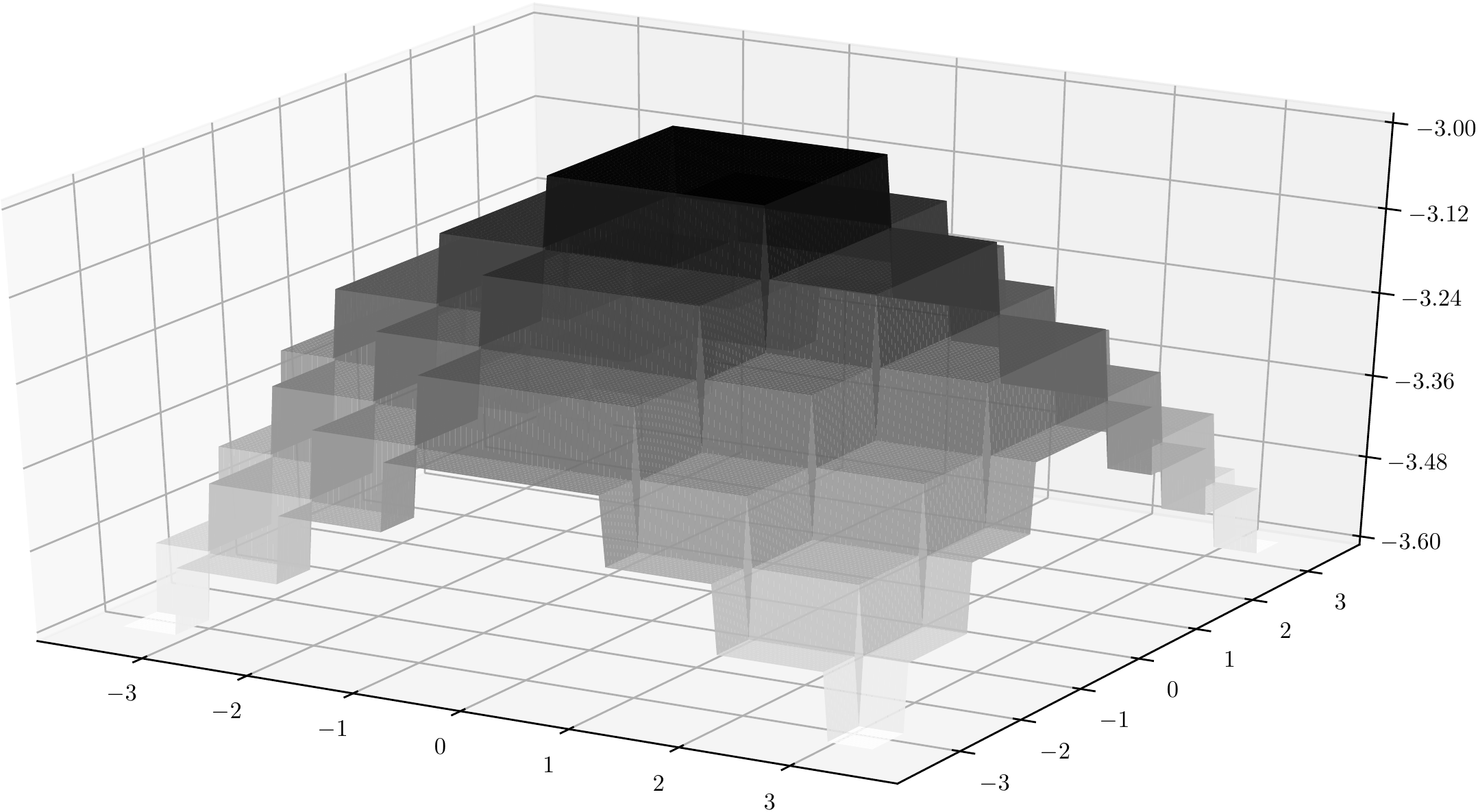}
			\caption{3D plot of the Plateau function}\label{fig:plateau3dplot}
	\end{figure}
	\item
	The sample size requirements of the algorithm CE2-ND witnessed during the empirical evaluation are relatively very less compared to GMCCE and MCCE. This is primarily attributed to the adaptive nature of the underlying stochastic approximation framework, where each sample drawn by the algorithm effectively and efficiently recalibrates the model parameters towards the singular distribution. The sample size requirements of the respective algorithms experienced during the experimental evaluation of the test cases $\mathcal{H}_{8}$, $\mathcal{H}_{9}$, and $\mathcal{H}_{10}$ are provided in Fig. \ref{fig:samplesizecompare}.
	\begin{figure}
		\centering
		\includegraphics[scale=0.32]{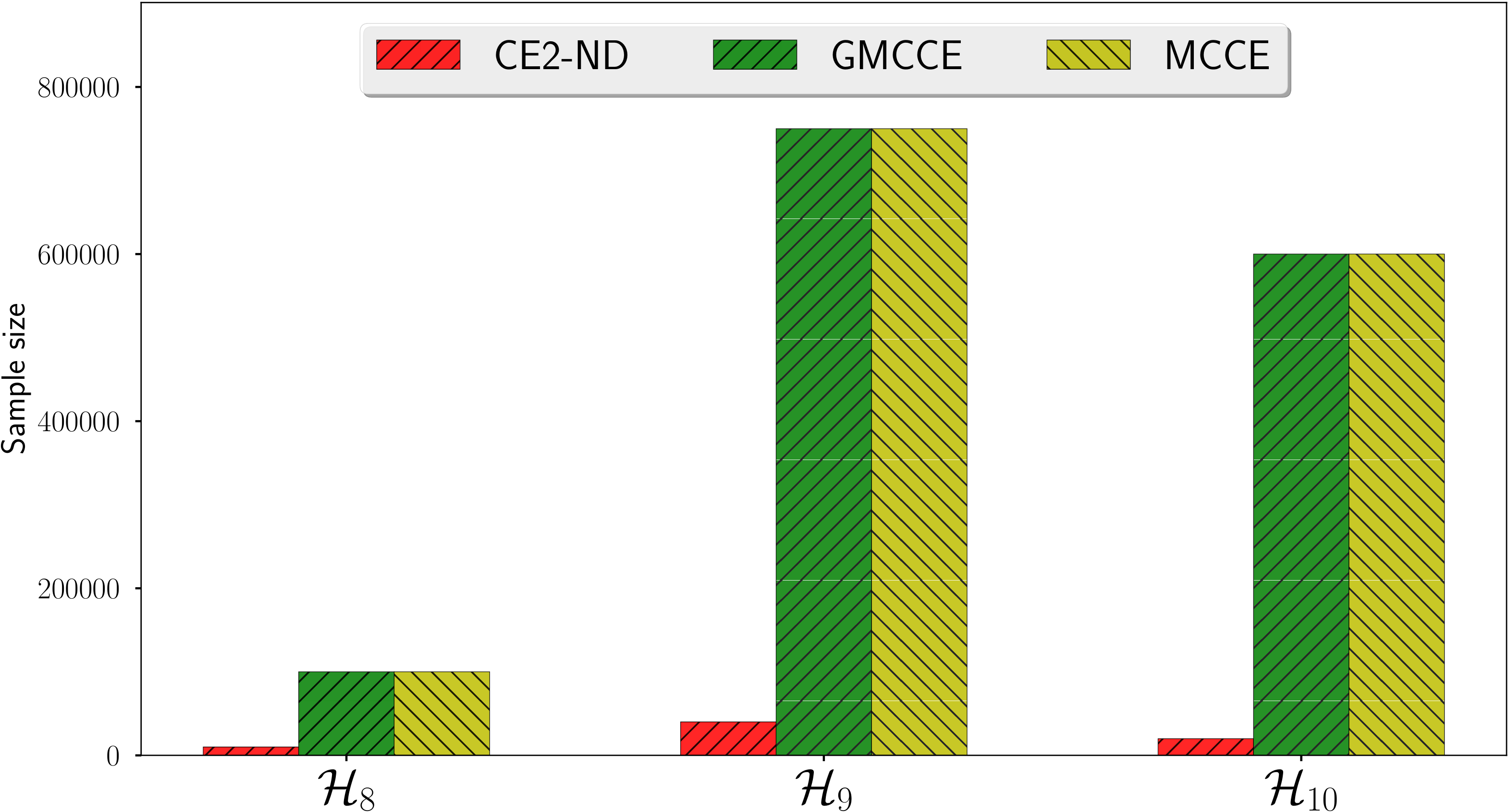}
		\caption{Comparison of the cumulative samples utilized by various algorithms.}\label{fig:samplesizecompare}
	\end{figure}
	\item
	The algorithm exhibits robustness with respect to the initial distribution $\theta_0$ in most of the test cases that we consider. Recall that in  CE2-ND, we employ a mixture PDF $\widehat{f}_{\theta_t}$  to draw the sample at time $t$. An initial distribution which weighs the solution space reasonably well, seems to be sufficient for all the test cases, except for $\mathcal{H}_{10}$ (Pathological function). Note that in the case of $\mathcal{H}_{10}$, both GMCCE and MCCE diverge. It is the unique landscape of the Pathological function which contains very narrow trenches and ridges with smooth regions of moderate values in between. See Fig. \ref{fig:pathological3d}. The global optimum is at the origin which is also contained in a very narrow crest. Hence for a given probability distribution over the solution space, the probabilities of the crests and trenches are very minimal and hence the samples drawn at each iteration of both GMCCE and MCCE are more likely to belong to the moderate range and hence the divergence. In CE2-ND, we use a mixture weight $\lambda = 0.2$ for this particular test case which is relatively higher compared to the rest of the test cases. The mixture PDF acts as a bridle to prevent the drift towards  horizon. We also consider an initial distribution $((0,0,\dots, 0)^{\top}_{50 \times 1}, \I_{50 \times 50})^{\top}$ (which is same for both MCCE and GMCCE). One can indeed argue that this choice of the initial distribution provides unwarranted bias towards the region containing the global optimum. But we believe that this information (in the form of the initial distribution) is quite naive, considering the fact that in the  unit hypercube around the origin (where the initial distribution is heavily concentrated), the terrain of the Pathological function is quite treacherous. See Fig. \ref{fig:pathological3dzoom}. The challenge to seek the global optimum is still hard and any effective global optimization algorithm has to both ascend as well as descend to find the global optimum. Any local optimization algorithm can indeed utilize this initial distribution to randomly pick the initial point, however, due to the uneven landscape in that region, the global optimum convergence is not guaranteed. In the case of CE2-ND, we observe an initial transient phase where the algorithm seems to wander randomly through the solution space (See Fig. \ref{fig:pathologicalresult}). Nonetheless, due to the mixture approach, there is a positive weight on the initial  distribution and hence there always exists a positive probability to draw a sample close to the origin. This will stabilize the exploration and guide the algorithm towards the origin.
	\begin{figure}[!h]
	\begin{subfigure}[h]{0.5\textwidth}
	\includegraphics[height=50mm, width=60mm]{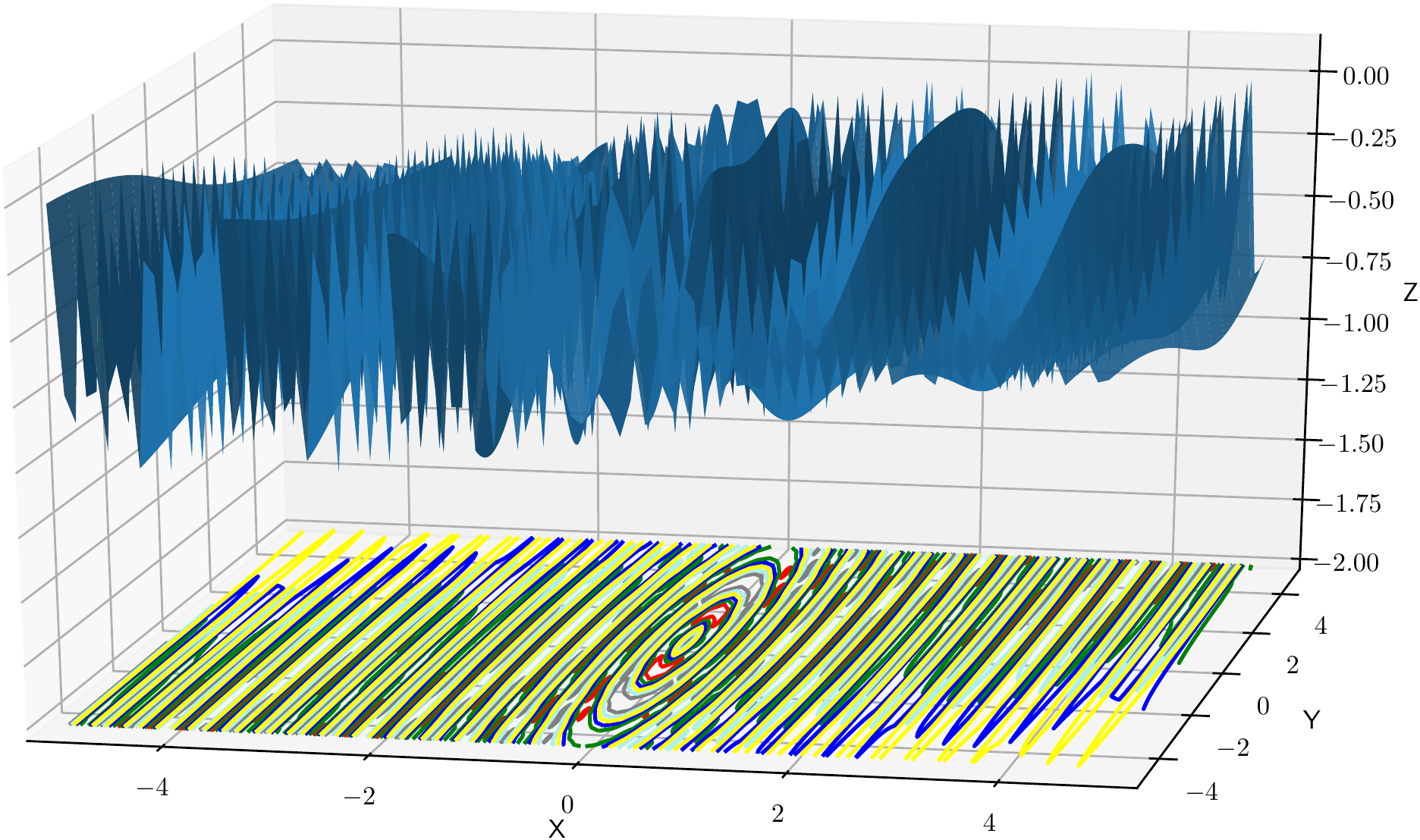}
	\subcaption{3D plot of the Pathological function in the region $[-5,5]\times[-5,5]$}
	\end{subfigure}		
	\begin{subfigure}[h]{0.5\textwidth}
		\includegraphics[height=50mm, width=60mm]{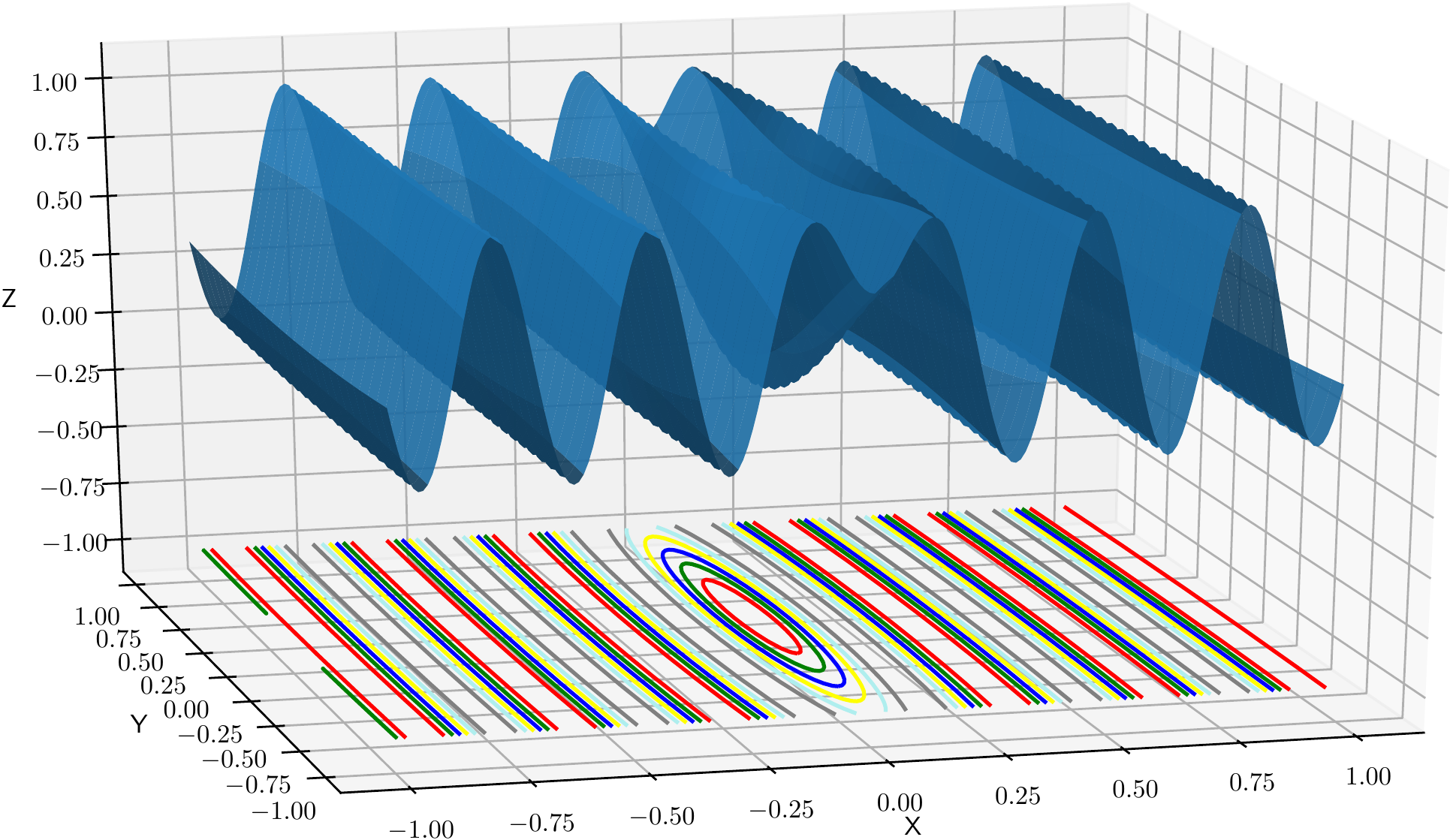}
		\subcaption{3D plot of the Pathological function in the region $[-1,1]\times[-1,1]$}\label{fig:pathological3dzoom}
	\end{subfigure}	
	\caption{3D plot of the Pathological function}\label{fig:pathological3d}	
	\end{figure} 
	\item
	We studied the sensitivity of the algorithm with regards to the quantile factor $\rho$. Recall that $\rho$ determines the threshold level used by the algorithm. The results are provided in  Fig. \ref{fig:rhocompare}.
		\begin{figure}[h]
		\vspace*{1mm}
		\centering
		\includegraphics[height=62mm, width=0.85\textwidth]{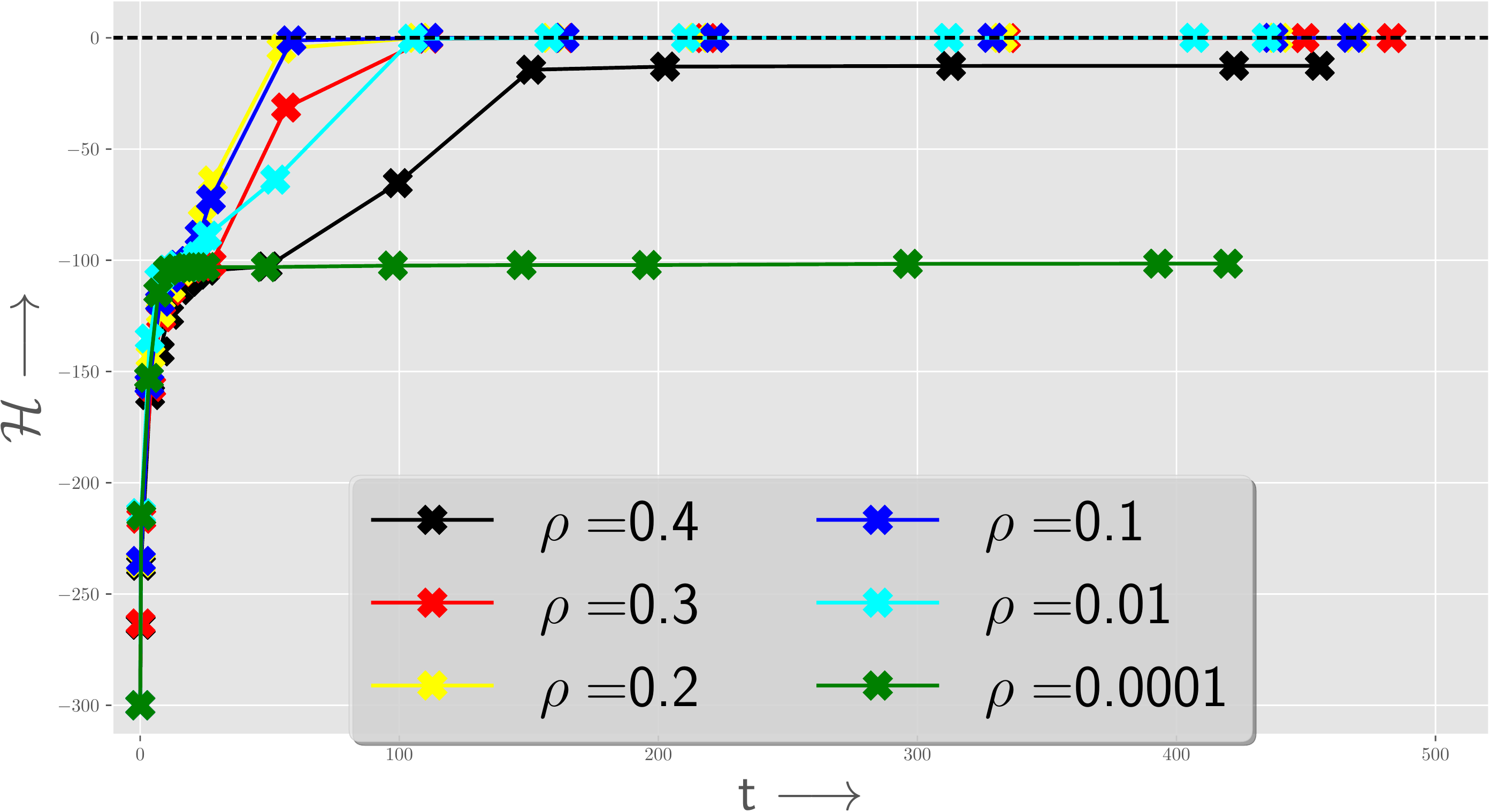}
		\caption{Performance comparison of CE2-ND for various values of $\rho$.}\label{fig:rhocompare}
	\end{figure}
	 For $\rho \in \{0.3,0.2,0.1,0.01\}$, the algorithm converges to the global optimum at a relatively faster rate. However, for $\rho=0.4$, the convergence rate is intermediate, while for $\rho=0.0001$, the drift is sluggish. Theoretically, the algorithm should converge for all values of $\rho$. However, in most practical cases, choosing $\rho$ in the range $[0.01,0.3]$ is highly recommended. Similar observation about the Monte-Carlo CE is mentioned in \cite{hu2007model}, and this needs further investigation.
	\item
	As with any stochastic approximation algorithm, the choice of the learning rate $\beta_t$ is vital.
	\item
	The computational and storage requirements of the algorithm CE2-ND are minimal. This is attributed to the streamlined and incremental nature of the algorithm. This attribute makes the algorithm suitable in settings where the computational and storage resources are scarce. 
	\begin{figure}[!h]
		\centering
		\includegraphics[height=62mm, width=0.75\textwidth]{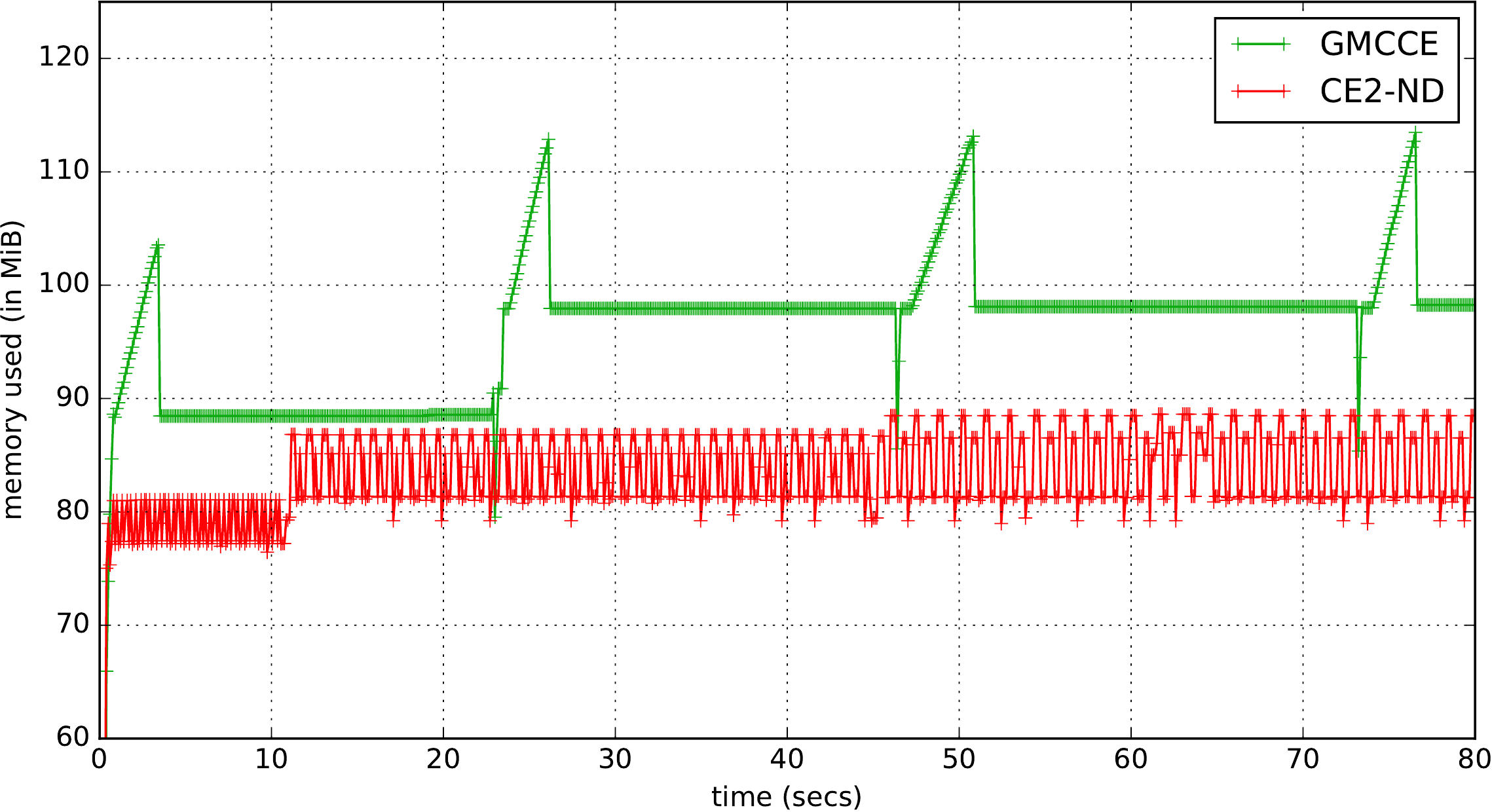}
		\caption{Memory usage comparison: CE2-ND uses very less memory compared to GMCCE. The spikes in the GMCCE case are attributed  to the sample generation.}\label{fig:memcompare}
	\end{figure}
\end{enumerate}
\section{Conclusion}
In this paper, we developed a stochastic approximation algorithm for continuous optimization based on the well known cross entropy (CE) method. Our technique efficiently and effectively tracks the ideal cross entropy method. It requires only one sample per iteration. The algorithm is incremental in nature and possesses attractive features like minimal restriction on the structural properties of the objective function, robustness, ease of implementation, stability  as well as computational and storage efficiency. We showed the almost sure convergence of our algorithm and proposed conditions required to achieve the convergence to the global maximum for a particular class of functions. Numerical experiments over diverse benchmark objective functions are shown to corroborate the theoretical findings.



\bibliographystyle{spmpsci}      

\bibliography{template}

\section*{\Large{Appendix}}
\subsection*{\bf{Stochastic Approximation Framework}}\label{sec:saframework}
\textit{Stochastic approximation algorithms} \cite{borkar2008stochastic,kushner2012stochastic,robbins1951stochastic} are a computationally appealing way of efficiently utilizing prior information and are primarily used for optimizing, tracking or regulating stochastic systems. They do so via a discounted averaging of the prior information and are usually expressed as recursive equations of the following form:
\begin{equation}\label{eq:strec}
\mathsf{Z}_{t+1} = \mathsf{Z}_{t} + \alpha_{t+1}\Delta \mathsf{Z}_{t},
\end{equation}
where $\Delta \mathsf{Z}_{t} = q(\mathsf{Z}_{t}) + b_{t} + \mathbb{M}_{t+1}$ is the \textit{increment term} also called the differential correction, $q(\cdot)$ is a Lipschitz continuous function, $b_{t}$ is the \textit{bias term} with $b_{t} \rightarrow 0$ and $\{\mathbb{M}_{t}\}$ is a \textit{martingale difference noise sequence}, \emph{i.e.}, $\mathbb{M}_{t}$ is $\mathcal{F}_{t}$-measurable and integrable and $\mathbb{E}[\mathbb{M}_{t+1} \vert \mathcal{F}_{t}] = 0, \forall t \geq 0$. Here $\{\mathcal{F}_{t}\}_{t \in \mathbb{N}}$ is a filtration, where the $\sigma$-field $\mathcal{F}_{t} = \sigma(\mathsf{Z}_i, \mathbb{M}_{i}, 1 \leq i \leq t, \mathsf{Z}_{0})$. The learning rate $\alpha_{t} > 0$ satisfies the Robbins-Monro condition $\Sigma_{t} \alpha_{t} = \infty$, $\Sigma_{t} \alpha_{t}^{2} < \infty$.

\end{document}